\documentclass[a4paper,10pt]{article}
\usepackage[utf8]{inputenc}
\usepackage[T1]{fontenc}
\usepackage{lmodern}
\usepackage{a4wide}
\usepackage{geometry}
\usepackage{graphicx}
\usepackage{amsmath, amsfonts, amsthm,amssymb,here,dsfont, hyperref}
\usepackage{color}
\usepackage{version}
\usepackage{todonotes}
\usepackage[lofdepth,lotdepth]{subfig}
\usepackage{wrapfig}
\usepackage{fullpage}
\usepackage{xcolor}

\newcommand{\one}{\mathds{1}}
\newcommand{\argmin}[1]{\underset{#1}{\operatorname{arg}\!\operatorname{min}}\;}

\newtheorem{theorem}{Theorem}
\newtheorem{proposition}{Proposition}

 \newtheorem{remark}{Remark}

\newtheorem*{assumption*}{Assumption}
\newtheorem{assumption}{Assumption}
\newtheorem{lemma}{Lemma}

\def\argmin{\mathop{\rm argmin}}

\newcommand{\Err}{{\rm Err}}
\DeclareMathOperator{\Var}{Var}
\newcommand{\ie}{{\em i.e.,~}}

\newcommand{\wrt}{{\em w.r.t.~}}
\newcommand{\iid}{{\rm i.i.d.~}}

\begin{document}

\title{Variance function estimation in regression model via aggregation procedures}
 \author{Ahmed Zaoui\footnote{Ahmed.Zaoui@univ-eiffel.fr}\\
\small{LAMA, UMR-CNRS 8050,}\\
\small{Universit\'e Gustave Eiffel}\\}
 \date{}
 
\maketitle

\begin{abstract}
In the regression problem, we consider the problem of estimating the variance function by the means of aggregation methods. 
We focus on two particular aggregation setting: Model Selection aggregation (\texttt{MS}) and Convex aggregation (\texttt{C}) where the goal is to select the best candidate and to build the best convex combination of candidates respectively among a collection of candidates. In both cases, the construction of the estimator relies on a two-step procedure and requires two independent samples. The first step exploits the first sample to build the candidate estimators for the variance function by the residual-based method and then the second dataset is used to perform the aggregation step. We show the consistency of the proposed method with respect to the $L^2$-error both for \texttt{MS} and \texttt{C} aggregations. We evaluate the performance of these two methods in the heteroscedastic model and illustrate their interest in the regression problem with reject option.
\\{\bf Keywords:} Regression, Conditional variance function, Aggregation
\end{abstract}
\section{Introduction}
\label{sec:intro}
Building efficient estimation of the level of noise is highly important for real applications and statistical analysis. 
For instance testing and confidence intervals are two historical statistical problems where a bad calibration of the noise may lead to bad conclusions. The range of use of the variance structure in the data is even wider such as in selection the optimal kernel bandwidth~\cite{Fan92}, estimation correlation structure of the heteroscedastic spatial data~\cite{Opsomer_Ruppert_Wand_Holst_Hossjer99}, estimation of the signal-to-noise ratio~\cite{Verzelen_Gassiat18}, or choosing optimal design~\cite{Muller_Stadtmuller87} and finds important applications for instance is finance with the problems of measuring volatility or risk~\cite{Anderson_Lund97} or long-term stock returns~\cite{Mammen_Nielsen_Scholz_Sperlich19}. In our case we highlight the interest of providing an efficient estimation of the variance function in the problem of regression with regret option where the good calibration of the rejection rule is highly dictated by the efficiency of the estimator of the noise level~\cite{Denis_Hebiri_Zaoui20}.  
We focus on the regression problem: we denote by $(X,Y)\in \mathbb{R}^{d}\times\mathbb{R}$ the couple of random variables where $X$ is the feature vector and $Y$ is the response variable such that 
\begin{equation*}
\label{mod:General}
Y= f^{*}(X)+W\enspace.
\end{equation*}
Here $W$ is the noise and is such that $\mathbb{E}[W|X]=0$ and $\mathbb{E}[W^2]<\infty$. In particular for any $x\in \mathbb{R}^d$ we write  $f^{*}(x)=\mathbb{E}\left[Y|X=x\right]$ to denote the regression function and $\sigma^{2}(x)=\Var(Y|X=x)=\mathbb{E}\left[(Y-f^{*}(X))^2|X=x\right]$ to denote the conditional variance function.

Despite the popularity of the problem of estimating the noise level, there remains much to do. In particular, we study in the present paper this problem from the aggregation perspective and build estimators of the conditional variance function based on Model Selection (\texttt{MS}) and Convex (\texttt{C}) aggregations. We study their consistency properties as well as their numerical performance. Our work is motivated by recent research in regression with reject option~\cite{Denis_Hebiri_Zaoui20}. There it has been observed that the rejection rule is fully determined by the variance structure. We hope that aggregation will improve the accuracy of the method.
\subsection{Related work}
Our literature review consists of three related fields:

\medskip

{\bf Conditional variance estimation}:
The problem of estimating the regression function is classical and widely studied, see for example~\cite{Biau_Devroye15,GKKW02,scornet2015, Stone77, Tsybakov08} and references therein.

Even though the problem of estimation of the conditional variance function is less studied, it has been considered in several works that can be cast into two groups according to the nature of the design (fixed or random).

When the design is fixed, the estimation of $\sigma^2$ has been studied mainly via residual-based methods~\cite{Hall_Caroll89,Hardle_Tsybakov97, Ruppert_Wand_Holst_Hosjer97} and difference-based methods~\cite{BL07,Brown_Levine07, Muller_Stadtmuller87, Wang_Brown_Cai_Levine08}. Difference-based estimators do not require the estimation of the regression function $f^*$. The first difference-based estimators has been developed by the authors in~\cite{Muller_Stadtmuller87}. They considered an initial variance estimates which are squared weighted sums of $m$ observations neighboring the fixed point where the variance function is to be estimated. The authors showed that the proposed initial variance estimates are not consistent. To solve this problem, they smoothed them with a kernel estimate. In~\cite{BL07}, the authors presented a class of non-parametric variance estimators based on different sequences and local polynomial estimation and established asymptotic normality. The authors in~\cite{ Wang_Brown_Cai_Levine08} were interested in the effect of the unknown mean on the estimation of the variance function and proved numerically that the residual-based method performs better than the first-order-difference-based estimator  when the unknown regression function $f^*$ is very smooth.

In this work we rather focus on random design. Less methods have been proposed to estimate the conditional variance function in this case. Most classical methods are the direct and the residual-based:
\begin{enumerate}
\item  \textbf{The direct method}: a simple decomposition the conditional variance function $\sigma^2$ is rewritten as the difference of the first two conditional moments, $\sigma^{2}(x)=\mathbb{E}[Y^2|X=x]-\left(\mathbb{E}[Y|X=x]\right)^2$. The direct method consists in estimating the two terms in the right side separately, see for example~\cite{Fan_Yao98, Hardle_Tsybakov97}. To be more specific, the direct estimator of $\sigma^2$ has the following form
\begin{equation*}
\hat{\sigma}^{2}_{d}(x)=\hat{g}(x)-(\hat{f}(x))^2\enspace ,
\end{equation*}
where $\hat{g}$ and $\hat{f}$ are estimators of $\mathbb{E}[Y^2|X=x]$ and $f^*$ respectively. The main drawback of this approach is that it is not always nonnegative for example if different smoothing parameters are used in estimating those terms and adaptation to the unknown regression function $f^*$ is still not available. The authors in~\cite{Hardle_Tsybakov97} proposed a local polynomial regression estimates of those terms using the same bandwidth and the same kernel function. They established the asymptotic normality of local polynomial estimators of the regression function and the variance function.
\item \textbf{The residual-based method}: this approach consists on two steps. First, one estimates the regression function and  computes the squared residuals $\hat{r}=(Y-\hat{f}(X))^2$ where $\hat{f}$ is the estimator of $f^*$.
Second, we estimate the variance function by  solving the regression problem when the input is the feature $X$ and the output variable is the residuals $\hat{r}$ . For more details, see~\cite{Fan_Yao98, Neumann94,Ruppert_Wand_Holst_Hosjer97}. It exists many ways to study this method. For instance, the authors in~\cite{Fan_Yao98} applied a local linear regression in both steps and showed that their estimator is fully regression-adaptive to the unknown regression function. Using the  local polynomial regression can be negative when the bandwidths are not selected appropriately. As a solution to this, the authors in~\cite{Yu_Jones04} proposed estimators based on a localized normal likelihood, using a standard local linear form for estimating the mean and a local log-linear form for estimating the variance. In~\cite{Ziegelmann02}, the authors introduced an exponential estimator of the conditional variance in the second step to ensure the nonnegativity of the estimator of $\sigma^2$. 
The authors in~\cite{Xu_Phillips11} used a reweighted local constant estimator (kernel estimator) based on maximizing the empirical likelihood subject to a bias-reducing moment restriction. Moreover, such estimators have the form $\hat{\sigma}^{2}(X)=\sum_i \omega_i(X)(Y_i-\hat{f}(X_i))^2$  where $\omega_i(X)$ are weight functions~\cite{Denis_Hebiri_Zaoui20, KulikWichelhaus11}. Recently the authors in~\cite{Denis_Hebiri_Zaoui20} used the previous estimator and focused on estimating the regression function and the variance function respectively by $k$NN. Under mild assumptions, they provided the rate of convergence of the $k$NN estimator of the conditional variance function in supremum norm. The residual-based method can be regarded as a generalized difference-based estimator. For more details, see~\cite{Fan_Yao98}.
\end{enumerate}

In this paper, we focus on the residual-based method to estimate the variance function since they appear more tractable. In particular, we develop an aggregation procedures for this task. 

\medskip

{\bf Aggregation methods}: 
Aggregation is a popular approach in statistics and machine learning. This technique is well known to estimate the regression function in the homoscedastic or heteroscedastic model. We refer the
reader to the baseline articles~\cite{Audibert09,Bunea_Tsybakov_Wegkamp07, Juditsky_Nemirovski00,Tsybakov03,Tsybakov14,Yuhong04}. Given a set of estimators of regression function $f^*$, the aggregation constructs a new estimator, called the aggregate, which mimics in a certain sense the behavior of the best estimator in a class of estimates. There are several popular types of aggregation and we focus on two: the model selection aggregation (\texttt{MS}) which allows to select the best estimator from the set; the convex aggregation (\texttt{C}) where the goal is to select the best convex combination of functions in the set. In general, the aggregation procedures are based on sample splitting, that is, the original data set $\mathcal{D}_N$ is split into two independent data sequences $\mathcal{D}_k$ and $\mathcal{D}_l$ with $N=l+k\geq 1$. The first subsample $\mathcal{D}_k$ is exploited to build $M>1$ competing estimators of the regression function $f^*$ and $\mathcal{D}_l$ is used to aggregate them. Most of the work has focused on fixing the first sample, resulting in fixed estimators (the estimators are then seen as fixed functions). Under mild assumptions, the auhors in~\cite{Tsybakov03} showed that the optimal rates for \texttt{MS} and \texttt{C} aggregation \wrt $L^2$-error in gaussian regression model are of order $\frac{\log(M)}{N}$, and $\frac{M}{N}$ if $M\leq \sqrt{N}$, respectively $\sqrt{\frac{1}{N}\log\left(\frac{M}{\sqrt{N}}+1\right)}$ if $M\geq \sqrt{N}$ in both cases.


In this paper we consider aggregation methods for the conditional variance estimation. Up to our knowledge, such approaches have not been considered yet for this problem.

\medskip
{\bf Reject option}: Reject option is important in nonparametric statistic since it helps avoiding uncertain prediction. 
It has been initially introduced in the classification setting~\cite{Chow57,Chow70, Denis_Hebiri19, Herbei_Wegkamp06, Lei14, Naadeem_Zucker_Hanczar10, Vovk_Gammerman_Shafer05} where it has shown important improvements in the quality of prediction.
It has been recently developed in the case of the regression model in the case where the rejection rate is controlled by the practitioner~\cite{Denis_Hebiri_Zaoui20}. There the authors provided a characterization of the optimal rule (knowing the true distribution of the data) and demonstrated that it relies on thresholding the conditional variance function. More formally, it is defined as follows: given a rejection rate $\varepsilon \in(0,1)$
\begin{equation*}
    \Gamma^{*}_{\varepsilon}(x):=\begin{cases}
  \left\{f^{*}(x)\right\}  & \text{if} \;\;
  {\sigma}^2(x)\leq F_{{\sigma}^2}^{-1}(1-\varepsilon)\\
  \emptyset      & \text{otherwise} \enspace,
  \end{cases}
\end{equation*}
where $F_{{\sigma}^2}^{-1}$ is the generalized inverse of the cumulative distribution function (CDF) $F_{{\sigma}^2}$ of $\sigma^2(X)$.
As can be observed, this optimal solution depends on several parameters: the rejection rate $\varepsilon$ that is known in advance, the CDF $F_{{\sigma}^2}$ that is efficiently estimated the empirical CDF, the regression function $f^{*}(x)$ for which efficient estimators exist in the literature, and the conditional variance $\sigma^2$. This last quantity is less considered in the literature and our goal is to build accurate estimators of the conditional variance that rely on aggregation. The ultimate purpose is then to make a sharper estimation of the optimal rule in the case of the rejection option in the regression setting.
\subsection{Main contribution}
We develop  the notions of model selection aggregation and convex aggregation to estimate the conditional variance function. To our knowledge, this work is the first to deal with this setting. We consider two independent datasets: the first will be used to build the initial estimators of the variance function and the second to aggregate them. We called these estimators the \texttt{MS}-estimator and \texttt{C}-estimator. We consider the residual-based method to build the initial estimators which is based on estimating the regression function in the first step. We focus on estimating the regression function by model selection aggregation and convex aggregation. The major part of the paper is then devoted to show the upper-bounds of $L^2$-error of the \texttt{MS}-estimator and \texttt{C}-estimator when the initial estimators can be arbitrary or verify very weak conditions such that boundeness. We establish that the rate of convergence for \texttt{MS} and \texttt{C} procedures is of order $O(\log(M_1)/N)^{1/8})$ when $Y$ is unbounded and is of order $O(\log(M_1)/N)^{1/4})$ when $Y$ is bounded. Finally, we obtain numerical results which show the performance of our procedures.

\subsection{Outline}
The paper is organized as follow. In the next section,
the aggregation problems, the model selection and convex aggregations, are described in detail.
Section~\ref{sec:Main_Results} is focused on investigating the upper-bounds for the $L^2$-error of our procedures. Finally, Section~\ref{sec:Numerical} presents a numerical comparison of the proposed method w.r.t. the heteroscedastic model as well as a direct application to the regression framework with reject option.\\
\\ \textbf{Notations.} 
 We introduce some notation that is used throughout this paper.  Let $p\geq 2$ be an integer, the set of integers $\{1,\cdots, p\}$ is denoted $[p]$. 
Let $N$ an integer. For any function $f:\mathbb{R}^{d}\rightarrow \mathbb{R}$, we define the empirical norm $\|f\|_{N}^{2}=\frac{1}{N}\sum_{i=1}^{N}|f(X_i)|^{2}$ and the supremum norm $\|f\|_{\infty}=\sup_{x\in \mathbb{R}^d}|f(x)|$. Moreover, we denote by $\Lambda^{p}:=\{\lambda\in \mathbb{R}^{p}: \lambda_j\geq 0, \sum_{j=1}^{p}\lambda_j=1\}$ the simplex.  Let $\|\cdot\|_{1,p}$ denote the $\ell_1$ norm on $\mathbb{R}^{p}$, that is, $\|\lambda\|_{1,p}:=\sum_{j=1}^{p}|\lambda_j|$. For the sake of simplicity, let $Z=(Y-f^{*}(X))^2$.
\section{Aggregation estimators}
\label{sec: Aggreg_Estimators}
In this section, we describe an estimation algorithm of the variance function $\sigma^2$ by aggregation. In particular, we focus on two types of aggregations: the model selection aggregation (\texttt{MS}), and the convex aggregation (\texttt{C}). These aggregation problems, (\texttt{MS}) and (\texttt{C}), have been considered to estimate the unknow regression function in the regression model. The objective is to estimate $f^*$ by a combination of elements of a known set called dictionary made up of deterministic functions or preliminary estimators. The collection of estimators or algorithms is given and can be of parametric, nonparametric or semi-parametric nature. Given a set of estimators, the \texttt{MS}-aggregation consists in constructing a new estimator which is approximately as good as the best estimator in the set, while the objective of \texttt{C}-aggregation is to construct a new estimator which is approximately as good as the best convex combination of the elements in the set, for more details see~\cite{Audibert09,Bunea_Tsybakov_Wegkamp07, Juditsky_Nemirovski00,Tsybakov03,Tsybakov14,Yuhong04}. Besides, to construct an aggregate of $\sigma^2$, we first introduce two independent learning samples: $\mathcal{D}_{n}=\{(X_i,Y_i), i=1,\cdots,n\}$ 
and $\mathcal{D}_{N}=\{(X_{i},Y_{i}), i=n+1,\cdots,n+N\}$ which consist of respectively $n$ and $N$ \iid copies of $(X,Y)$.
\subsection{Model selection aggregation}
\label{sec:AggreMS}
In this first paragraph, we detail how we perform \texttt{MS}-aggregation in order to estimate of the conditional variance function $\sigma^{2}$ by \texttt{MS}. It consists in two steps: one step for the estimation of the regression function $f^*$ and a second one devoted to the estimation of $\sigma^2$. More precisely, in the first one we bluid $M_1$ estimators of the regression function $\hat{f}_1,\cdots,\hat{f}_{M_1}$ based on $\mathcal{D}_{n}$ with $ 2\leq M_1<\infty$. Then, we use the second dataset $\mathcal{D}_{N}$ to estimate  $f^{*}$ by \texttt{MS}: we select the optimal index, denoted $\hat{s}$ as follows
\begin{eqnarray}
\label{est:hats}
\hat{s}\in \argmin_{s\in[M_1]}\hat{\mathcal{R}}_{N}(\hat{f}_{s}), \enspace \text{where} \enspace \hat{\mathcal{R}}_{N}(\hat{f}_{s})=\frac{1}{N}\sum_{i=1}^{N}|Y_i-\hat{f}_{s}(X_i)|^{2}\enspace, 
\end{eqnarray}
and the aggregate of the regression function, denoted by $\hat{f}_{\texttt{MS}}$, is then given by
\begin{equation}
\hat{f}_{\texttt{MS}}:=\hat{f}_{\hat{s}}.
\label{est: hatfMS}
\end{equation}
In the second step, given the estimator $\hat{f}_{\texttt{MS}}$ builded on $\mathcal{D}_N$, we construct using back the sample $\mathcal{D}_{n}$ $M_2$ estimators of the variance function $\sigma^2$, denoted $\hat{\sigma}^{2}_{\hat{s},1},\cdots,\hat{\sigma}^{2}_{\hat{s},M_2}$, by residual-based method with $2\leq M_2<\infty$. Finally, based on $\mathcal{D}_{N}$ again, we select the optimal single, denoted $\hat{m}$, as follows
\begin{eqnarray*}
\hat{m}\in \argmin_{m\in [M_2]}\hat{R}_{N}(\hat{\sigma}^{2}_{\hat{s},m})\enspace \text{where}\enspace \hat{R}_{N}(\hat{\sigma}^{2}_{\hat{s},m})= \frac{1}{N}\sum_{i=1}^{N}|\hat{Z}_{i}-\hat{\sigma}^{2}_{\hat{s},m}(X_i)|^{2}\enspace
\end{eqnarray*}
with $ \hat{Z}_{i}=\left(Y_{i}-\hat{f}_{\texttt{MS}}(X_i)\right)^2$. Therefore,  the aggregate of the variance function, called \texttt{MS}-estimator and denoted $\hat{\sigma}^{2}_{\texttt{MS}}$, is defined as
\begin{equation}
\hat{\sigma}^{2}_{\texttt{MS}}:=\hat{\sigma}^{2}_{\hat{s},\hat{m}}.
\label{est:MS}
\end{equation} 
\subsection{Convex aggregation}
\label{sec:AggreCM}
Convex aggregation procedures for nonparametric regression are discussed in~\cite{ Audibert04,Bunea_Tsybakov_Wegkamp07, Tsybakov03}. We describe here an algorithm for aggregating estimates of the conditional variance function $\sigma^2$ by \texttt{C}-aggregation. As for \texttt{MS}-aggregation, the construction of the aggregate of $\sigma^2$ needs two independent datasets $\mathcal{D}_{n}$ and $\mathcal{D}_{N}$. The aggregation still proceeds in two steps: one for estimating $f^*$ and the second for the estimation of $\sigma^2$. Each step is again decomposed in two parts. Firstly, we consider $M_1$ estimators of the regression function $f^*$, $\{\hat{f}_1,\dots, \hat{f}_{M_1}\}$, based on $\mathcal{D}_{n}$, and for any $\lambda \in \Lambda^{M_{1}}$ we define the linear combinations $\hat{f}_{\lambda}$ 
$$
\hat{f}_{\lambda}=\sum_{j=1}^{M_1}\lambda_{j}\hat{f}_{j}.
$$
Then, aggregates of the regression function based on the sample $\mathcal{D}_{N}$ have the form
\begin{equation}
\hat{f}_{\texttt{C}}:=\hat{f}_{\hat{\lambda}}=\sum_{j=1}^{M_1}\hat{\lambda}_{j}\hat{f}_{j}\enspace,
\label{est:fhatCM}
\end{equation}
where the estimator $\hat{\lambda}$ is defined by
\begin{equation*}
\hat{\lambda}\in \argmin_{\lambda\in \Lambda^{M_1}}\hat{\mathcal{R}}_{N}(\hat{f}_{\lambda}).
\end{equation*}
Once $\hat{f}_{\texttt{C}}$ is obtained, we focus on the estimation of $\sigma^2$. Based on the sample $\mathcal{D}_{n}$, we build $M_2$ estimators for the conditional variance function by the residual-based method, denoted $\hat{\sigma}_{\hat{\lambda},1}^{2},\cdots, \hat{\sigma}_{\hat{\lambda},M_2}^{2} $, and for any $\beta\in \Lambda^{M_2}$ we  define $\hat{\sigma}^{2}_{\hat{\lambda},\beta}$ as follows
$$
\hat{\sigma}^{2}_{\hat{\lambda},\beta}=\sum_{j=1}^{M_2}\beta_{j}\hat{\sigma}^{2}_{\hat{\lambda},j}\enspace .$$
Finally, based on $\mathcal{D}_{N}$, the aggregate estimate for $\sigma^2$ is given by
\begin{equation}
\hat{\sigma}^{2}_{\texttt{C}}:=\hat{\sigma}^{2}_{\hat{\lambda},\hat{\beta}}\enspace,
\label{est:CM}
\end{equation} 
where the estimator $\hat{\beta}$ is defined by
\begin{equation*}
\hat{\beta}\in \argmin_{\beta\in\Lambda^{M_2}}\hat{R}_{N}(\hat{\sigma}^{2}_{\hat{\lambda},\beta}), \enspace \text{where}\enspace \hat{R}_{N}(\hat{\sigma}^{2}_{\hat{\lambda},\beta})= \frac{1}{N}\sum_{i=1}^{N}|\hat{Z}_{i}-\hat{\sigma}^{2}_{\hat{\lambda},\beta}(X_i)|^{2}\enspace
\end{equation*}
with $\hat{Z}_{i}=(Y_i-\hat{f}_{\texttt{C}}(X_i))^{2}$. We called $\hat{\sigma}^{2}_{\texttt{C}}$ the \texttt{C}-estimator.
\section{Main results}
\label{sec:Main_Results}
This section is devoted to studying the $L^2$-error of \texttt{MS}-estimator and \texttt{C}-estimator. Firstly, we introduce general conditions required on the model in Section~\ref{sec:Assumptions}. Secondly, we show the consistency of our methods in Sections~\ref{sec:boundMsEst} and~\ref{sec:boundCEst}.
\subsection{Assumptions}
\label{sec:Assumptions}
The following assumptions are the bedrock of our theoretical analysis:
\begin{assumption}
\label{ass:fstar_bound}
The functions $f^*$ and $\sigma^2$ are bounded.
\end{assumption}
\begin{assumption}
\label{ass:Y_bound}
 $Y$ is bounded or $Y$ satisfies the gaussian model
 \begin{equation}
 \label{modelGaussian}
Y= f^{*}(X)+\sigma(X)\xi,
\end{equation}
where $\xi$ is independent of $X$ and distributed according to a standard normal distribution.
\end{assumption}
These assumptions are relatively weak and play a key role in our approach. They allow to use the Hoeffding's inequality in the case of boundness of $Y$ or $\xi$. In particular, it is important to emphasize that Assumptions~\ref{ass:fstar_bound} and~\ref{ass:Y_bound} guarantee that the variable $Y-f^{*}(X)$ is sub-Gaussian (see Lemma~\ref{lem:HoeffdingLemma} in the case where $Y$ is bounded).
\subsection{Upper bound for $\hat{\sigma}^{2}_{\texttt{MS}}$}
\label{sec:boundMsEst}
We study the $L^2$-error of the \texttt{MS}-estimator $\hat{\sigma}^{2}_{\texttt{MS}}$. Let $\mathcal{R}(\hat{f}_{s})=\mathbb{E}\left[|Y-\hat{f}_{s}(X)|^{2}\right]$ for all $s\in[M_1]$. We define $s^{*}$ as follows
 \begin{equation}
 \label{est:star(s)}
 s^{*}\in \argmin_{s\in[M_1]}\mathcal{R}(\hat{f}_{s}) \enspace .
 \end{equation} 
Besides, we need the following assumptions in the case of \texttt{MS}-aggregation:
\begin{assumption}
\label{ass:hatfs_hatsigma_hats_m_bound}
For all $s\in [M_1]$  and all $m\in [M_2]$ , $\hat{f}_{s}$ and $\hat{\sigma}^{2}_{s,m}$ are bounded a.s $\mathcal{D}_n$. More precisely, there exist two positive constants $K_1$ and $K_2$ such that for all $n\in \mathbb{N}^{*}$ 
\begin{equation*}
\max_{s\in [M_1]}\|\hat{f}_{s}\|_{\infty}\leq K_1, \enspace \text{and} \enspace \max_{(s,m)\in [M_1]\times [M_2]}\|\hat{\sigma}_{s,m}^{2}\|_{\infty}\leq K_2.
\end{equation*}
\end{assumption}
\begin{assumption}[Separability hypothesis]
\label{ass:SeperHyp}
There exists $\delta_0 >0$ such that $$\delta^*=\min_{s\neq s^*}\{|\mathcal{R}(\hat{f}_{s})-\mathcal{R}(\hat{f}_{s^*})|\}>\delta_0\enspace.$$
\end{assumption}
Both assumptions are used to control the $L^2$-error of the \texttt{MS}-estimator $\hat{\sigma}^{2}_{\texttt{MS}}$. Assumption~\ref{ass:hatfs_hatsigma_hats_m_bound} describes the boundedness of the estimators. In our construction, the constants $K_1$ and $K_2$  do need not to be known. Let $\mathbb{E}$ be the expectation which is taken with respect to both $X$ and the samples $\mathcal{D}_{n}$ and $\mathcal{D}_{N}$. We establish the following result
\begin{theorem}
\label{thm:Risk_MS_sigma3}
Let $\hat{f}_{\texttt{MS}}$ and $\hat{\sigma}^{2}_{\texttt{MS}}$ be two \texttt{MS}-estimators of $f^*$ and $\sigma^2$ defined in Eq.~\eqref{est: hatfMS} and~\eqref{est:MS} respectively. Then, under Assumptions~\ref{ass:fstar_bound}-~\ref{ass:SeperHyp}, there exist two absolute constants $C_1>0$ and $C_2>0$ such that 
\begin{multline*}
\label{ineq:BoundMSSigma}
\mathbb{E}\left[|\hat{\sigma}_{\texttt{MS}}^{2}(X)-\sigma^{2}(X)|^2\right]\leq 
\mathbb{E}\left[\min_{m\in[M_2]}\mathbb{E}_{X}\left[|\hat{\sigma}^{2}_{s^*,m}(X)-\sigma^{2}(X)|^{2}\right]\right] +C_1\left\{\min_{s\in[M_1]}\mathbb{E}\left[\|\hat{f}_{s}-f^{*}\|_{N}^{2}\right]\right\}^{1/2p}+\\ C_2\left(\frac{\log(M_{1})}{N}\right)^{1/4p},
\end{multline*}
with $p=1$ if $Y$ is bounded and $p=2$ otherwise.
\end{theorem}
The proof of this result is postponed to the Appendix. Let's give a sketch of the proof. The $L^2$-error is the exces risk of $\hat{\sigma}^{2}_{\texttt{MS}}$ where $\mathbb{E}\left[|\hat{\sigma}^{2}_{\texttt{MS}}(X)-\sigma^{2}(X)|^2\right]:=\mathbb{E}\left[R(\hat{\sigma}^{2}_{\texttt{MS}})- R(\sigma^2)\right]$ with $R(\sigma^2)=\mathbb{E}\left[|Z-\sigma^{2}(X)|^2\right]$. We introduce the minimizer $\bar{\sigma}^{2}_{\texttt{MS}}:= \hat{\sigma}^{2}_{\hat{s},\bar{m}}$  where $\bar{m}\in \argmin_{m\in[M_2]} R(\hat{\sigma}^{2}_{\hat{s},m})$. We consider the decomposition $\mathbb{E}\left[|\hat{\sigma}^{2}_{\texttt{MS}}(X)-\sigma^{2}(X)|^2\right]= \mathbb{E}\left[R(\hat{\sigma}^{2}_{\texttt{MS}})-R(\bar{\sigma}^{2}_{\texttt{MS}})\right]+\mathbb{E}\left[R(\bar{\sigma}^{2}_{\texttt{MS}})- R(\sigma^2)\right]$. We control the two terms in the right side separately. The first one is the  estimation error (variance term). To control it, we need to introduce  $\tilde{\sigma}^{2}_{\texttt{MS}}:= \hat{\sigma}^{2}_{\hat{s},\tilde{m}}$ where $\tilde{m}\in \argmin_{m\in[M_2]} R_{N}(\hat{\sigma}^{2}_{\hat{s},m})$, 
with $R_N(\hat{\sigma}^{2}_{\hat{s},m})=\frac{1}{N}\sum_{i=1}^{N}|Z_i-\hat{\sigma}^{2}_{\hat{s},m}(X_i)|^2$. The upper bound of the variance depends on the $L^2$-error of the aggregate $\hat{f}_{\texttt{MS}}$ with respect to the empircal norm. The second one is the approximation error. Its upper-bound is linked to $\mathbb{P}(\hat{s}\neq s^*)$. 

\medskip
Theorem~\ref{thm:Risk_MS_sigma3} gives the upper-bound for $L^2$-error of $\hat{\sigma}^{2}_{\texttt{MS}}$. This bound consists of two parts: the first part is the bias of \texttt{MS}-estimator $\hat{\sigma}^{2}_{\texttt{MS}}$ and depends on the deterministic selector $s^*$; the second part is composed of the two remaining terms and corresponds to the estimation error (variance). The first term is the bias term  of $\hat{f}_{\texttt{MS}}$ expressed in terms of the empirical norm $\|\cdot\|_{N}^{2}$, and the second one characterises the price to pay for \texttt{MS}-aggregation and is of order $\left(\log(M_{1})/N\right)^{1/4p}$ where $p=1$ if $Y$ is bounded and $p=2$ otherwise. Note that this rate is slower than in the case of the estimation of the regression function $f^*$. This slow rate is due to the double aggregation that we need to perform for the estimation of the conditional variance function.
\subsection{Upper bound for $\hat{\sigma}^{2}_{\texttt{C}}$ }
\label{sec:boundCEst}
In this part, we focus in studying the $L^2$-error of $\hat{\sigma}^{2}_{\texttt{C}}$. The construction of $\hat{\sigma}^{2}_{\texttt{C}}$ needs the following estimators $\{\hat{f}_i\}_{i=1}^{M_1}$ and  $\{\hat{\sigma}_{\hat{\lambda},i}\}_{i=1}^{M_2}$. We require the following conditions
\begin{assumption}
\label{ass:hatflambda_hatsigma_hatlambda__bound}
For all $i\in [M_1]$, all $\lambda\in \Lambda^{M_1}$ and all $j\in [M_2]$ , $\hat{f}_{i}$ and $\hat{\sigma}^{2}_{\lambda,j}$ are bounded a.s. $\mathcal{D}_n$.
\end{assumption}
\begin{assumption}
\label{ass:KLipchitz}
Suppose that there exists a constant $K\geq 0$ such that for every $j\in[M_2]$
\begin{equation*}
\mathbb{E}\left[|\hat{\sigma}^{2}_{\lambda_1,j}(X)-\hat{\sigma}^{2}_{\lambda_2,j}(X)|\right]\leq K\|\lambda_1-\lambda_2\|_{1,M_1}, \enspace \forall \lambda_1,\lambda_2 \in \Lambda^{M_1}\enspace a.s.
\end{equation*}
\end{assumption}
Assumption~\ref{ass:KLipchitz} is a strong condition. However, it holds, for instance, for estimators of the form $\hat{\sigma}^{2}_{\lambda,j}(X)=\sum_i \omega_i(X)(Y_i-\hat{f}_\lambda(X_i))^2$ 
 where $\omega_i(X)$ are weight functions, that are nonnegative and sum to one. The next theorem is the main result of this section, it display the upper-bound of $L^2$-error for $\hat{\sigma}^{2}_{\texttt{C}}$.
\begin{theorem}
\label{thm:UpperBoundCM}
Let $\hat{f}_{\texttt{C}}$ and $\hat{\sigma}^{2}_{\texttt{C}}$ be two \texttt{C}-estimators of $f^*$ and $\sigma^2$ defined in Eq.~\eqref{est:fhatCM} and~\eqref{est:CM} respectively. Then, under Assumptions~\ref{ass:fstar_bound},~\ref{ass:Y_bound},~\ref{ass:hatflambda_hatsigma_hatlambda__bound}, and~\ref{ass:KLipchitz}, there exist two absolute constants $C_1>0$ and $C_2>0$ such that 
\begin{multline*}
\label{ineq:BoundCMSigma}
\mathbb{E}\left[|\hat{\sigma}_{\texttt{C}}^{2}(X)-\sigma^{2}(X)|^2\right]\leq 
\mathbb{E}\left[\inf_{\beta\in \Lambda^{M_2}}\mathbb{E}_{X}\left[|\hat{\sigma}^{2}_{\hat{\lambda},\beta}(X)-\sigma^{2}(X)|^{2}\right]\right] +C_1\left\{\inf_{\lambda\in \Lambda^{M_1}}\mathbb{E}\left[\|\hat{f}_{\lambda}-f^{*}\|_{N}^{2}\right]\right\}^{1/2p}+\\ C_2\left(\frac{\log(M_{1})}{N}\right)^{1/4p},
\end{multline*}
with $p=1$ if $Y$ is bounded and $p=2$ otherwise.
\end{theorem}
As for Theorem~\ref{thm:Risk_MS_sigma3}, the upper-bound for the $L^2$-error of \texttt{C}-estimator $\hat{\sigma}_{\texttt{C}}^2$ is composed of three terms. The first one is the bias term of $\hat{\sigma}_{\texttt{C}}^2$ which depends on the random selector $\hat{\lambda}$, the second and third ones is a bound of the variance term that rely on the bias term of $\hat{f}_{\texttt{C}}$ with respect to the empirical norm $\|\cdot\|_{N}^{2}$ and on the price to pay for convex aggregation which is of order $\left(\log(M_{1})/N\right)^{1/4p}$ where $p=1$ if $Y$ is bounded and $p=2$ otherwise.  

\medskip

We notice that both procedures $\texttt{MS}$ and $\texttt{C}$ have the same rate. Indeed, the variance term of 
$\hat{\sigma}_{\texttt{MS}}^{2}$ and $\hat{\sigma}_{\texttt{C}}^{2}$ is based on the upper bound for $\hat{f}_{\texttt{MS}}$ and $\hat{f}_{\texttt{C}}$. Moreover, the aggregates  $\hat{f}_{\texttt{MS}}$ and $\hat{f}_{\texttt{C}}$ have the same rate which is of order $(\log(M_1)/N)^{1/2}$ with respect to the empirical norm $\|\cdot\|_{N}^{2}$, see Proposition~\ref{prob:Emp_Nor_fhatMS} and Proposition~\ref{prop:EmpiricalNormfCM} in the Appendix. Let's now compare with the rates of $\hat{f}_{\texttt{MS}}$ and $\hat{f}_{\texttt{C}}$ with respect to $L^2$-error and $L^2$-risk. For the Gaussian and bounded regression model, the rate of the variance term of $\hat{f}_{\texttt{MS}}$ and  $\hat{f}_{\texttt{C}}$ is of order $\frac{\log(M)}{N}$ and $\frac{M}{N}$ if $M\leq \sqrt{N}$ respectively, $\sqrt{\frac{1}{N}\log\left(\frac{M}{\sqrt{N}}+1\right)}$ if $M\geq \sqrt{N}$ in both cases, see~\cite{Bunea_Tsybakov_Wegkamp07,Lecue13,Lecue_Mendelson09, Tsybakov03}. We can deduce that our rates are very slow because our procedures need to estimate at the same time the unknown regression function $f^*$  and the variance function $\sigma^2$ by aggregation procedures.  
\section{Numerical results}
This section is devoted to the numerical analysis of our procedures. In Section~\ref{subsec:Data}, we describe four heteroscedastic models in the gaussian case and two models when $Y$ is bounded. 
Second, we evaluate the performances of \texttt{MS}-estimator and \texttt{C}-estimator for different examples in Section~\ref{subsec:BenefitAggre}. Once we have calibrated our estimate of the variance function $\sigma^2$, we exploit it to consider the problem of regression with reject option in Section~\ref{subsec:App}. 
\label{sec:Numerical}
\subsection{Data}
\label{subsec:Data}
Our numerical study relies on synthetic data:
\\\textbf{Heteroscedastic models:} we propose four examples of heteroscedastic models~\eqref{modelGaussian}:
\begin{itemize}
\item Model $1$: let $a\in\{1/4,1\}$ and $X=(X_1,X_2,X_3)$ have a uniform distribution on $[0,1]^3$. Let
\begin{enumerate}
\item $f^{*}(X)=0.1\cos(X_1)+\exp(-X_{3}^{2})$;
\item $\sigma^{2}(X)=a\left(0.1+\exp(-7(X_1-0.2)^2)+\exp(-10(X_2-0.5)^2+\exp(-50(X_3-0.9)^2\right).$
\end{enumerate}
\item Model $2$: let $X=(X_1,\cdots,X_{10})$ have a uniform distribution on $[0,1]^{10}$. We define 
\begin{enumerate}
\item $f^{*}(X)= 0.1+\exp(-X_{1}^2)+0.2\sin(X_2+X_3+X_4+0.1X_{5}^2)$;
\item $\sigma^{2}(X)= \frac{1}{2}(0.5+\sqrt{X_{1}(1-X_2)}+0.8X_{3}X_{4}+X_{5}X_{6}X_{7}^2+0.9\exp(-500(X_8+X_9+X_{10}-0.5)^2))^2$. 
\end{enumerate}

\item Model $3$: sparse model 
\begin{enumerate}
\item $f^{*}(X)= X\beta$, \enspace $\beta \in \mathbb{R}^p$;
\item $\sigma^{2}(X)=\frac{1}{2}\left(0.3+\sqrt{X_1(1-X_1)}\sin\left(\frac{2.1\pi}{X_2+0.05}\right)+0.5X_{3}+X_4\right)^2$.
\end{enumerate}
\end{itemize}
\textbf{Bounded $Y$:} we consider the following regression model when $Y$ is bounded
\begin{equation}
Y=f^{*}(x)+\sigma(X)\varrho
\label{mod: bound}
\end{equation}
where $\varrho$ have a uniform distribution on $[-\sqrt{3},\sqrt{3}]$. We give the following examples of models~\label{mod: bound}:
\begin{itemize}
\item Model $4$: let $X$ have a uniform distribution on $[0,1]^2$ and
\begin{enumerate}
\item $f^{*}(X)= X_{1}+\exp(-X_{2}^{2})$;
\item $\sigma^{2}(X)= 0.01+X_{1}\exp\left(-(X_{2}-0.9)^2\right)$.
\end{enumerate}
\item Model $5$: let $X=(X_1,X_2,X_{3})$ have a uniform distribution on $[0,1]^{3}$ and
\begin{enumerate}
\item $f^{*}(X)= X_1+X_2+0.5 \cos(X_3)$;
\item $\sigma^{2}(X)= \left(0.3+\sqrt{X_{1}(1-X_{1})}\sin\left(\frac{(2.1)\pi}{X_2+0.05}\right)+X_3\right)^2$.
\end{enumerate}

\end{itemize}

We describe the previous models. We display in Figures~\ref{hist:HistSIgma} and~\ref{hist:HistSIgmaRegB} the histograms of the variance function for every model. 
\begin{figure}[ht]%
 \centering
 \subfloat{\includegraphics[scale=0.24]{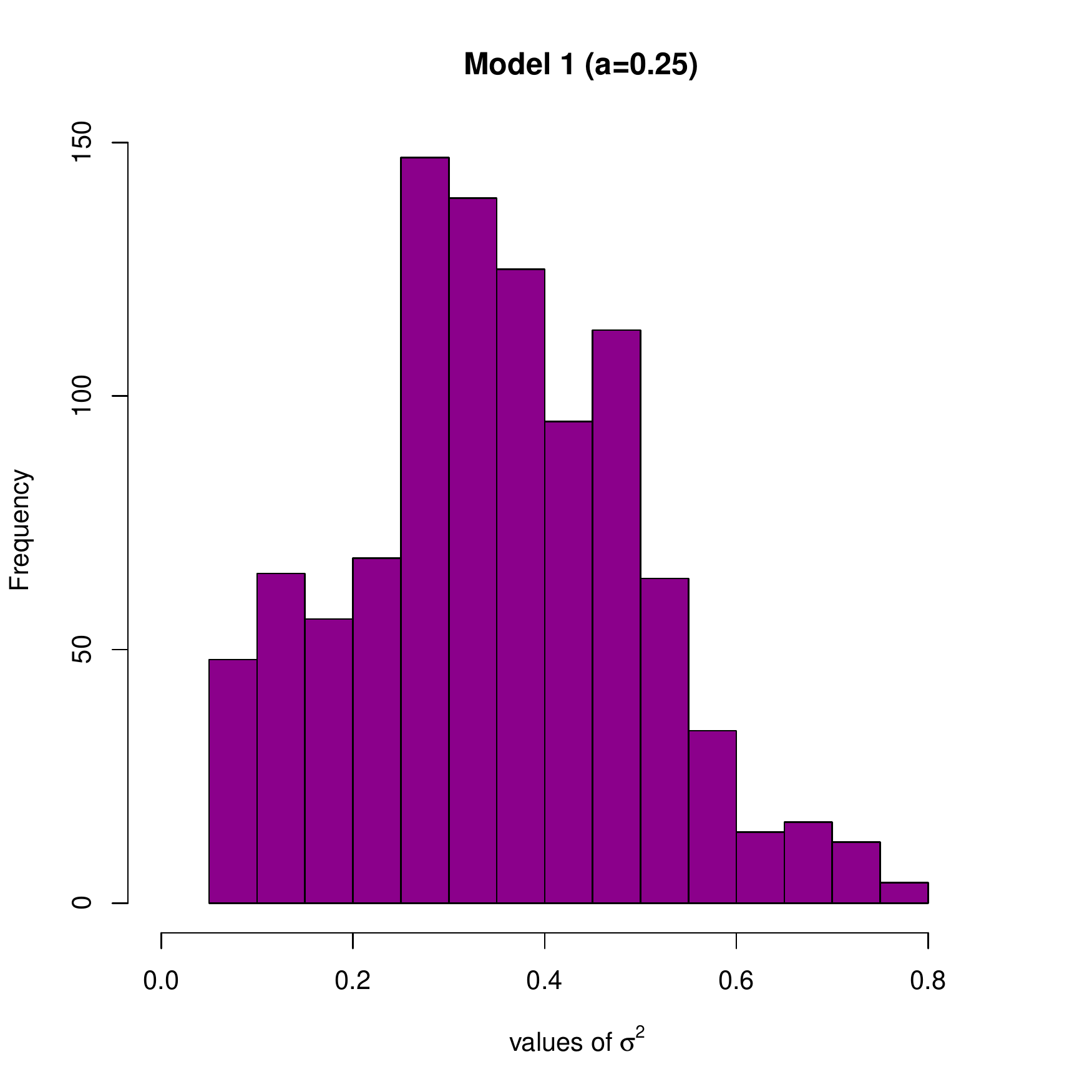}}%
 \subfloat{\includegraphics[scale=0.24]{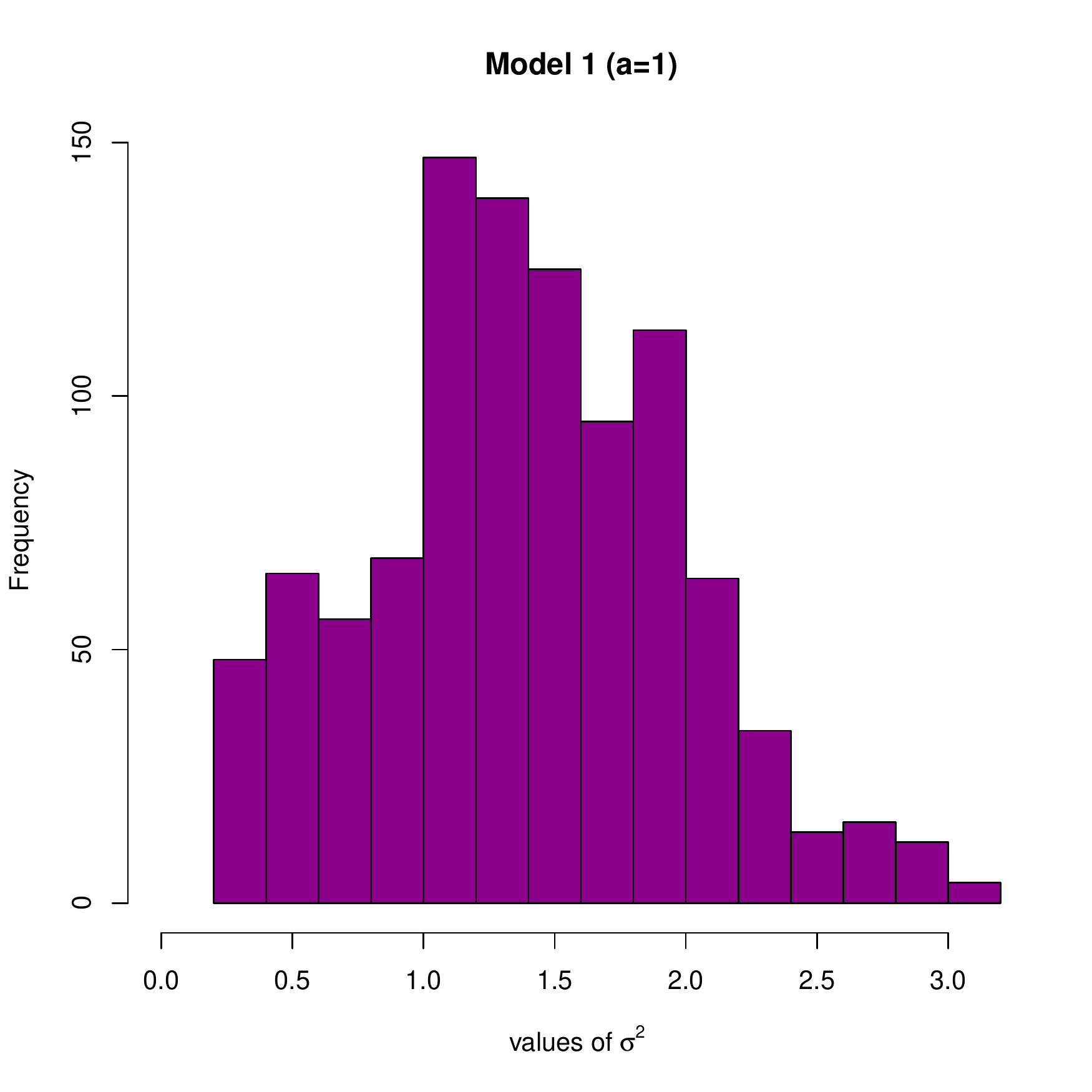}}
  \subfloat{\includegraphics[scale=0.24]{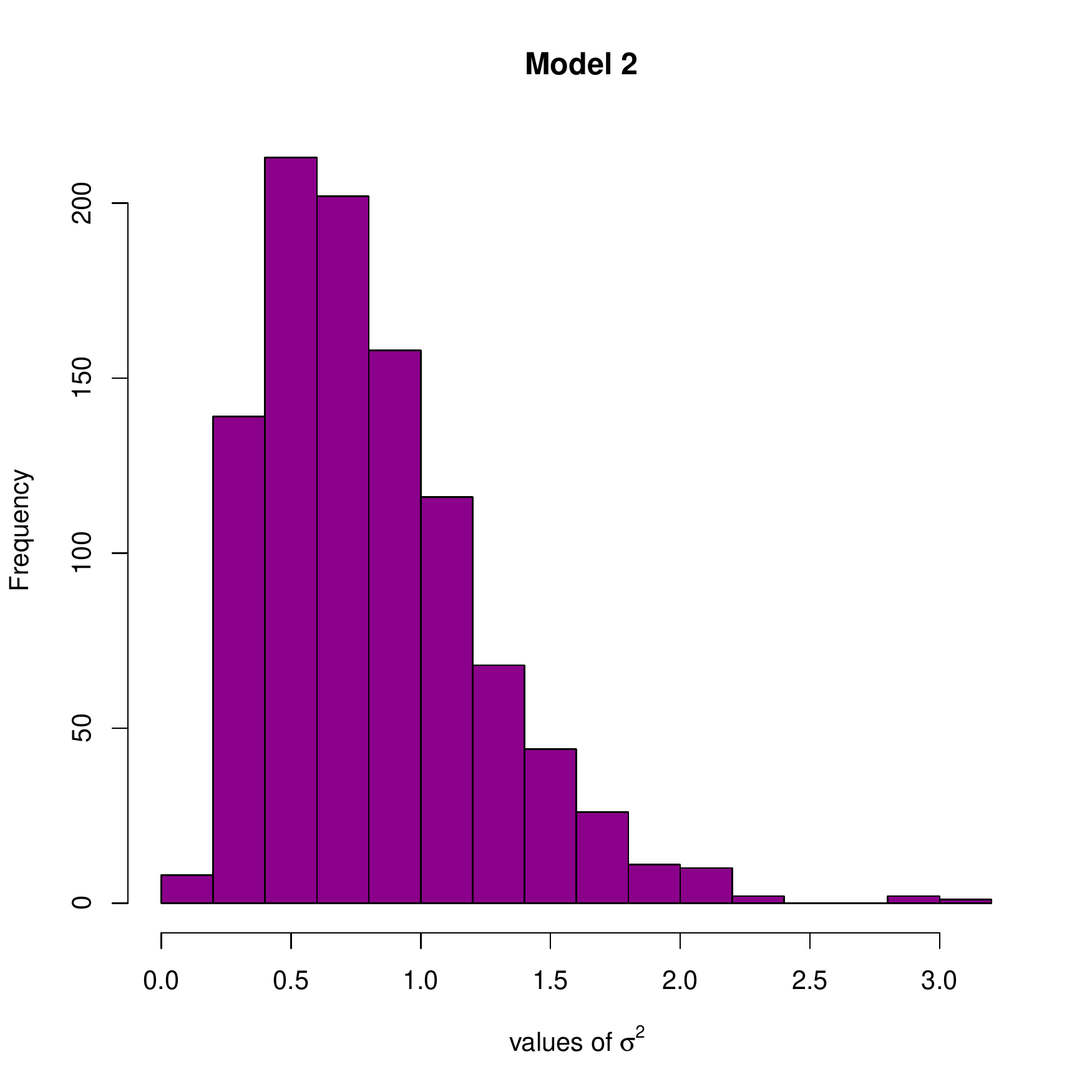}}
   \subfloat{\includegraphics[scale=0.24]{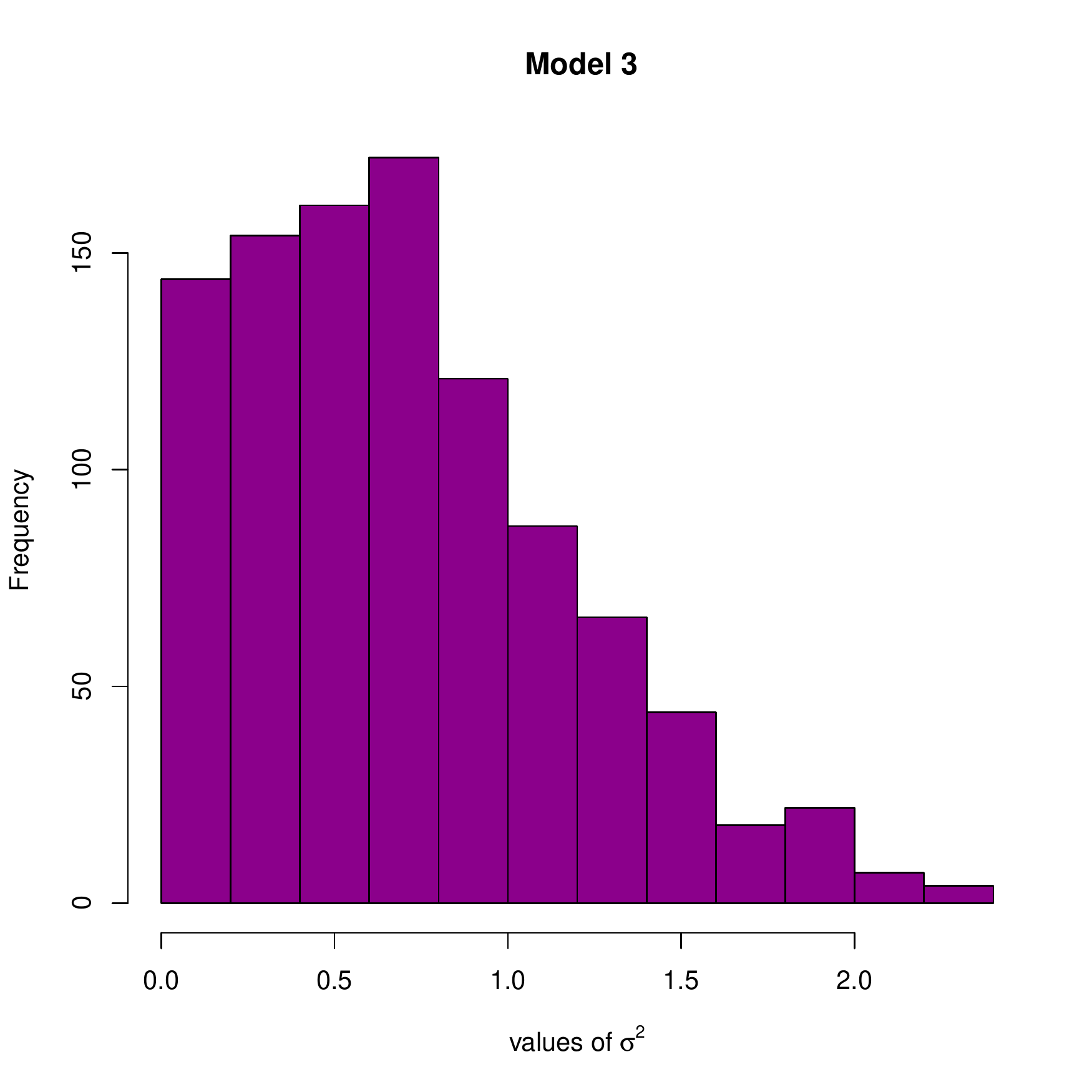}}
 \caption{Histogram of values of $\sigma^2$ in Gaussian models.}%
 \label{hist:HistSIgma}%
\end{figure}
\begin{figure}[ht]%
 \centering
 \subfloat{\includegraphics[scale=0.24]{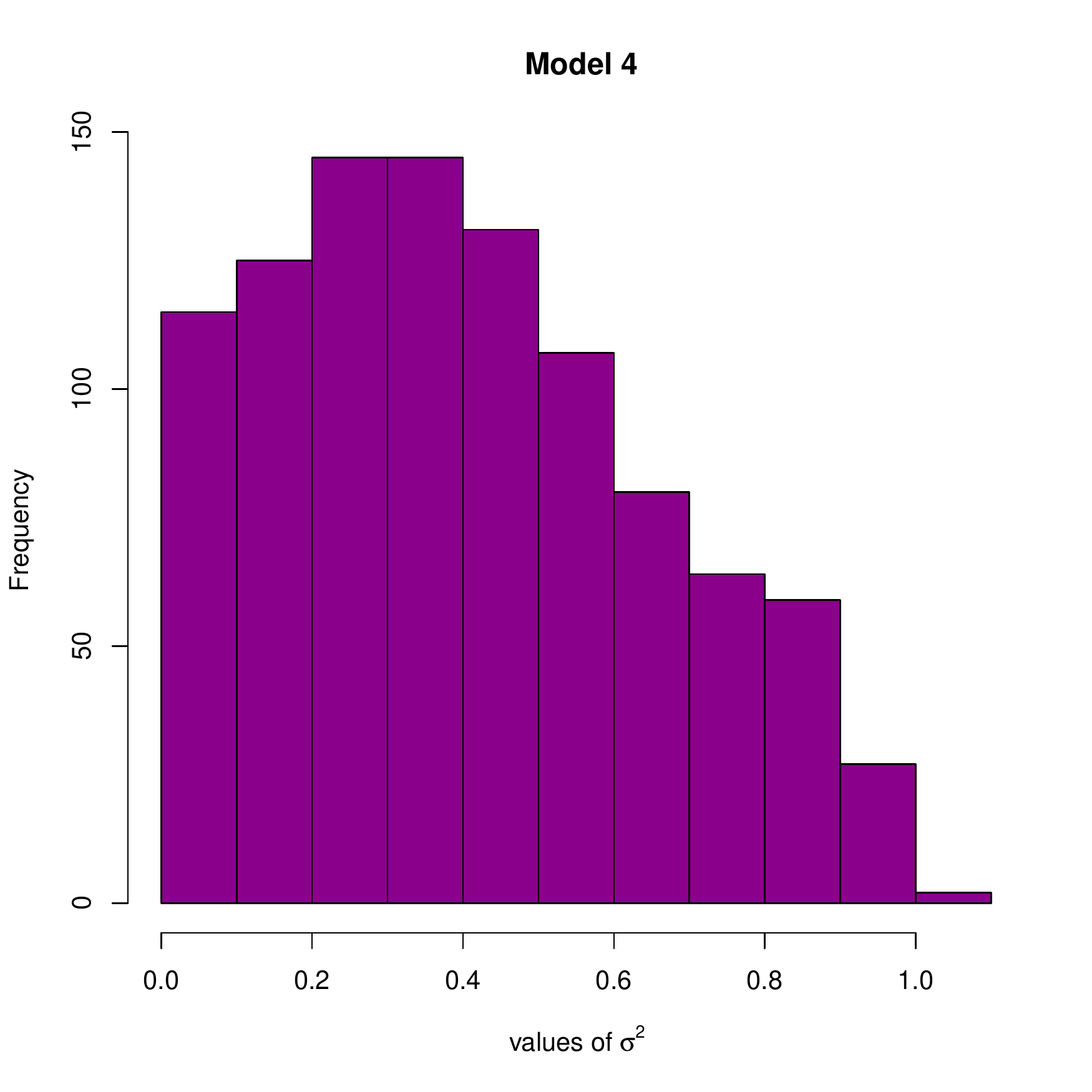}}%
 \subfloat{\includegraphics[scale=0.24]{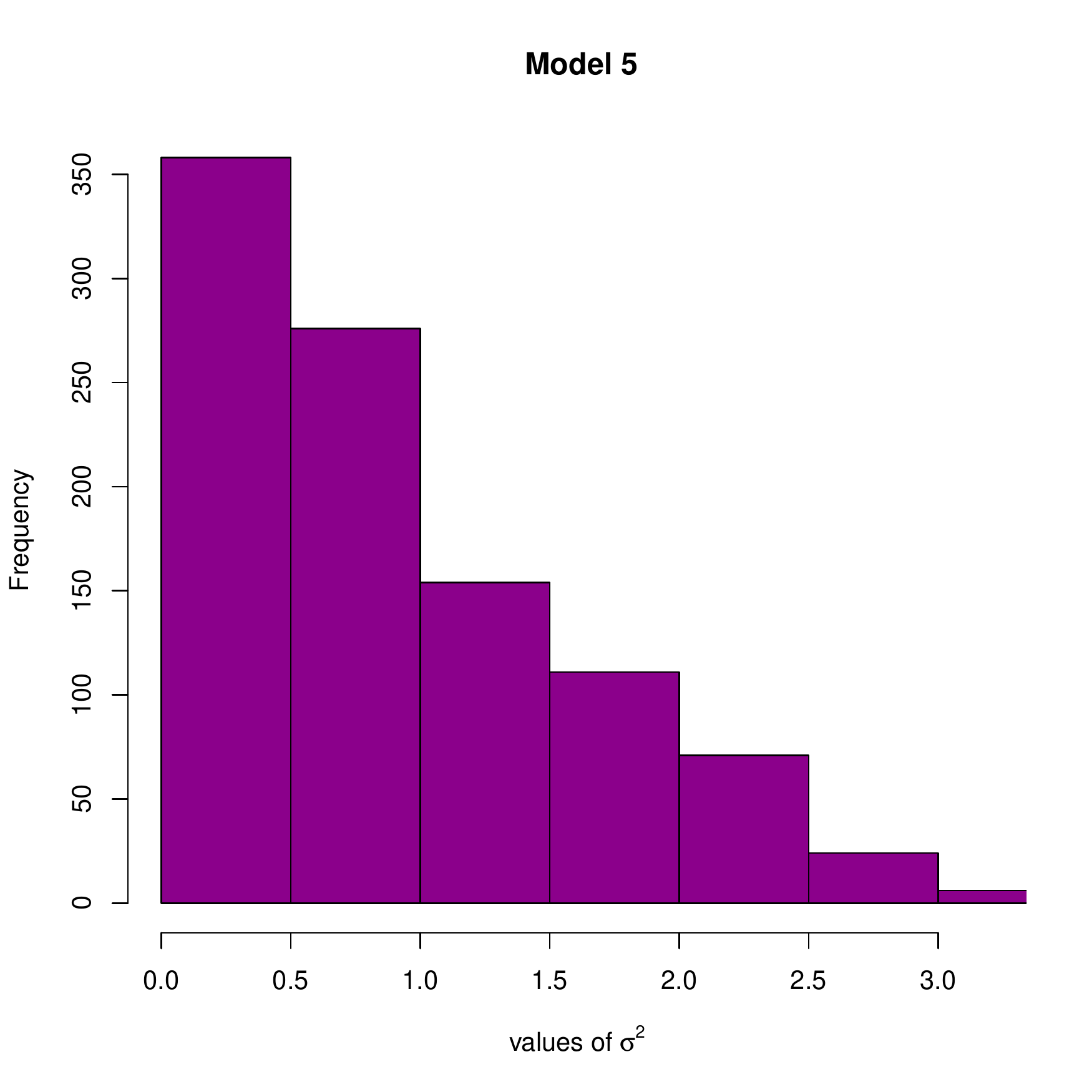}}
 \caption{Histogram of values of $\sigma^2$ in regression models for bounded responses.}%
 \label{hist:HistSIgmaRegB}%
\end{figure}
Model $1$ is a multivariate model in which the regression and variance functions are regular functions. In the case $a=1/4$, the problem of estimation of $\sigma^2$ is hard since it takes a large proportion of values smaller than $1$, while the case $a=1$ is simpler because about $76.3\%$ of the values of $\sigma^2$ are larger than $1$ and $0.04\%$ larger than $3$. Moreover, Model $2$ is also a multivariate model where we introduce higher order terms in the variance function. In this sense the estimation of the variance function is hard since in addition, there are only $28\%$ of values of $\sigma^2$ greater than $1$. In Model $3$ we consider a sparse model for the regression function where $X$ is an $N\times p$ matrix ($p$ is the number of predictors) with independent uniform entries, $\beta \in \mathbb{R}^p$ is a vector of weights, and $\xi\in \mathbb{R}^N$ is a standard Gaussian noise vector and is independent of the feature $X$. We fix $p=50$. The vector $\beta$ is chosen to be $s$-sparse where $ s< p$, that means $\beta$ has only first $s$ coordinates different from $0$; $\beta_i= \one_{\{i\leq s\}}$. Here, we choose $s=14$. In addition, the variance function in this model is less difficult. Indeed, $\sigma^2$ takes only $24.8\%$ values greater than $1$. Finally, the last two examples are two models when $Y$ is bounded. Model $4$ is bivariate regression model where the estimation of $\sigma^2$ is difficult (about $99.8\%$ of the values are less than
$1$). Lastly, considering Model $5$, the values of $\sigma^2$ are between $0$ and $3.12$. There are $36.6\%$ of values that are larger than $1$. From this perspective the estimation of the variance function is less complicated. However, the presence of higher order terms makes the problem harder.

\subsection{Benefit of aggregation}
\label{subsec:BenefitAggre}
In this section we improve the classical methods based on residual-based approach by considering aggregation. In the same time we compare \texttt{MS} and \texttt{C} aggregation.
\subsubsection{Machines and simulation scheme}
\label{sec:machines}
The construction of the aggregates $\hat{\sigma}^{2}_{\texttt{MS}}$ and $\hat{\sigma}^{2}_{\texttt{C}}$ is described in Sections~\ref{sec:AggreMS} and~\ref{sec:AggreCM}. We recall that we focus on the residual-based method to compute the candidates of the variance function $\sigma^2$. One of the advantages of using the aggregation approach is that the collection of candidates is chosen by the practitioner and can be arbitrary. We build three dictionaries $\mathcal{F}=\{\hat{f}_{s}\}_{s=1}^{12}$, $\mathcal{G}_{1}=\left\{\hat{\sigma}^{2}_{\hat{s},12}\right\}_{m=1}^{12}$ and $\mathcal{G}_{2}=\{\hat{\sigma}^{2}_{\hat{\lambda},j}\}_{j=1}^{12}$ that contain $12$ machines each: the random forest with different number of trees (ntree=$50$, $150$, $500$), the $k$NN  with different values of $k$ ($k=7, 13, 22$), the Lasso with different values of tuning parameter ($\lambda= 0.5, 2$), the Ridge  with different values of  tuning parameter ($\lambda= 0.9, 3$), regression tree and the Elastic Net regression with a penalty term $\lambda=1$ and a parameter $\alpha=0.6$ that compromises between the $\ell_1$ and the $\ell_2$ terms in the penalty. The first dictionary is exploited to compute the aggregates $\hat{f}_{\texttt{MS}}$ and $\hat{f}_{\texttt{C}}$ while the last two are used to calculate respectively $\hat{\sigma}^{2}_{\texttt{MS}}$ and $\hat{\sigma}^{2}_{\texttt{C}}$ with those $12$ machines. 
For the $12$ algorithms, we use the following R packages:
\begin{itemize}
\item  Regression tree (R package \texttt{tree}, \cite{Ripley19});
\item  $k$-nearest neighbors regression (R package \texttt{FNN}, \cite{Li19});
\item  RandomForest regression (R package \texttt{randomForest}, \cite{LiawWiener02});
\item  Lasso regression (R package \texttt{glmnet}, \cite{FriedmanHastieTibshirani10});
\item  Ridge regression (R package \texttt{glmnet});
\item Elastic Net regression (R package \texttt{glmnet}).
\end{itemize}
Other parameters are set by default. In addition to that, we use \texttt{Optim} function in \texttt{R}  which is based on method \texttt{BFGS} to compute $\hat{\lambda}$ and $\hat{\beta}$.
Now, we evaluate the performances of $\hat{\sigma}^{2}_{\texttt{MS}}$ and $\hat{\sigma}^{2}_{\texttt{C}}$ on previous models. Besides, we provide estimation of  the  $L^2$-error for $\hat{\sigma}^{2}_{\texttt{MS}}$ and $\hat{\sigma}^{2}_{\texttt{C}}$ and repeat independently $L=100$ times the following steps 
\begin{enumerate}
\item simulate three datasets $\mathcal{D}_n$, $\mathcal{D}_{N}$ and $\mathcal{D}_{T}$ with $n \in \{100,1000\}$, $N\in \{100,1000\}$  and $T=1000$;
\item based on $\mathcal{D}_{n}$, we compute the dictionary $\mathcal{F}$, and then based on $\mathcal{D}_N$, we compute the aggregates $\hat{f}_{\texttt{MS}}$ (that is $\hat{s}$) and $\hat{f}_{\texttt{C}}$ (that is $\hat{\lambda}$) of the  regression function $f^{*}$ provided in Eqs~\eqref{est: hatfMS} and~\eqref{est:fhatCM};
\item  based on $\mathcal{D}_{n}$ and $\hat{f}_{\texttt{MS}}$ (resp. $\hat{f}_{\texttt{C}}$), we compute the collection $\mathcal{G}_1$ (resp. $\mathcal{G}_2$) and we calculate $\hat{\sigma}^{2}_{\texttt{MS}}$ and $\hat{\sigma}^{2}_{\texttt{C}}$ on $\mathcal{D}_{N}$;
\item based on $D_n\cup D_N$: firstly, we compute the collection $\mathcal{F}$\footnote{Note that this set of estimators differ from the dictionary computed in step 2. since it is computed in the whole data $D_n\cup D_N$. We abuse in the notation to avoid extra notation that are irrelevant for the understanding.}; secondly, for each estimate $\hat{f}_s$ in $\mathcal{F}$ we calculate the estimators $\{\hat{\sigma}^{2}_{s,m}\}_{1 \leq  m\leq 12}$ of $\sigma^2$ corresponding to the $12$ procedures in $\mathcal{F}$ and we pick the best estimator among them;
\item finally, over $\mathcal{D}_{T}$, we compute the empirical $L^2$-error of the aggregates $\hat{\sigma}^{2}_{\texttt{MS}}$ and $\hat{\sigma}^{2}_{\texttt{C}}$ and the best estimator computed in step $4$. More precisely, we compute the following quantity $\widehat{\Err}(\hat{\sigma}^{2})=\frac{1}{T}\sum_{i=1}^{T}\left(\hat{\sigma}^{2}(X_i)-\sigma^{2}(X_i)\right)^{2}$.
\end{enumerate}
From these experiments, we compute the means and standard deviations of both empirical risks $\widehat{\Err}$ for $\hat{\sigma}^{2}_{\texttt{MS}}$, $\hat{\sigma}^{2}_{\texttt{C}}$ and the best estimator in step $4$  and we display the boxplot of the empirical $L^2$-error. 

\subsubsection{Results}
We present our results in Figures~\ref{fig:boxplot_Mod1}-\ref{fig:boxplot_Mod5} and Tables~\ref{Tab:1} and~\ref{Tab:2}. We make several observations. First, the convex aggregation method is better than the model selection aggregation method in all models. 
Second, we notice that when $n$ and $N$ are enough, the \texttt{MS}-estimator $\hat{\sigma}^{2}_{\texttt{MS}}$  and the \texttt{C}-estimator $\hat{\sigma}^{2}_{\texttt{C}}$ have similar performance, that is close to the performance of the best estimator. 
These results reflect our theory: the consistency of \texttt{MS}-estimator and the \texttt{C}-estimator. 
Third, we observe that the empirical $L^2$-error of $\hat{\sigma}^{2}_{\texttt{MS}}$ and $\hat{\sigma}^{2}_{\texttt{C}}$ decreases faster in the simpler models (with respect to the estimation of the variance function) when $n$ and $N$ increase (see the evolution of the boxplots in Figures~\ref{fig:boxplot_Mod1_a} and~\ref{fig:boxplot_Mod5} as compared to Figures~\ref{fig:boxplot_Mod1} and~\ref{fig:boxplot_Mod3}. 
In addition, our numerical results highlight an interesting fact: when we split data, it is advantageous to put more data in the second dataset $\mathcal{D}_N$ used in the aggregation step. Indeed, it seems as illustrated in Table~\ref{Tab:2} that the methods have better performance for large samples $\mathcal{D}_N$ is all cases. As an example, the mean error in Model~1 with $a=1$ for C-aggregation is $0.33$ when $n=1000$ and $N=100$ and $0.26$ when $n=100$ and $N=1000$.
 \begin{table}[!ht]
\centering
\footnotesize{
{\setlength{\tabcolsep}{3pt}
\vspace*{0.25cm}
\begin{tabular}{||l || c | c | c || c | c | c|| }
\multicolumn{1}{c}{} & \multicolumn{3}{c}{{$n=N=100$ }} &  \multicolumn{3}{c}{{$n=N=1000$}}\\ \hline
\multicolumn{1}{c}{}  &  \multicolumn{1}{c}{{\texttt{C}}} & \multicolumn{1}{c}{{\texttt{MS}}} & \multicolumn{1}{c}{{\texttt{Best}}} & \multicolumn{1}{c}{{\texttt{C}}} & \multicolumn{1}{c}{{\texttt{MS}}} & \multicolumn{1}{c}{{\texttt{Best}}}
\\ \hline \noalign{\smallskip}
Model & $\widehat{\Err}$ &  $\widehat{\Err}$  &  $\widehat{\Err}$  & $\widehat{\Err}$ & $\widehat{\Err}$  &  $\widehat{\Err}$  \\ \hline \noalign{\smallskip}
Model $1$ ($a=0.25$) & 0.028 (0.018)  & 0.031 (0.023)  & 0.013 (0.003)   &  0.011 (0.004) &  0.014 (0.003)  & 0.011 (0.001)  \\
Model $1$ ($a=1$) & 0.407 (0.214)  & 0.428 (0.279) & 0.200 (0.44)  & 0.155 (0.044) & 0.200 (0.040)  & 0.164 (0.013)\\
Model $2$ & 0.247 (0.133) & 0.272 (0.180) & 0.110 (0.025)  & 0.106 (0.046)  & 0.100 (0.093)  & 0.070 (0.010)\\
Model $3$ & 0.287 (0.092) & 0.302(0.125)  & 0.218 (0.019) & 0.194 (0.021) & 0.198 (0.044) & 0.164 (0.011)  \\
Model $4$ & 0.032 (0.027) & 0.034 (0.036) & 0.010 (0.005)  & 0.010 (0.005) & 0.011 (0.003)  & 0.009 (0.001)  \\
Model $5$ & 0.382 (0.116) & 0.405 (0.168) & 0.264 (0.032)  & 0.209 (0.040) & 0.223 (0.024)  & 0.178 (0.016)  \\
\hline
\end{tabular}}}
\caption{Average and standard deviation of the empirical $L^2$-error of the three estimators with $n=N$.}
\label{Tab:1}
\end{table}
\begin{table}[!ht]
\centering
\footnotesize{
{\setlength{\tabcolsep}{3pt}
\vspace*{0.25cm}
\begin{tabular}{||l || c | c | c || c | c | c|| c |}
\multicolumn{1}{c}{}  &  \multicolumn{3}{c}{{$n=1000$ $N=100$}}& \multicolumn{3}{c}{{$n=100$, $N=1000$}}\\ \hline
\multicolumn{1}{c}{}  &  \multicolumn{1}{c}{{\texttt{C}}} & \multicolumn{1}{c}{{\texttt{MS}}} & \multicolumn{1}{c}{{\texttt{Best}}} & \multicolumn{1}{c}{{\texttt{C}}} & \multicolumn{1}{c}{{\texttt{MS}}} & \multicolumn{1}{c}{{\texttt{Best}}} \\ \hline \noalign{\smallskip}
Model & $\widehat{\Err}$ &  $\widehat{\Err}$  &  $\widehat{\Err}$  & $\widehat{\Err}$ & $\widehat{\Err}$  &  $\widehat{\Err}$ & \\ \hline \noalign{\smallskip}
Model $1$ ($a=0.25$) & 0.023 (0.015)  & 0.028 (0.023)   & 0.012 (0.002) & 0.018 (0.008)  & 0.019 (0.006) &  0.012 (0.002)  \\
Model $1$ ($a=1$) & 0.335 (0.265)  & 0.381 (0.343) & 0.170 (0.014)  & 0.262 (0.090) & 0.278 (0.081)  & 0.169 (0.018) \\
Model $2$ & 0.193 (0.132) & 0.227 (0.189)  & 0.074 (0.013) & 0.159 (0.055)  & 0.148 (0.054) & 0.073 (0.010)  \\
Model $3$ & 0.252 (0.082) & 0.278 (0.149)  & 0.180 (0.015)  & 0.259 (0.029)  & 0.270 (0.035)  & 0.179 (0.015)  \\
Model $4$ & 0.021 (0.014) & 0.026 (0.027)  & 0.009 (0.002) & 0.019 (0.012) & 0.017 (0.015)  & 0.009 (0.002)  \\
Model $5$ & 0.295 (0.144)  & 0.336 (0.209)  & 0.195 (0.019)  & 0.313 (0.079) & 0.317 (0.095)  & 0.194 (0.015)  \\
\hline
\end{tabular} }}
\caption{Average and standard deviation of the empirical $L^2$-error of the three estimators with $n\neq N$.}
\label{Tab:2}
\end{table}
\begin{figure}[ht]%
 \centering
 \subfloat[$n=N=100$]{\includegraphics[scale=0.24]{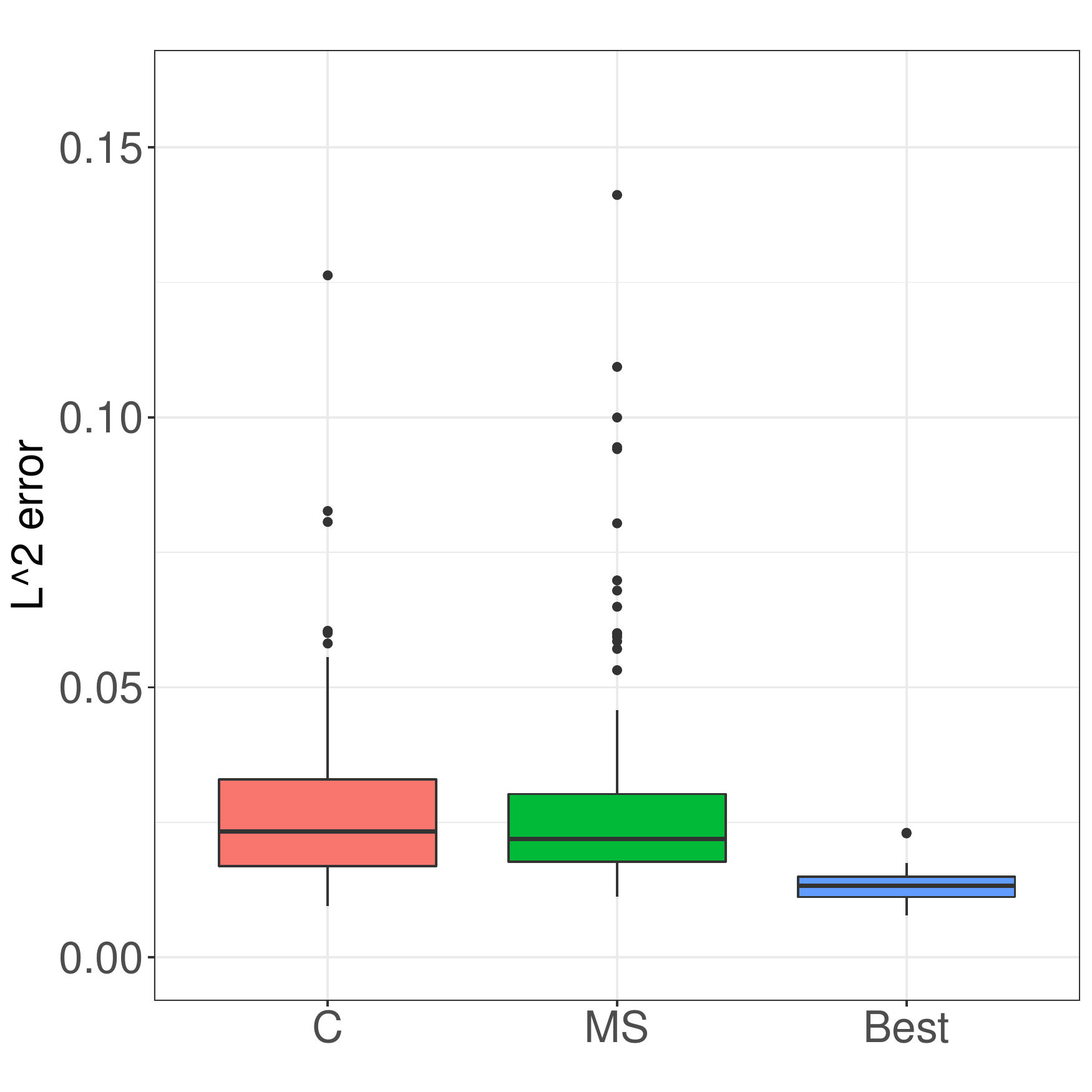}}%
 \subfloat[$n=1000, N=100$]{\includegraphics[scale=0.24]{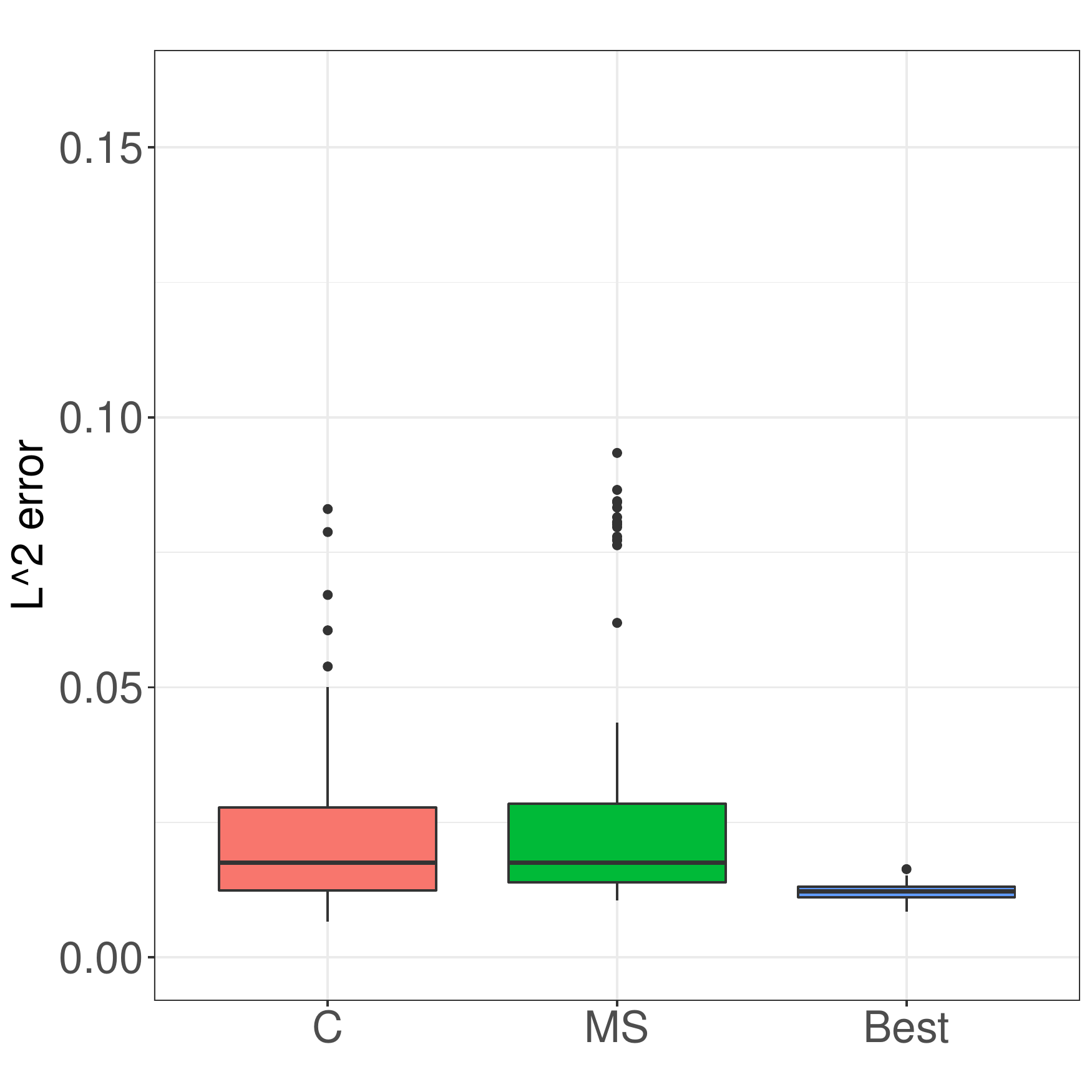}}
 \subfloat[$n=100, N=1000$]{\includegraphics[scale=0.24]{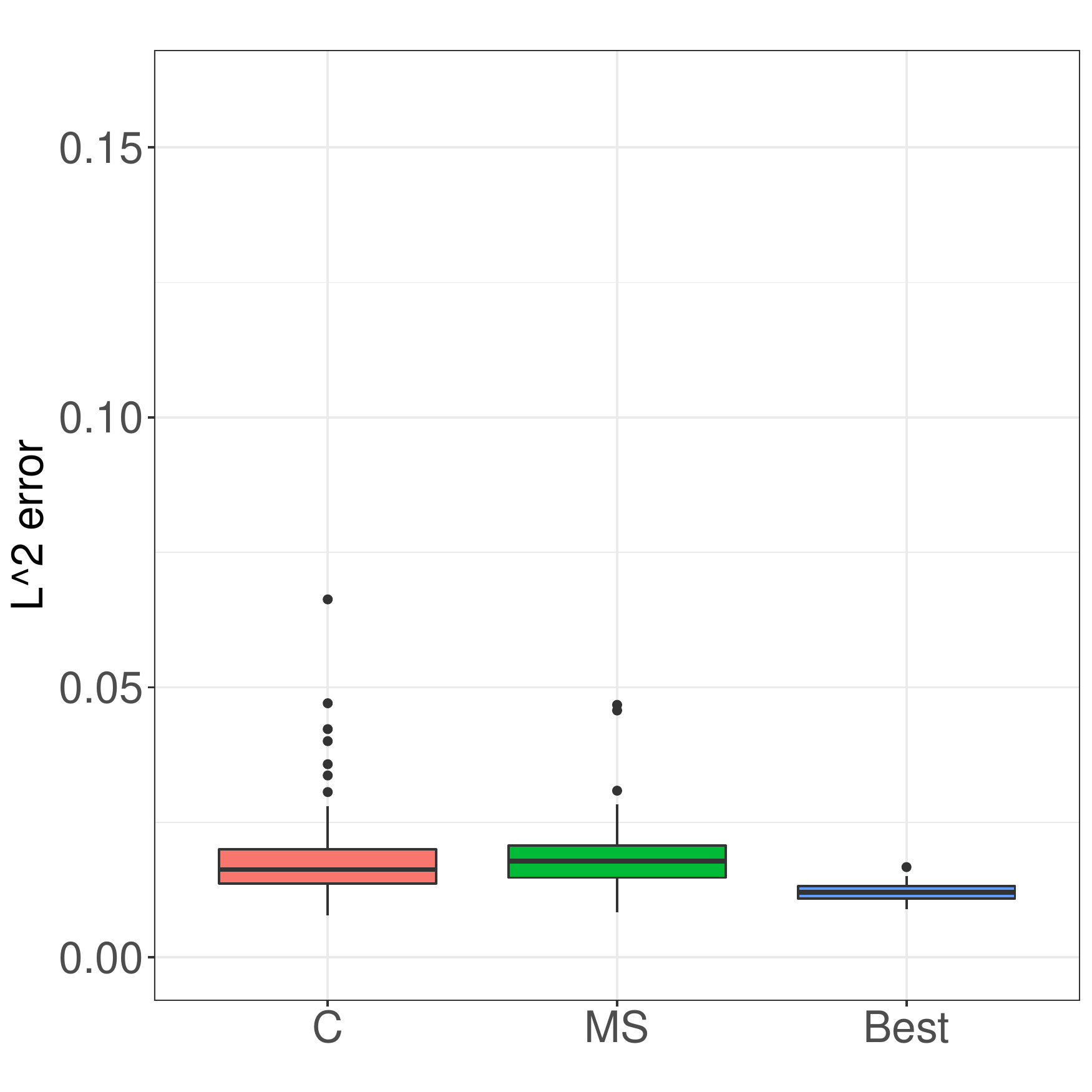}}%
  \subfloat[$n=N=1000$]{\includegraphics[scale=0.24]{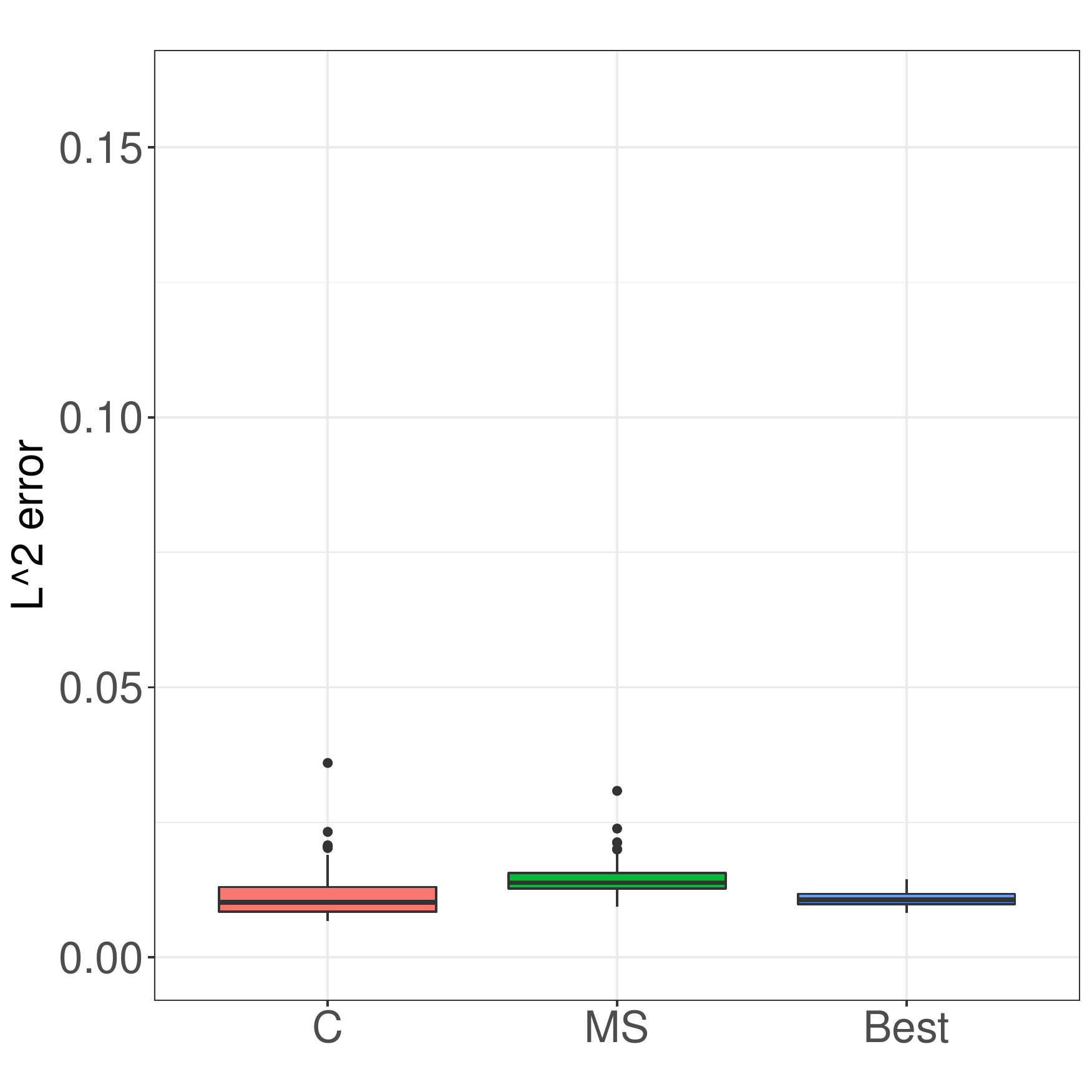}}
 \caption{Boxplot of the  empirical $L^2$-error  of the estimators in Model $1$ ($a=0.25$)}%
 \label{fig:boxplot_Mod1}%
\end{figure}

\begin{figure}[ht]%
 \centering
 \subfloat[$n=N=100$]{\includegraphics[scale=0.24]{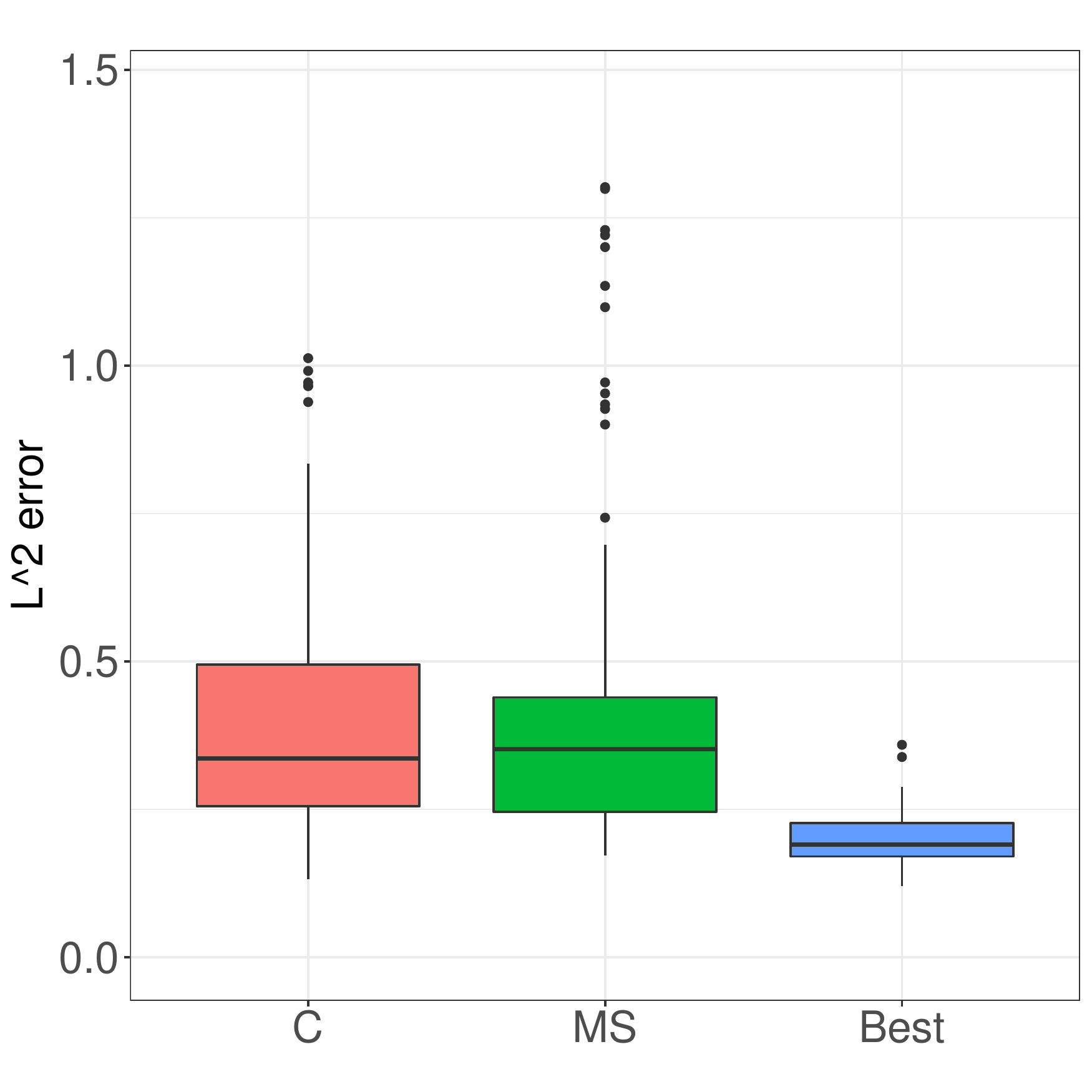}}%
 \subfloat[$n=1000, N=100$]{\includegraphics[scale=0.24]{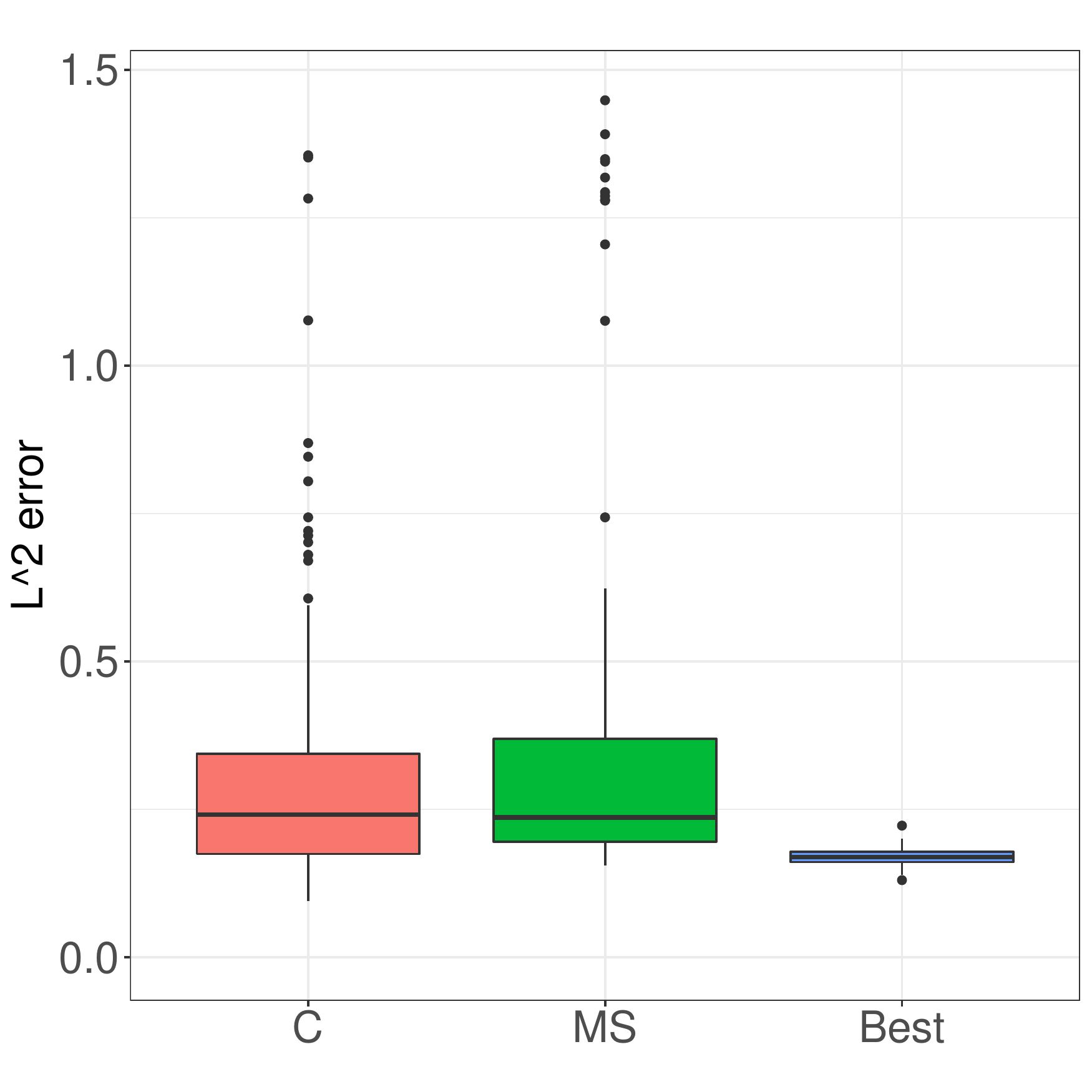}}
 \subfloat[$n=100, N=1000$]{\includegraphics[scale=0.24]{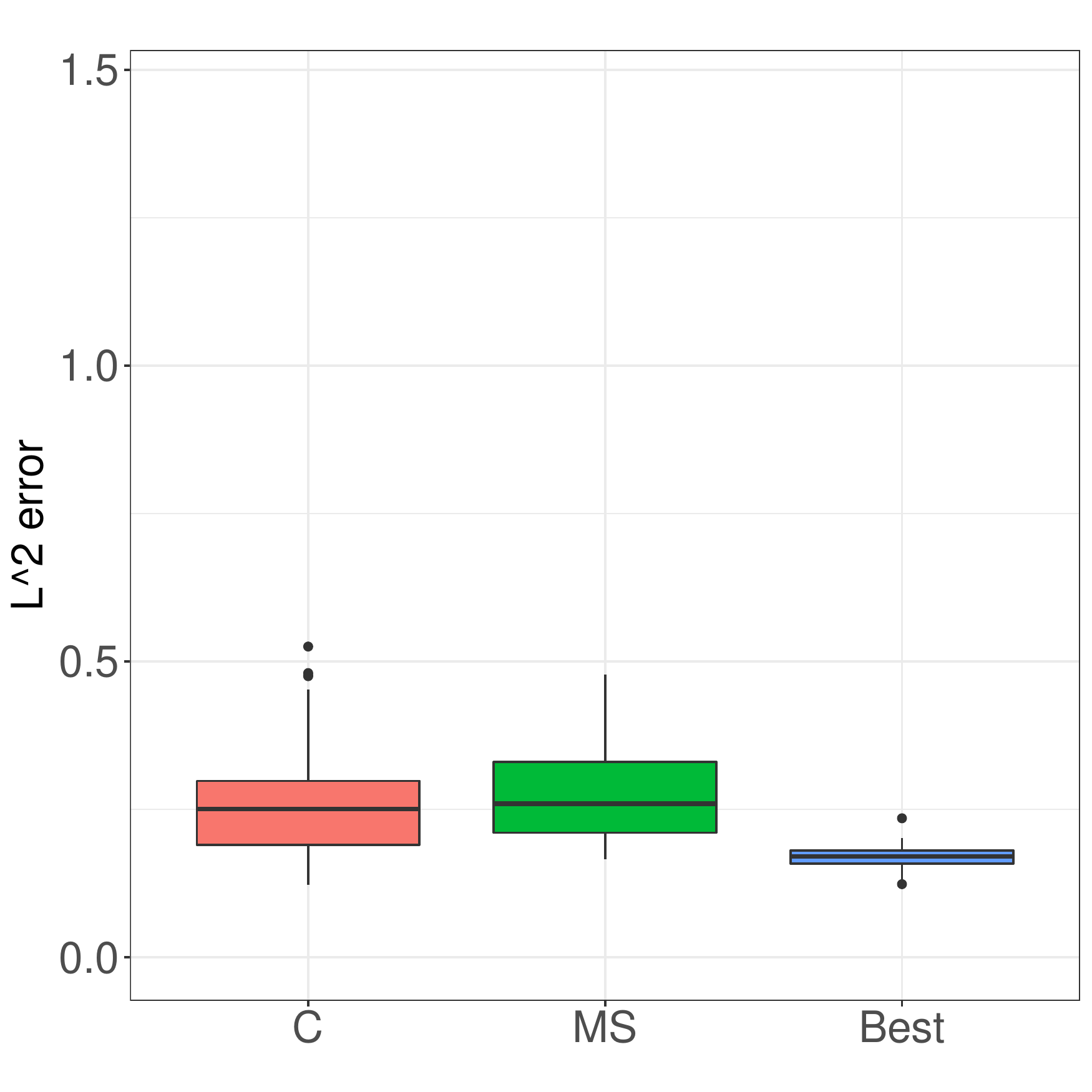}}%
  \subfloat[$n=N=1000$]{\includegraphics[scale=0.24]{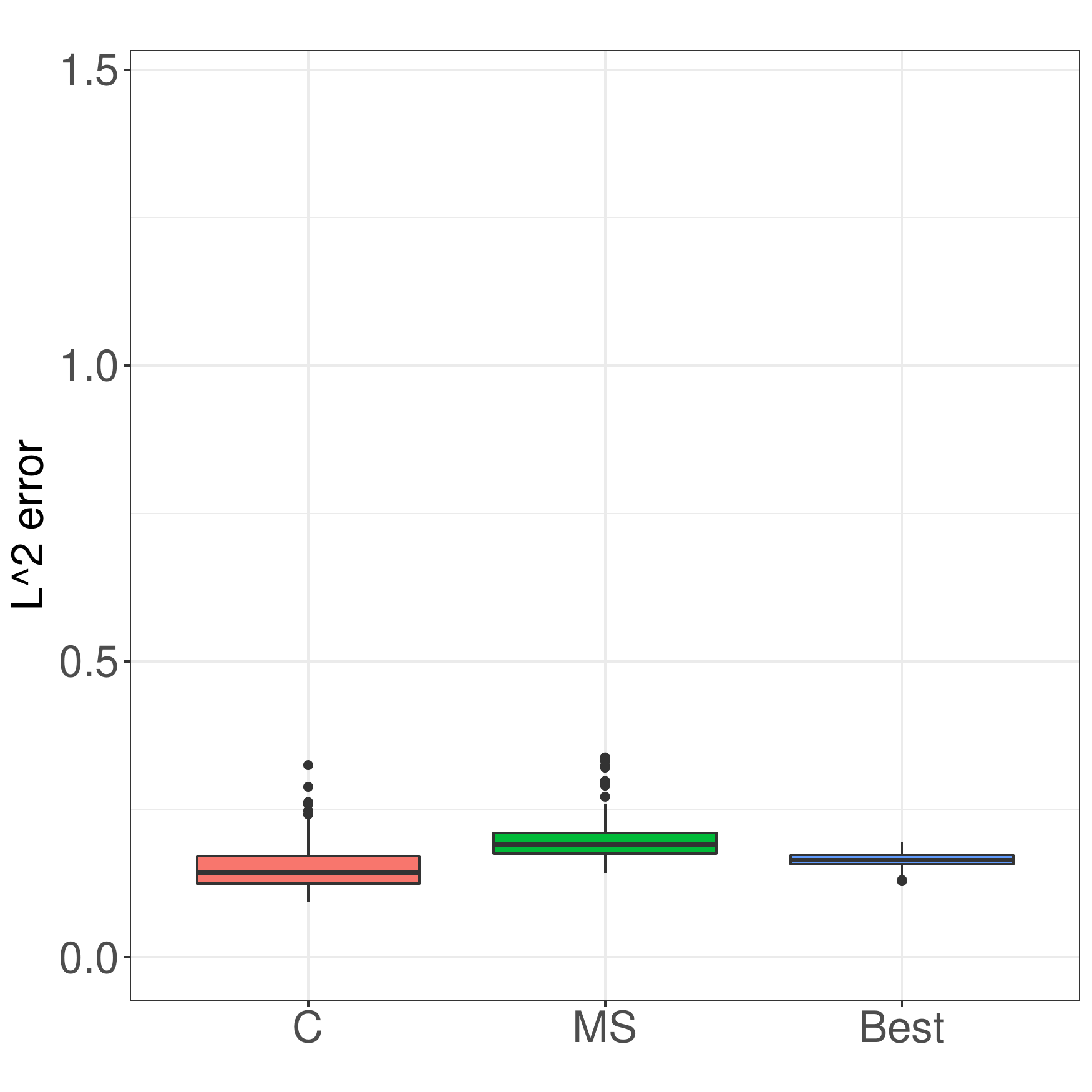}}
 \caption{Boxplot of the  empirical $L^2$-error  of the estimators in Model $1$ ($a=1$)}%
 \label{fig:boxplot_Mod1_a}%
\end{figure}
\begin{figure}[ht]%
 \centering 
 \subfloat[$n=N=100$]{\includegraphics[scale=0.24]{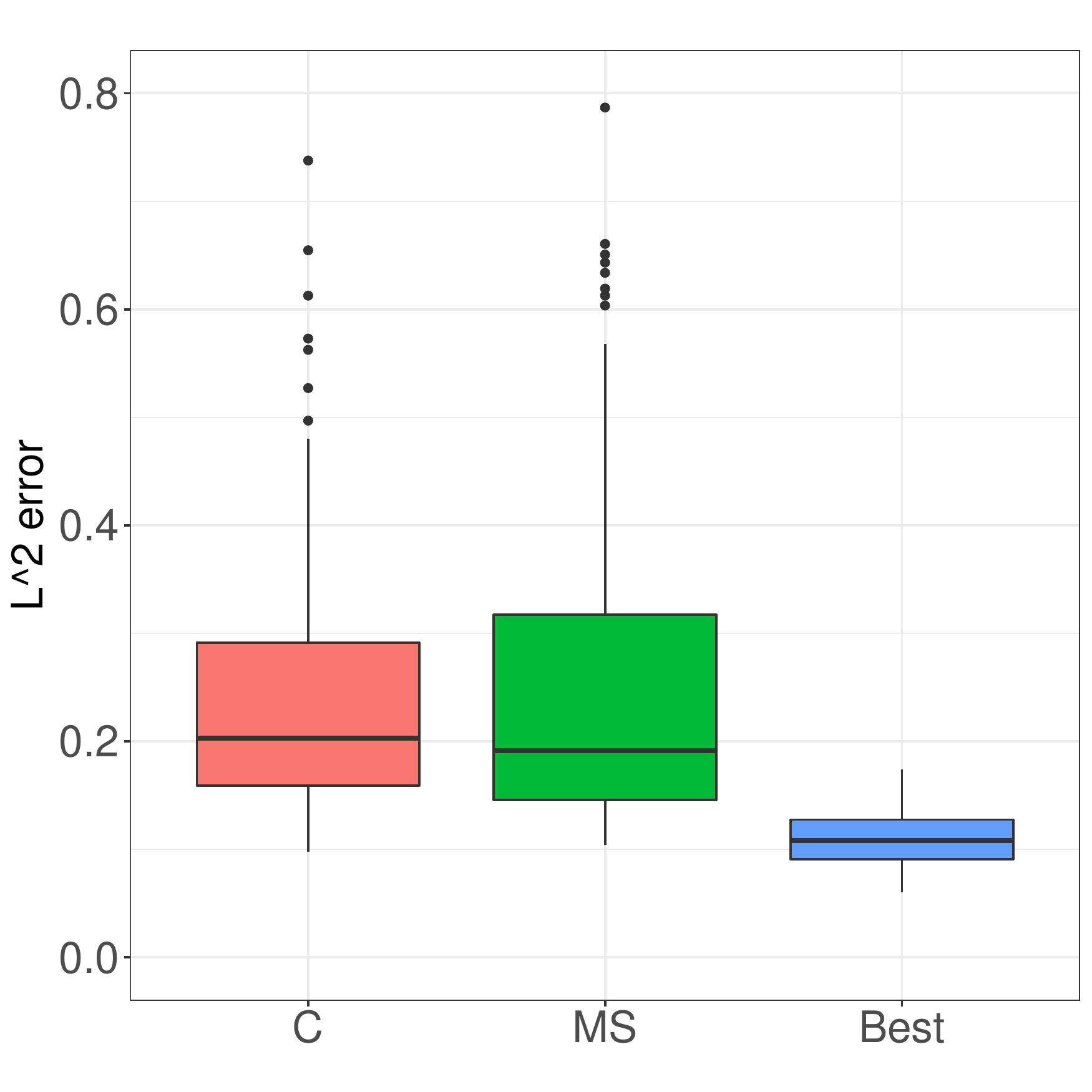}}%
 \subfloat[$n=1000, N=100$]{\includegraphics[scale=0.24]{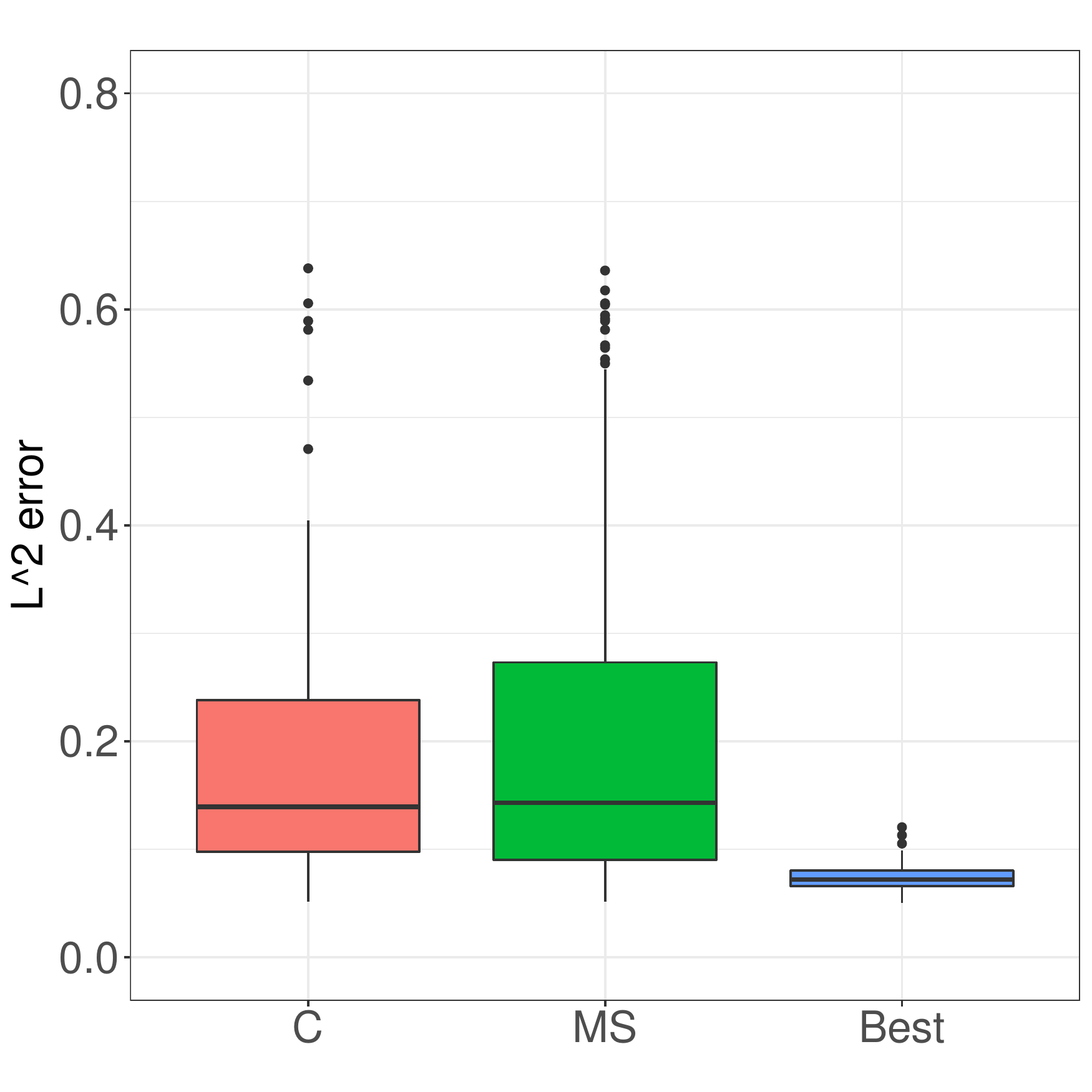}}
 \subfloat[$n=100, N=1000$]{\includegraphics[scale=0.24]{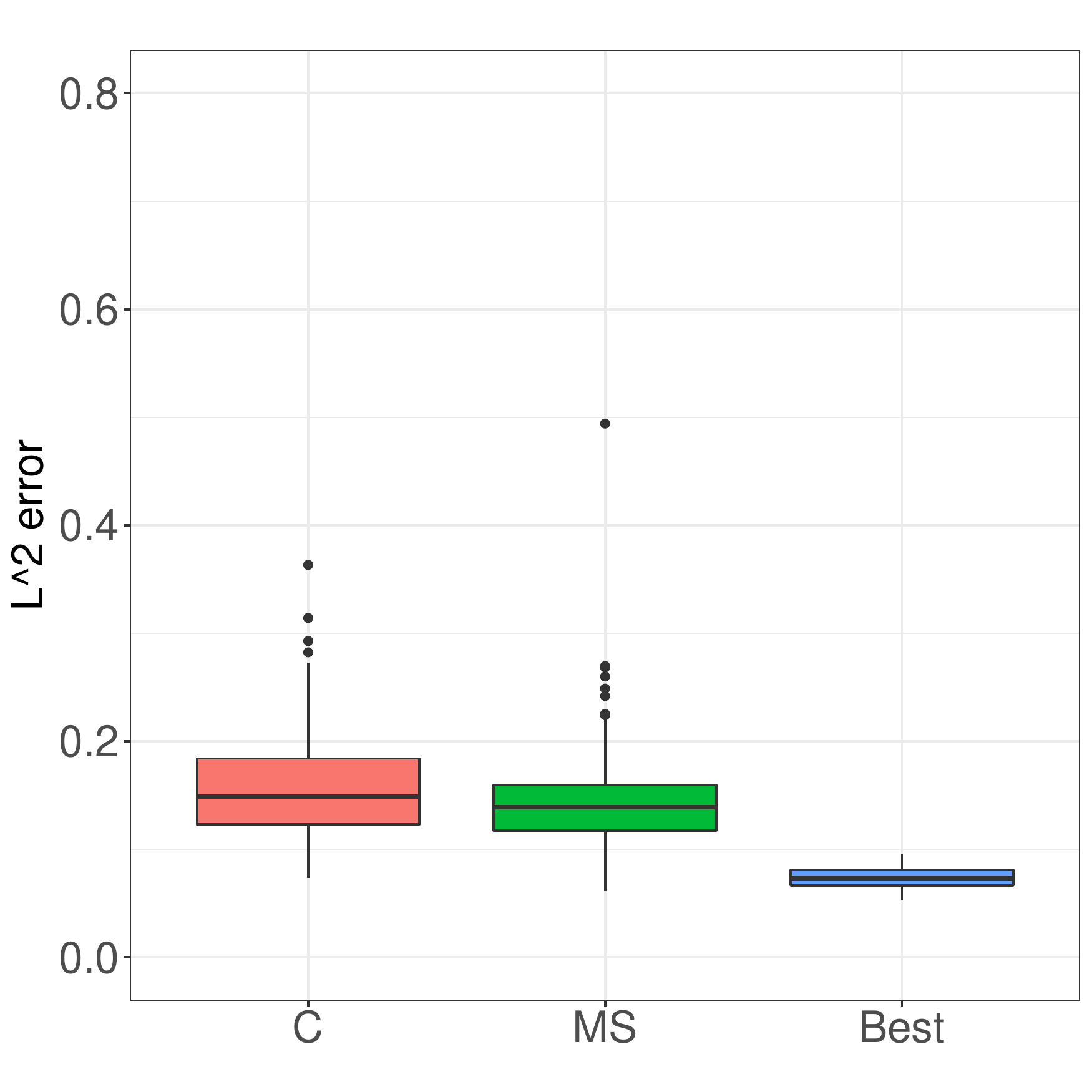}}%
  \subfloat[$n=N=1000$]{\includegraphics[scale=0.24]{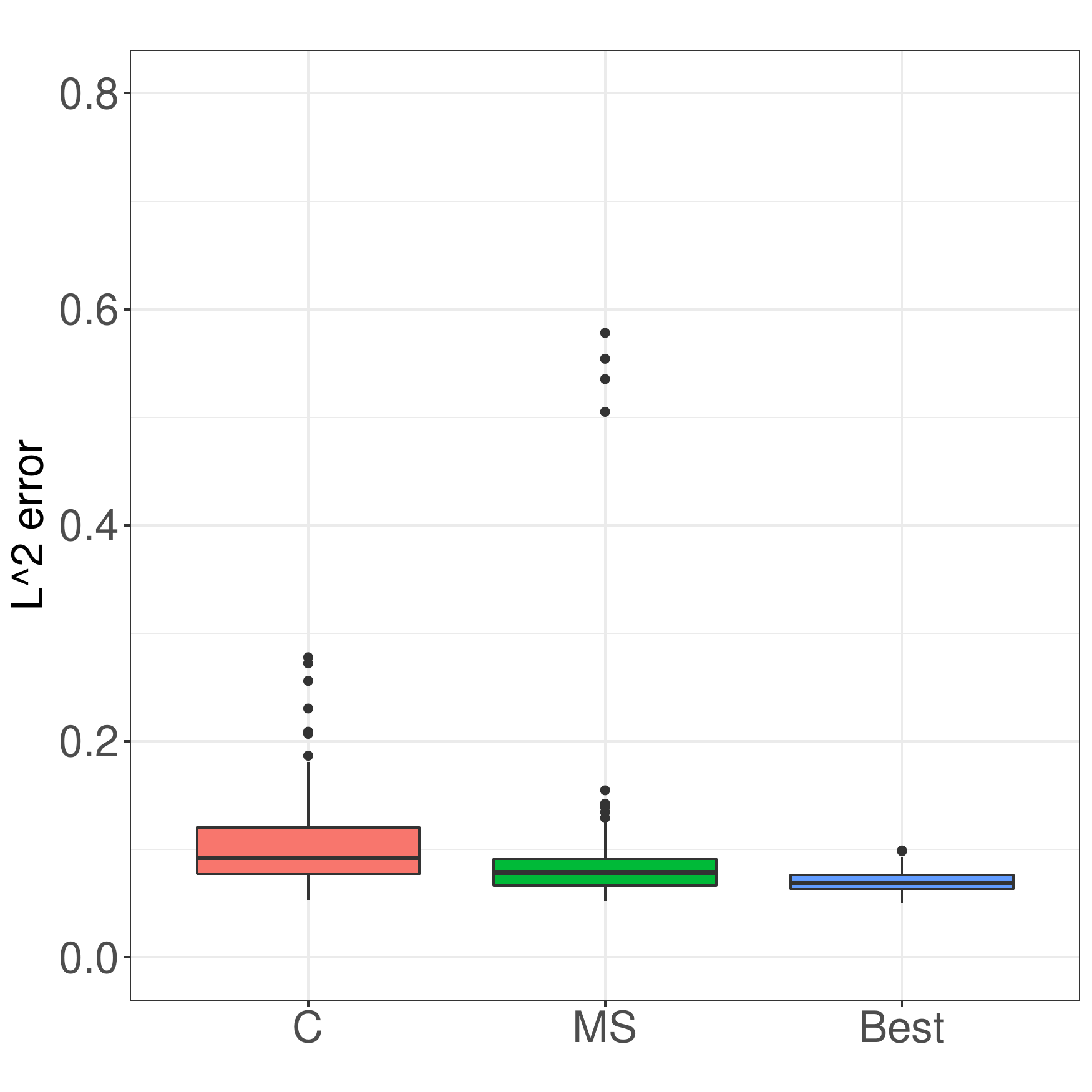}}
 \caption{Boxplot of the empirical $L^2$-error  of the estimators in Model $2$}%
 \label{fig:boxplot_Mod2}%
\end{figure}

\begin{figure}[ht]%
 \centering
 \subfloat[$n=N=100$]{\includegraphics[scale=0.24]{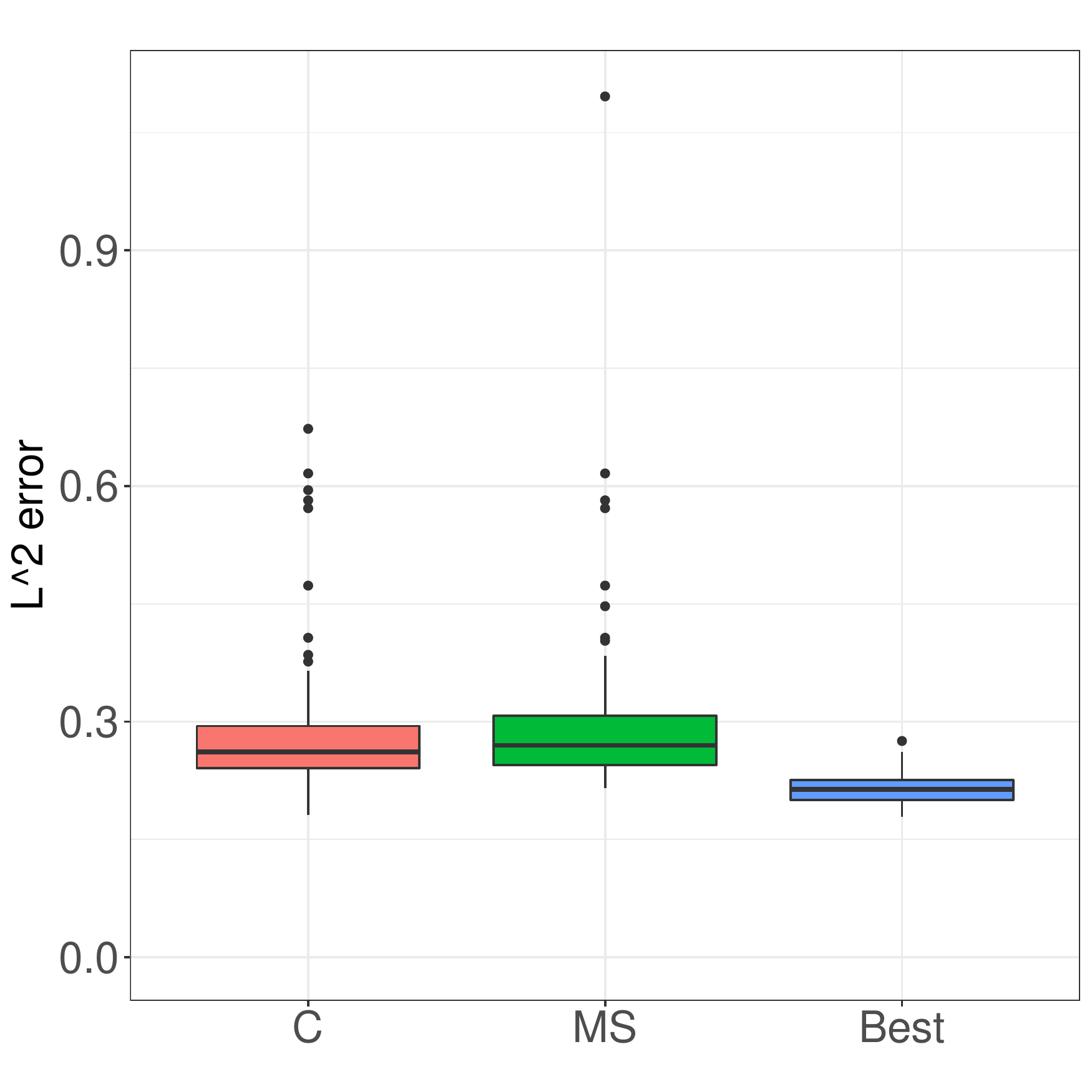}}%
 \subfloat[$n=1000, N=100$]{\includegraphics[scale=0.24]{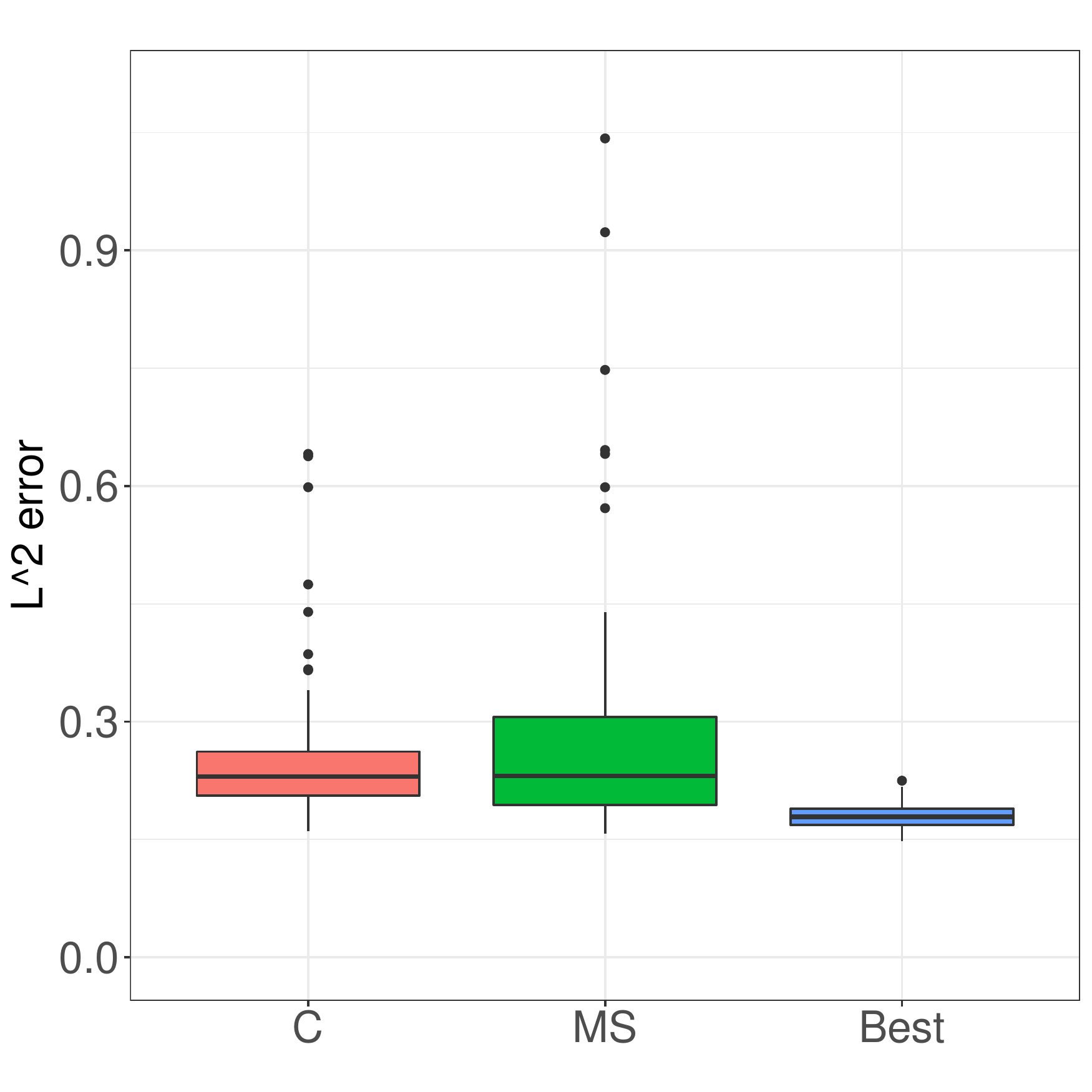}}
  \subfloat[$n=100, N=1000$]{\includegraphics[scale=0.24]{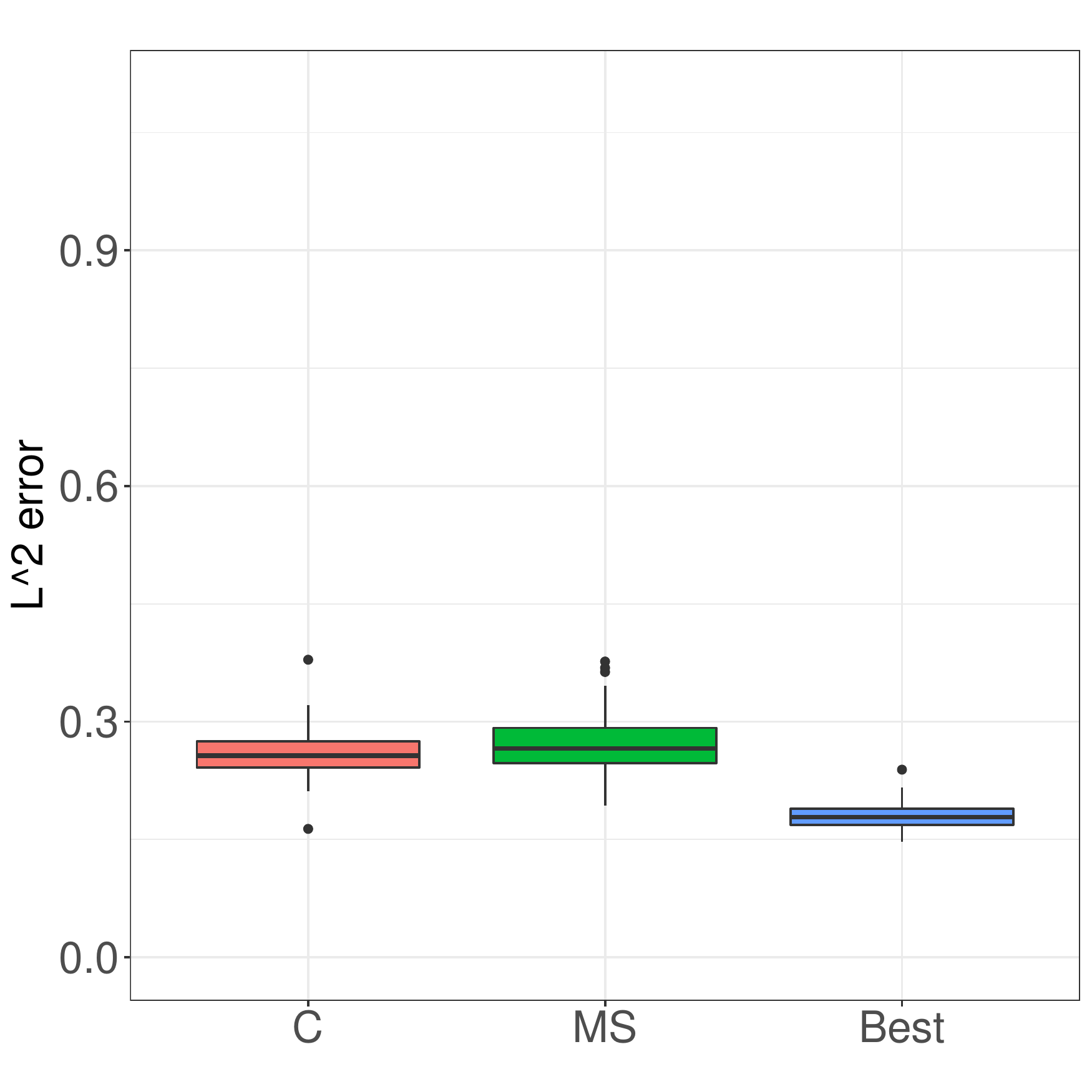}}
   \subfloat[$n=N=1000$]{\includegraphics[scale=0.24]{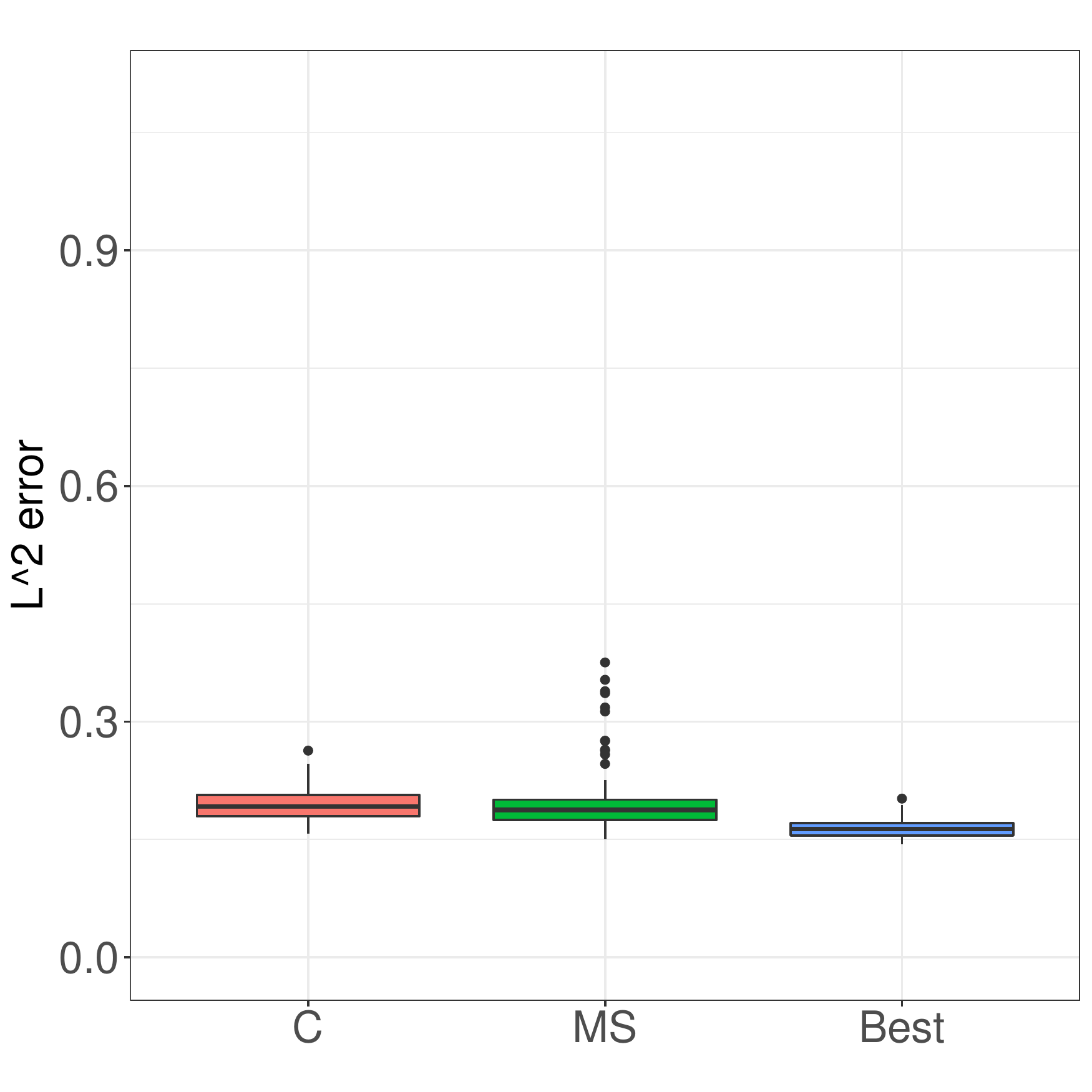}}
 \caption{Boxplot of the  $L^2$-error  of the estimators in sparse model when $p=50$, and $s=14$.}%
 \label{fig:boxplot_Mod3}%
\end{figure}
\begin{figure}[ht]%
 \centering
 \subfloat[$n=N=100$]{\includegraphics[scale=0.24]{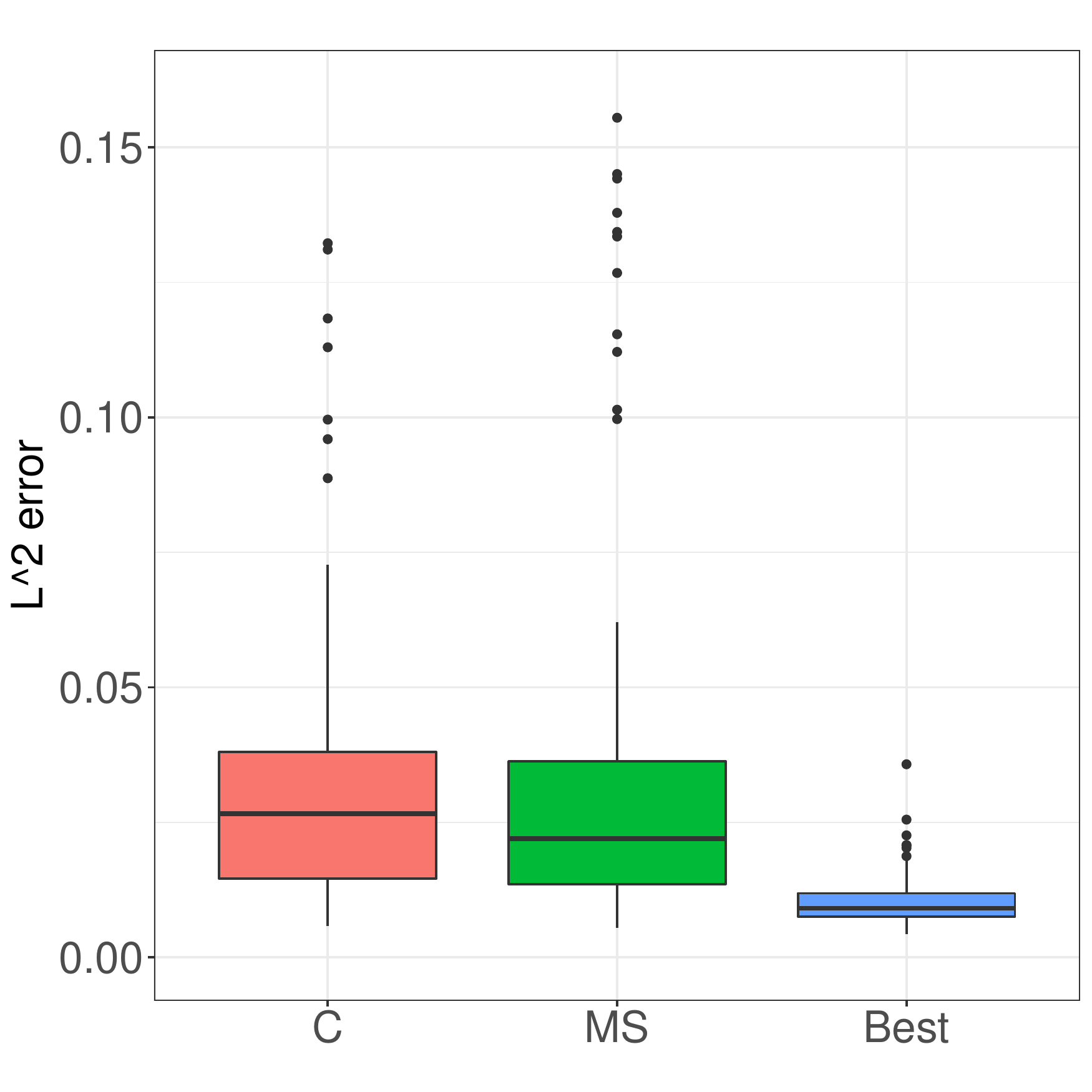}}%
 \subfloat[$n=1000, N=100$]{\includegraphics[scale=0.24]{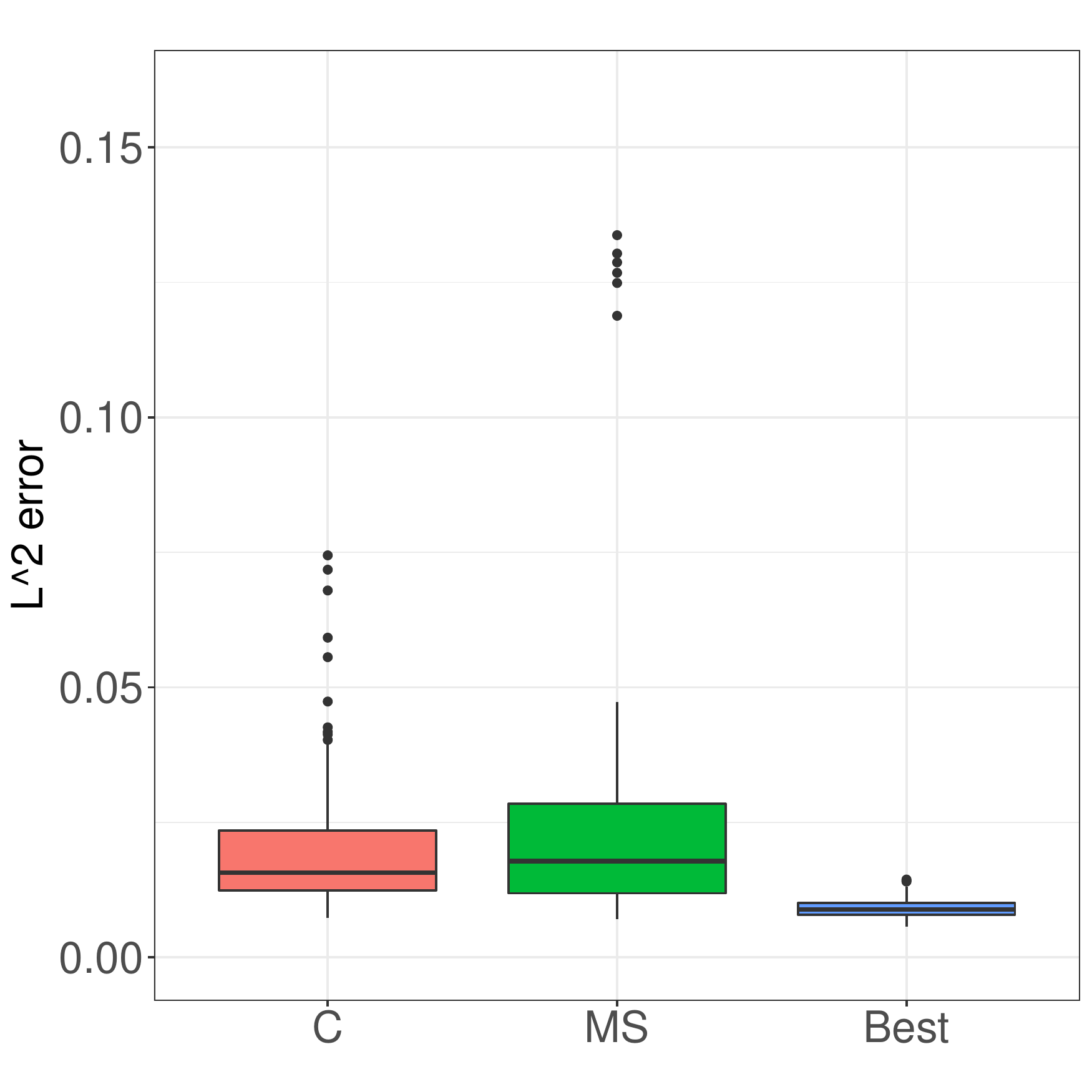}}
  \subfloat[$n=100, N=1000$]{\includegraphics[scale=0.24]{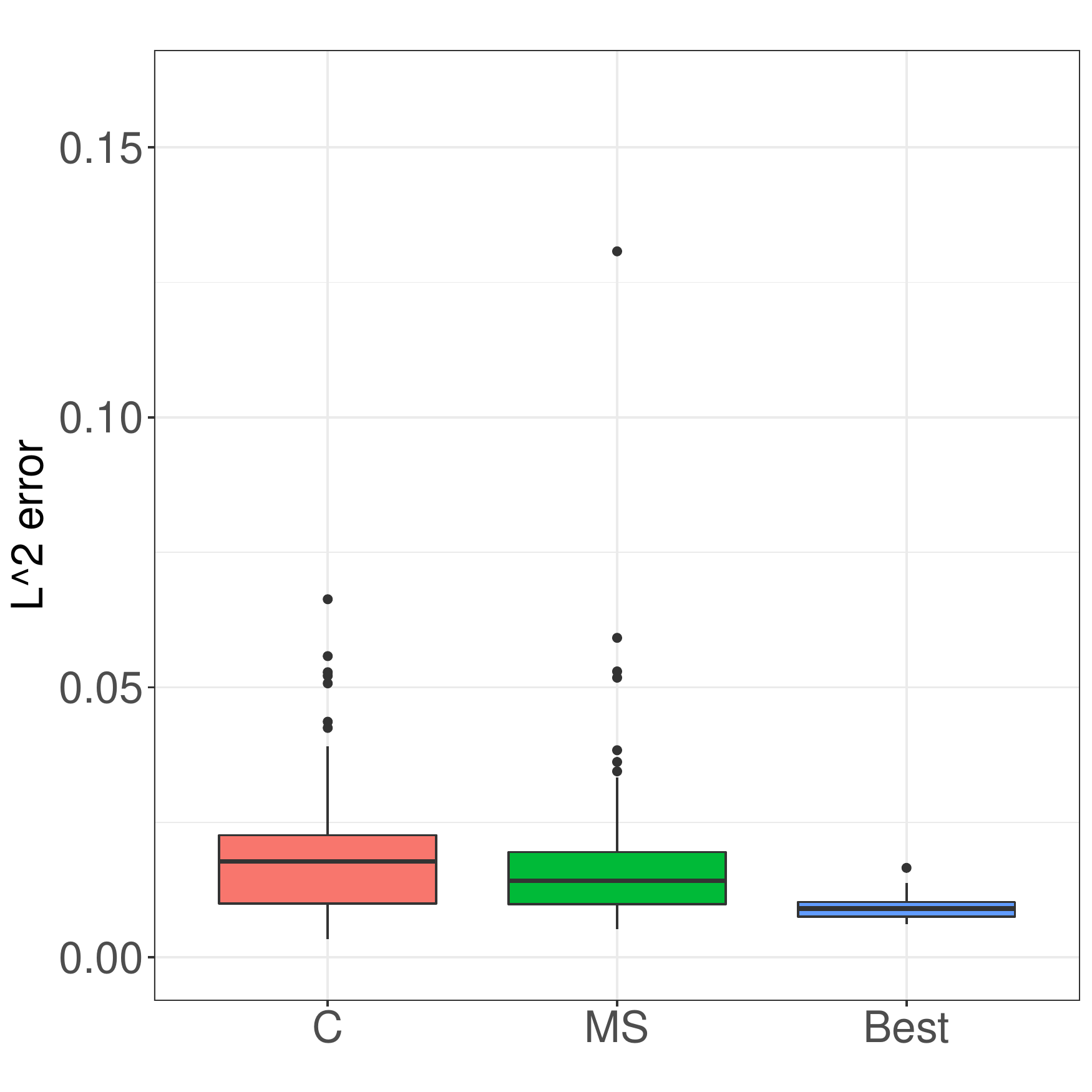}}
   \subfloat[$n=N=1000$]{\includegraphics[scale=0.24]{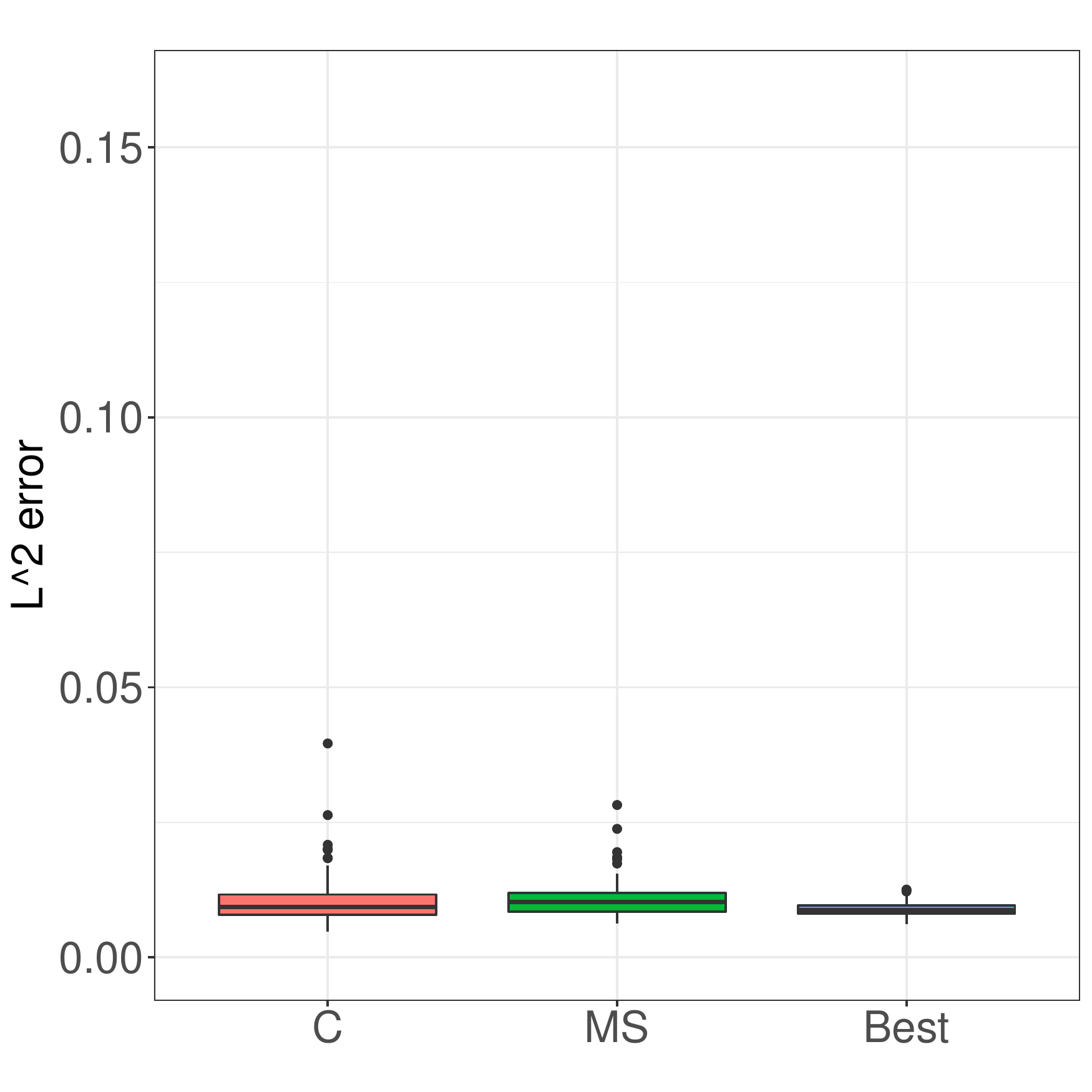}}
 \caption{Boxplot of the  $L^2$-error  of the estimators in Model $4$.}%
 \label{fig:boxplot_Mod4}%
\end{figure}
\begin{figure}[ht]%
 \centering
 \subfloat[$n=N=100$]{\includegraphics[scale=0.24]{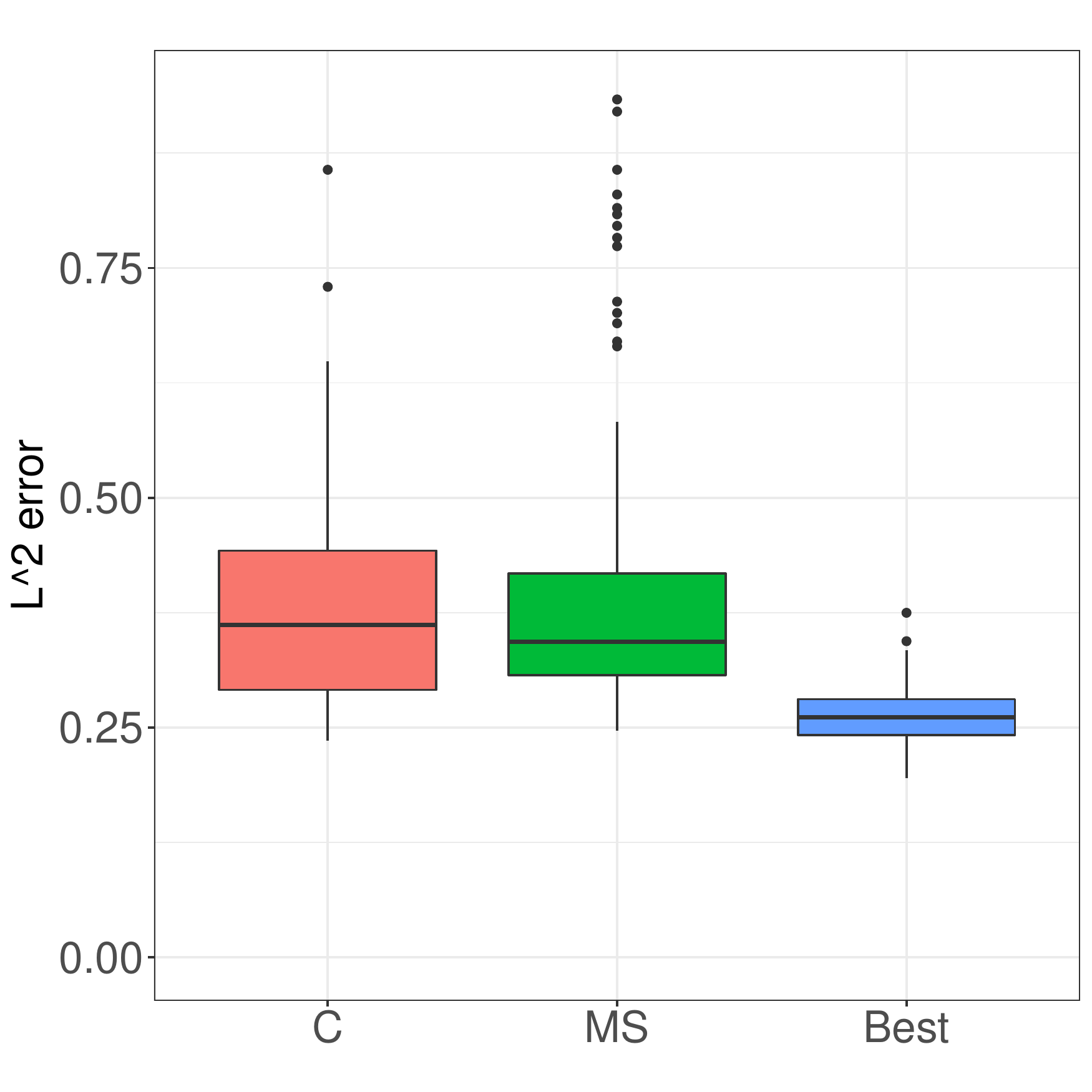}}%
 \subfloat[$n=1000, N=100$]{\includegraphics[scale=0.24]{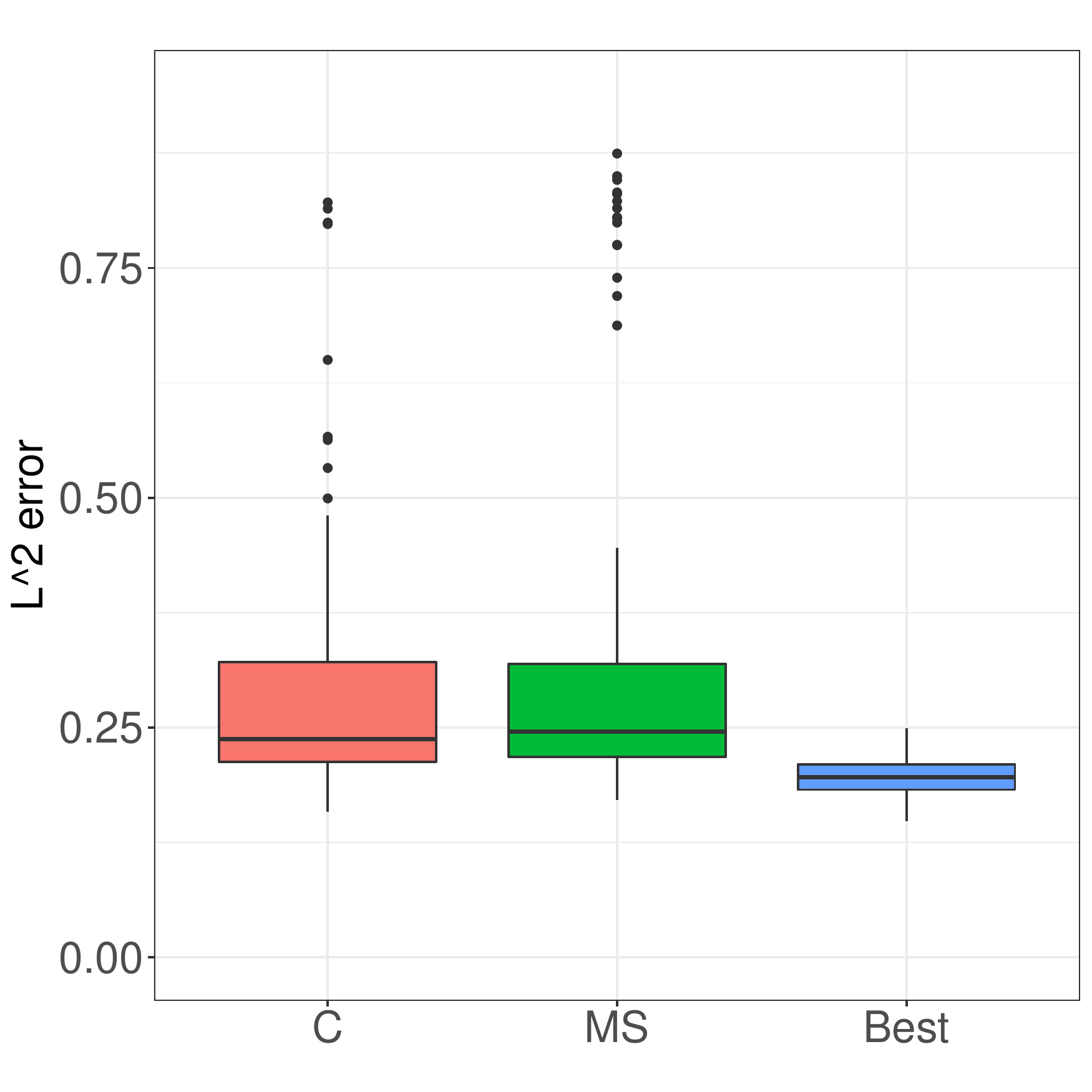}}
  \subfloat[$n=100, N=1000$]{\includegraphics[scale=0.24]{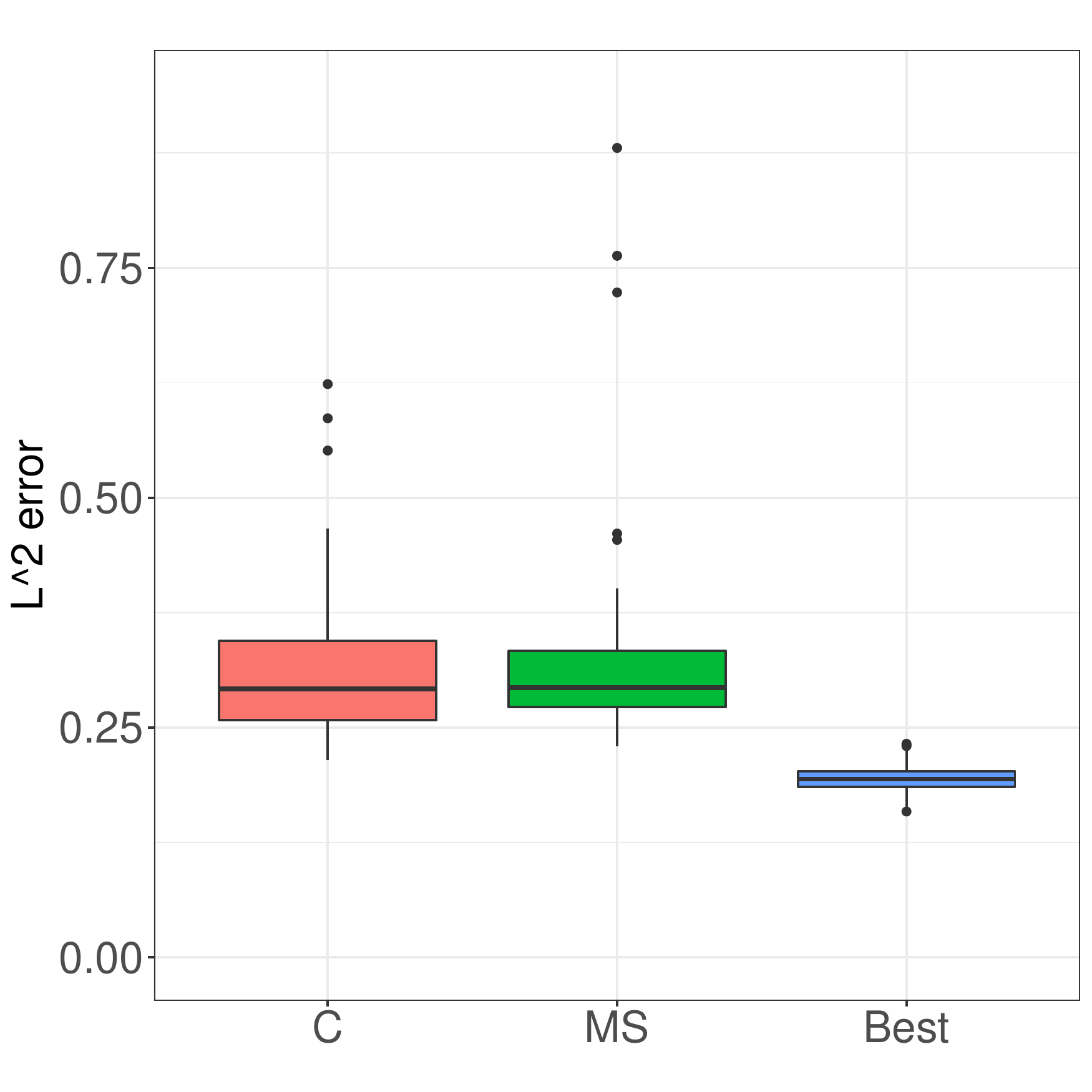}}
   \subfloat[$n=N=1000$]{\includegraphics[scale=0.24]{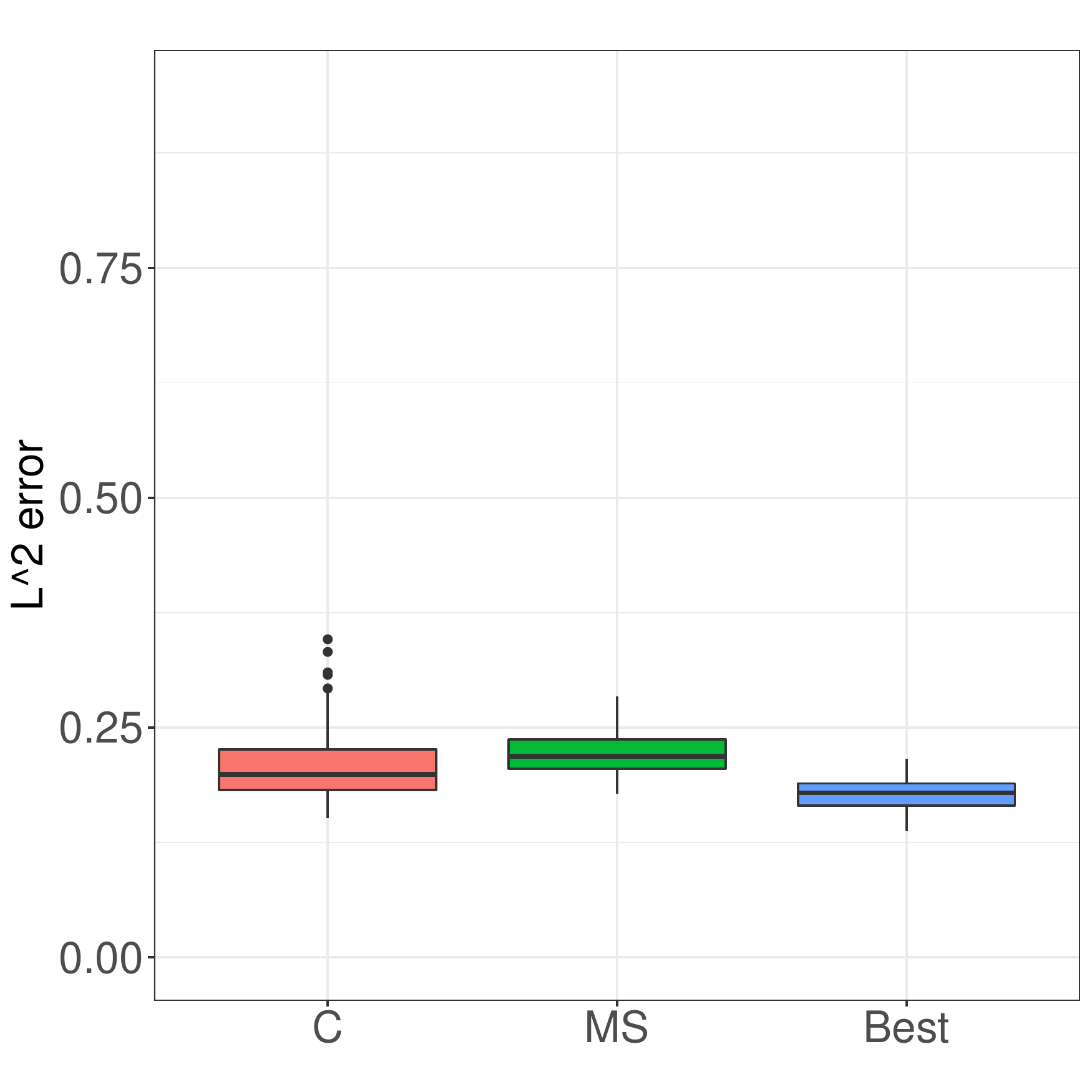}}
 \caption{Boxplot of the  $L^2$-error  of the estimators in Model $5$.}%
 \label{fig:boxplot_Mod5}%
\end{figure}
\subsection{Application to regression with reject option}
\label{subsec:App}
The variance function plays an important role in regression with reject option~\cite{Denis_Hebiri_Zaoui20}. The aim is to abstain from predicting at some hard instances and we consider the framework where rejection (abstention) rate is controlled. Let $\varepsilon \in (0,1)$, the $\varepsilon$-predictor (the optimal rule) relies on thresholding the variance function
\begin{equation*}
\label{eq:eqEpsilonPred}
    \Gamma^{*}_{\varepsilon}(x):=\begin{cases}
  \left\{f^{*}(x)\right\}  & \text{if} \;\;
  F_{{\sigma}^2}({\sigma}^2(x))\leq 1-\varepsilon\\
  \emptyset      & \text{otherwise} \enspace,
  \end{cases}
\end{equation*}
where $F_{\sigma^2}$ is the cumulative distribution function of $\sigma^{2}(X)$. The $\varepsilon$-predictor has rejection rate exactly $\varepsilon$ 
$$r\left(\Gamma^*_{\varepsilon}\right):= \mathbb{P}\left(|\Gamma_{\varepsilon}^*(X) |=  0\right) =
\mathbb{P}\left(F_{\sigma^2}(\sigma^{2}(X)) \geq 1-\varepsilon \right) = \varepsilon.
$$
Moroever, the performance of $\Gamma^{*}_{\varepsilon}$ is measured by the  $L_2$-error when prediction is performed
\begin{equation*}
\Err\left(\Gamma^{*}_{\varepsilon} \right)  :=  \mathbb{E}\left[(Y-f^{*}(X))^2 \; | \; |\Gamma^{*}_{\varepsilon} (X)| = 1\right] \enspace .
\end{equation*}
The $L_2$ error and the rejection rate of $\Gamma^{*}_{\varepsilon}$ are working in two opposite directions \wrt $\varepsilon$.
\begin{proposition}[Proposition~1 in~\cite{Denis_Hebiri_Zaoui20}]
 For any $\varepsilon < \varepsilon'$, the following holds
\begin{equation*}
\Err\left(\Gamma^{*}_{\varepsilon'}\right) \leq \Err\left(\Gamma^{*}_{\varepsilon }\right) \;\; {\rm and} \;\;
r\left(\Gamma^{*}_{\varepsilon'}\right) \geq  r\left(\Gamma^{*}_{\varepsilon }\right) \enspace .
\end{equation*}
\end{proposition}
The estimate of $\Gamma^{*}_{\varepsilon}$ needs two independent samples $\mathcal{D}_N$ and $\mathcal{D}_{\mathcal{N}}$ where $\mathcal{D}_{\mathcal{N}}$ is composed of $\mathcal{N}$ independent copies of the feature $X$. The sample $\mathcal{D}_N$ will be used to construct estimators $\hat{f}$ and $\hat{\sigma}^2$ of $f^*$ and $\sigma^2$. Besides, we consider the randomized prediction $\hat{\hat{\sigma}}^{2}(X)=\hat{\sigma}^{2}(X)+\zeta$ where $\zeta\sim \mathcal{U}([0,u])$ is independent of every other random variable with $u>0$ is small fixed real number. Thus, we use $\mathcal{D}_{\mathcal{N}}$ to estimate $F_{{\sigma}^2}$ which is given by the empirical distribution function of $\hat{\hat{\sigma}}^2$
\begin{equation*}
\hat{F}_{\hat{\hat{\sigma}}^2}(\cdot) = \frac{1}{\mathcal{N}} \sum_{i = 1}^{\mathcal{N}}    \one_{\{ \hat{\sigma}^2(X_{N+i}) + \zeta_i \leq \cdot\}} \enspace.
\end{equation*}
Finally, the {\it plug-in $\varepsilon$-predictor} is the predictor with reject option defined for each $x \in \mathbb{R}^d$ as
 \begin{equation}
    \hat{\Gamma}_{\varepsilon}(x)=\begin{cases}
  \left\{\hat{f}(x)\right\}  & \text{if} \;\;
  \hat{F}_{\hat{\hat{\sigma}}^2}(\hat{\hat{\sigma}}^2(x))\leq 1-\varepsilon\\
  \emptyset      & \text{otherwise} \enspace.
  \end{cases}
  \label{eq:eqPlugin}
\end{equation}
This plug-in approach is shown to be consistent, see~\cite{Denis_Hebiri_Zaoui20}. In particular, the plug-in $\varepsilon$-predictor bahaves asymptotically as well as the the best predictor $\Gamma_{\varepsilon}^{* }$ both in terms of risk and rejection rate. According to the Model~1 and Model~5 we have two different situations to use of reject option. In the case $a=0.25$, the use of the reject option may seem less significant because of the values of the variance function is smaller than $1$. On the contrary, we have 71.3\% of values of $\sigma^2$ are larger than $1$ in the case when $a=1$ and 36.6\% of values of $\sigma^2$ are larger than $1$ in Model~5. Then we may have some doubts in the associated prediction. In the sequel, we provide the estimation of the error of the $\varepsilon$-predictor. For each $\varepsilon\in \{0,0.1, 0.3,0.5,0.9\}$, we repeat $100$ times the following steps
\begin{enumerate}
\item[(i)] simulate two datasets $\mathcal{D}_{\mathcal{N}}$ and $\mathcal{D}_{M}$ with $N=100$ and $M=1000$;
\item[(ii)] based on $\mathcal{D}_{\mathcal{N}}$ which contains only unlabeled features, we compute the empirical cumulative distribution of $\sigma^{2}(X)$; 
\item[(iii)] finally, we compute the empirical rejection rate $r$ and the empirical error $\Err$ of $\Gamma_{\varepsilon}^{* }$.
\end{enumerate}
From these experiments, we compute the average and standard deviation of $\Err$ and $r$. We report the results in Table~\ref{Tab:3}
\begin{table}[!ht]
\centering
\footnotesize{
\vspace*{0.25cm}
\begin{tabular}{||l || c | c  ||  c | c|| c | c || c|}
\multicolumn{1}{c}{}  &  \multicolumn{3}{c}{{}}& \multicolumn{2}{c}{}\\ \hline
\multicolumn{1}{c}{}  &  \multicolumn{2}{c}{{Model $1$} ($a=0.25$)} & \multicolumn{2}{c}{{Model $1$ ($a=1$)}} &\multicolumn{2}{c}{{Model $5$}} 
\\ \hline \noalign{\smallskip}
$\varepsilon$ & $\Err$ &  $r$  &  $\Err$  & $r$ & $\Err$  &  $r$ & \\ \hline \noalign{\smallskip}
0 & 0.34 (0.02)  & 0.00 (0.00)  & 1.38 (0.08) & 0.00 (0.00) & 0.90 (0.04)  &  0.00 (0.00)  \\
0.1 & 0.31 (0.01) & 0.10 (0.02) & 1.26 (0.07) & 0.10 (0.02) & 0.73 (0.04) & 0.10 (0.02) \\
0.3 & 0.27 (0.01)  & 0.30 (0.03)& 1.06 (0.07) & 0.30 (0.02) & 0.51 (0.03) & 0.30 (0.02)  \\
0.5 & 0.22 (0.02) & 0.49 (0.03) & 0.89 (0.06) & 0.50 (0.03) &  0.34 (0.03) &  0.50 (0.02) \\
0.7 & 0.17 (0.02) & 0.70 (0.03) & 0.69 (0.06) & 0.70 (0.02) & 0.18 (0.02)  & 0.70 (0.03) \\
0.9 & 0.10 (0.02) & 0.90 (0.02) & 0.41 (0.07) & 0.90 (0.02) &  0.03 (0.01) & 0.90 (0.02) \\
\hline
\end{tabular}}
\caption{Average and standard deviation of the empirical $L^2$-error of the three estimators with $n\neq N$.}
\label{Tab:3}
\end{table}
Now, we evaluate the performance of {\it plug-in $\varepsilon$-predictor} considering the same algorithm for both estimation tasks and build five plug-in $\varepsilon$-predictors based respectively on support vector machines (\texttt{svm}), random forests (\texttt{rf}), and regression tree (\texttt{tree}), \texttt{C} and \texttt{MS} algorithms (constructed in section~\ref{sec:machines}). For each $\varepsilon\in \{0,0.1, 0.3,0.5,0.9\}$, and each plug-in $\varepsilon$-predictor, we compute the empirical rejection rate $\hat{r}$ and the empirical error $\widehat{\Err}$. So, we repeat independently $100$ times the following steps:
\begin{enumerate}
\item simulate four datasets $\mathcal{D}_{n}$, $\mathcal{D}_{N}$, $\mathcal{D}_{\mathcal{N}}$ and $\mathcal{D}_{T}$ with $n=N=100$ or $1000$, $\mathcal{N}=100$ and $T=1000$
\item based on $\mathcal{D}_{n}$, we compute the estimators in $\mathcal{F}$, and then based on $\mathcal{D}_N$,  we compute  an aggregates $\hat{f}_{\texttt{MS}}$,  $\hat{f}_{\texttt{CM}}$, and the \texttt{knn}, \texttt{rf} and \texttt{svm} estimators  of the  regression function $f^{*}$;
\item based on $\mathcal{D}_{n}$ and $\hat{f}_{\texttt{MS}}$ (resp. $\hat{f}_{\texttt{CM}}$), we compute the estimators in $\mathcal{G}_1$ (resp. $\mathcal{G}_2$). Then, based $\mathcal{D}_{N}$ we calculate $\hat{\sigma}^{2}_{\texttt{MS}}$, $\hat{\sigma}^{2}_{\texttt{CM}}$, \texttt{tree}, \texttt{rf} and \texttt{svm} estimators of $\sigma^2$; 
\item based on $\mathcal{D}_{\mathcal{N}}$, we compute the empirical cumulative distribution function of the randomized estimators $\hat{\hat{\sigma}}^{2}(X)$ taking $\zeta\sim \mathcal{U}([0,10^{-9}])$;
\item finally, over $\mathcal{D}_{T}$, we compute the empirical rejection rate $\hat{r}$ and the empirical error $\widehat{\Err}$ for $\hat{\Gamma}_{\varepsilon}$.
\end{enumerate}
From these estimations, we compute the average and standard deviation of $\hat{r}$ and $\widehat{\Err}$. The results are reported in Tables~\ref{Tab:ErrorErr}-~\ref{Tab:Rejectr} and Figure~\ref{fig:ErrorReject}. Firstly, we can see that the performance of $\varepsilon$-predictor in Model~1 when $a=1$ and Model~5 reflects that the regression problem is quite difficult because the decrease of the error is slower. Indeed, if $\varepsilon=0$, the empirical error of $\Gamma_{\varepsilon}^{*}$ equals $1.38$ and if $\varepsilon=0.9$ equals $0.41$. Secondly, the plug-in $\varepsilon$-predictors are decreasing \wrt $\varepsilon$ for all models (see Table~\ref{Tab:ErrorErr} and Figure~\ref{fig:ErrorReject}) and their empirical rejection rates are very close to their expected values (see Table~\ref{Tab:Rejectr}). Moroever, the performance of the methods gets better when $n$ and $N$ increase. Finally, we observe that the plug-in $\varepsilon$-predictors based on $\texttt{MS}$ and $\texttt{C}$ (\texttt{C}-plug-in $\varepsilon$-predictor is a little better than \texttt{MS}-plug-in $\varepsilon$-predictor) are close to optimal rule. Therefore, we deduce that a good estimation of regression and variance functions leads to a better plug-in $\varepsilon$-predictors. 

\begin{table}
\begin{center}
\footnotesize{
\begin{tabular}{l || ccccc||}
\multicolumn{1}{c}{} & \multicolumn{5}{c}{{Model~1 (a=0.25)}}\\
\hline
\multicolumn{1}{c}{} &  \multicolumn{5}{c}{$N=n= 100$}  \\
\hline\noalign{\smallskip}
  $\varepsilon$ \;\;\;  & \;\;   \texttt{Tree} \;\;\; & \texttt{rf}  \;\;\;& \texttt{svm} & \;\;   \texttt{C} \;\;\; &  \texttt{MS}  \\
\noalign{\smallskip}
\hline
\noalign{\smallskip}
 0   & 0.50 (0.05) & 0.39 (0.02) & 0.39 (0.02) & 0.36 (0.02) &  0.37 (0.02) \\
 0.1 & 0.49 (0.04) & 0.38 (0.03) & 0.38 (0.03) & 0.35 (0.02) &  0.35 (0.02) \\
 0.3 & 0.46 (0.05) & 0.35 (0.03) & 0.36 (0.03) & 0.32 (0.03) &   0.32 (0.03) \\
 0.5 & 0.44 (0.06) & 0.31 (0.03) & 0.34 (0.04) & 0.30 (0.04) &  0.30 (0.04) \\
 0.7 & 0.44 (0.09) & 0.29 (0.04) & 0.32 (0.05) & 0.27 (0.05) &  0.28 (0.05) \\
 0.9 & 0.45 (0.14) & 0.25 (0.08) & 0.29 (0.10) & 0.23 (0.08) & 0.23 (0.07)  \\

\hline
\end{tabular}
}
\vspace*{0.25cm}

\footnotesize{
\begin{tabular}{l || ccccc||}
\multicolumn{1}{c}{} & \multicolumn{5}{c}{{Model~1 (a=0.25)}}\\
\hline
\multicolumn{1}{c}{} &  \multicolumn{5}{c}{$N=n= 1000$}  \\
\hline\noalign{\smallskip}
  $\varepsilon$ \;\;\;  & \;\;   \texttt{Tree} \;\;\; & \texttt{rf}  \;\;\;& \texttt{svm} & \;\;   \texttt{C} \;\;\; &  \texttt{MS}  \\
\noalign{\smallskip}
\hline
\noalign{\smallskip}
 0   & 0.35 (0.02) & 0.37 (0.02)& 0.35 (0.02)  & 0.35 (0.02)&  0.35 (0.02) \\
 0.1 & 0.33 (0.02) & 0.35 (0.02) & 0.33 (0.01) & 0.33 (0.02) & 0.33 (0.02) \\
 0.3 & 0.31 (0.02) & 0.31 (0.02) & 0.30 (0.02) & 0.29 (0.02) &  0.30 (0.02) \\
 0.5 & 0.27 (0.02) & 0.28 (0.02) & 0.26 (0.02) & 0.25 (0.02)& 0.26 (0.02) \\
 0.7 & 0.25 (0.04) & 0.23 (0.03) & 0.22 (0.03) & 0.20 (0.02) & 0.21 (0.02) \\
 0.9 & 0.26 (0.06) & 0.16 (0.03) & 0.16 (0.04) & 0.14 (0.04) & 0.17 (0.04) \\

\hline
\end{tabular}
}
\vspace*{0.25cm}

\footnotesize{
\begin{tabular}{l || ccccc||}
\multicolumn{1}{c}{} & \multicolumn{5}{c}{{Model~1 (a=1)}}\\
\hline
\multicolumn{1}{c}{} &  \multicolumn{5}{c}{$N=n= 100$}  \\
\hline\noalign{\smallskip}
  $\varepsilon$ \;\;\;  & \;\;   \texttt{tree} \;\;\; & \texttt{rf}  \;\;\;& \texttt{svm} & \;\;   \texttt{C} \;\;\; &  \texttt{MS}  \\
\noalign{\smallskip}
\hline
\noalign{\smallskip}
 0   & 1.99 (0.15) & 1.54 (0.08) & 1.55 (0.10) & 1.42 (0.07) &   1.43 (0.08) \\
 0.1 & 1.96 (0.20) & 1.48 (0.09) & 1.50 (0.12) & 1.37 (0.09) & 1.38 (0.10) \\
 0.3 & 1.87 (0.21) & 1.35 (0.12) & 1.41 (0.14) & 1.22 (0.09) & 1.24 (0.11)   \\
 0.5 & 1.83 (0.28) & 1.27 (0.14) & 1.35 (0.15) & 1.16 (0.12) & 1.18 (0.15) \\
 0.7 & 1.76 (0.35) & 1.14 (0.20) & 1.28 (0.22) & 1.01 (0.19) &  1.07 (0.22) \\
 0.9 & 1.88 (0.52) & 1.01 (0.27) & 1.18 (0.36) & 0.85 (0.26) & 0.87 (0.28) \\

\hline
\end{tabular}
}

\vspace*{0.25cm}
\footnotesize{
\begin{tabular}{l || ccccc||}
\multicolumn{1}{c}{} & \multicolumn{5}{c}{{Model~1 (a=1)}}\\
\hline
\multicolumn{1}{c}{} &  \multicolumn{5}{c}{$N=n= 1000$}  \\
\hline\noalign{\smallskip}
  $\varepsilon$ \;\;\;  & \;\;   \texttt{tree} \;\;\; & \texttt{rf}  \;\;\;& \texttt{svm} & \;\;   \texttt{C} \;\;\; &  \texttt{MS}  \\
\noalign{\smallskip}
\hline
\noalign{\smallskip}
 0   & 1.41 (0.07) & 1.50 (0.08) & 1.43 (0.08) & 1.39 (0.08) & 1.39 (0.08)   \\
 0.1 & 1.33 (0.08) & 1.41 (0.09) & 1.34 (0.08) & 1.30 (0.07) & 1.30 (0.07) \\
 0.3 & 1.21 (0.08) & 1.25 (0.09) & 1.19 (0.08) & 1.15 (0.07) &  1.16 (0.07) \\
 0.5 & 1.07 (0.10) & 1.11 (0.09) & 1.05 (0.09) & 0.99 (0.09) &  1.02 (0.08) \\
 0.7 & 0.96 (0.13) & 0.92 (0.09) & 0.87 (0.10) & 0.79 (0.08) & 0.84 (0.09) \\
 0.9 & 0.98 (0.22) & 0.66 (0.11) & 0.67 (0.15) & 0.56 (0.12) &  0.67 (0.15) \\

\hline
\end{tabular}
}

\vspace*{0.25cm}
\footnotesize{
\begin{tabular}{l || ccccc||}
\multicolumn{1}{c}{} & \multicolumn{5}{c}{{Model~5}}\\
\hline
\multicolumn{1}{c}{} &  \multicolumn{5}{c}{$N=n= 100$}  \\
\hline\noalign{\smallskip}
  $\varepsilon$ \;\;\;  & \;\;   \texttt{tree} \;\;\; & \texttt{rf}  \;\;\;& \texttt{svm} & \;\;   \texttt{C} \;\;\; &  \texttt{MS}  \\
\noalign{\smallskip}
\hline
\noalign{\smallskip}
 0   & 1.30 (0.10) & 1.02 (0.06) & 1.04 (0.06) & 0.97 (0.05) &  0.97 (0.06) \\
 0.1 & 1.27 (0.13) & 0.93 (0.07) & 0.99 (0.09) & 0.89 (0.07) &  0.89 (0.07) \\
 0.3 & 1.10 (0.14) & 0.76 (0.07) & 0.84 (0.11) & 0.72 (0.07) & 0.72 (0.08)   \\
 0.5 & 1.04 (0.21) & 0.64 (0.07) & 0.76 (0.12) & 0.58 (0.07) &  0.59 (0.07) \\
 0.7 & 0.96 (0.25) & 0.50 (0.08) & 0.66 (0.15) & 0.46 (0.09) &  0.47 (0.09) \\
 0.9 & 1.03 (0.47) & 0.35 (0.10) & 0.61 (0.22) & 0.33 (0.12) &  0.34 (0.11) \\

\hline
\end{tabular}
}

\vspace*{0.25cm}

\footnotesize{
\begin{tabular}{l || ccccc||}
\multicolumn{1}{c}{} & \multicolumn{5}{c}{{Model~5}}\\
\hline
\multicolumn{1}{c}{} &  \multicolumn{5}{c}{$N=n= 1000$}  \\
\hline\noalign{\smallskip}
  $\varepsilon$ \;\;\;  & \;\;   \texttt{tree} \;\;\; & \texttt{rf}  \;\;\;& \texttt{svm} & \;\;   \texttt{C} \;\;\; &  \texttt{MS}  \\
\noalign{\smallskip}
\hline
\noalign{\smallskip}
 0   & 0.97 (0.05) & 0.98 (0.05) & 0.95 (0.05) & 0.93 (0.05) &   0.94 (0.05) \\
 0.1 & 0.85 (0.05) & 0.85 (0.04) & 0.83 (0.05) & 0.80 (0.05) &  0.82 (0.05) \\
 0.3 & 0.68 (0.05) & 0.66 (0.04) & 0.66 (0.04) & 0.61 (0.04) &   0.64 (0.04) \\
 0.5 & 0.55 (0.05) & 0.48 (0.04) & 0.50 (0.05) & 0.45 (0.03) &  0.47 (0.04) \\
 0.7 & 0.47 (0.07) & 0.33 (0.04) & 0.37 (0.04) & 0.31 (0.03) & 0.33 (0.04)  \\
 0.9 & 0.48 (0.11) & 0.15 (0.04) & 0.30 (0.06) & 0.16 (0.05) &  0.18 (0.06) \\

\hline
\end{tabular}
}
\end{center}
\caption{Empirical error $\widehat{\Err}$ of five plug-in $\varepsilon$-predictors in Model~1 when $a=0.25$ or $1$ and Model~5. The standard deviation is
provided between parenthesis.
}
\label{Tab:ErrorErr}
\end{table}
\begin{table}

\begin{center}
\footnotesize{
\begin{tabular}{l || ccccc||}
\multicolumn{1}{c}{} & \multicolumn{5}{c}{{Model~1 (a=0.25)}}\\
\hline
\multicolumn{1}{c}{} &  \multicolumn{5}{c}{$N=n= 100$}  \\
\hline\noalign{\smallskip}
  $\varepsilon$ \;\;\;  & \;\;   \texttt{tree} \;\;\; & \texttt{rf}  \;\;\;& \texttt{svm} & \;\;   \texttt{C} \;\;\; &  \texttt{MS} \\
\noalign{\smallskip}
\hline
\noalign{\smallskip}
 0   &  0.00 (0.00) & 0.00 (0.00) &0.00 (0.00) & 0.00 (0.00) &0.00 (0.00)  \\
 0.1 & 0.10 (0.02) & 0.10 (0.02) & 0.10 (0.01) & 0.10 (0.02)  & 0.10 (0.02)\\
 0.3 & 0.30 (0.02) & 0.30 (0.03) & 0.30 (0.02) & 0.30 (0.03) & 0.30 (0.02)\\
 0.5 & 0.50 (0.03) & 0.50 (0.03) & 0.50 (0.03) & 0.50 (0.03) &  0.50 (0.03)\\
 0.7 & 0.70 (0.02) & 0.70 (0.03) & 0.70 (0.03) & 0.70 (0.03) & 0.70 (0.03) \\
 0.9 & 0.90 (0.02) & 0.90 (0.02) & 0.90 (0.02) & 0.90 (0.02) & 0.90 (0.02) \\

\hline
\end{tabular}
}
\vspace*{0.25cm}
\footnotesize{
\begin{tabular}{l || ccccc||}
\multicolumn{1}{c}{} & \multicolumn{5}{c}{{Model~1 (a=0.25)}}\\
\hline
\multicolumn{1}{c}{} &  \multicolumn{5}{c}{$N=n= 1000$}  \\
\hline\noalign{\smallskip}
  $\varepsilon$ \;\;\;  & \;\;   \texttt{tree} \;\;\; & \texttt{rf}  \;\;\;& \texttt{svm} & \;\;   \texttt{C} \;\;\; &  \texttt{MS} \\
\noalign{\smallskip}
\hline
\noalign{\smallskip}
 0   &  0.00 (0.00) & 0.00 (0.00) &0.00 (0.00) &0.00 (0.00) &0.00 (0.00)  \\
 0.1 & 0.10 (0.01) & 0.10 (0.02) & 0.10 (0.02) & 0.10 (0.01) & 0.10 (0.02) \\
 0.3 & 0.30 (0.03) & 0.29 (0.03) & 0.30 (0.03) & 0.30 (0.03) & 0.29 (0.03) \\
 0.5 & 0.50 (0.03) & 0.49 (0.03) & 0.50 (0.03) & 0.50 (0.03)  &  0.49 (0.03) \\
 0.7 & 0.70 (0.02) & 0.70 (0.02) & 0.70 (0.02) & 0.70 (0.03) & 0.70 (0.02) \\
 0.9 & 0.90 (0.02) & 0.90 (0.02) & 0.90 (0.02) & 0.90 (0.02) & 0.90 (0.02) \\

\hline
\end{tabular}
}
\vspace*{0.25cm}

\footnotesize{
\begin{tabular}{l || ccccc||}
\multicolumn{1}{c}{} & \multicolumn{5}{c}{{Model~1 (a=1)}}\\
\hline
\multicolumn{1}{c}{} &  \multicolumn{5}{c}{$N=n= 100$}  \\
\hline\noalign{\smallskip}
  $\varepsilon$ \;\;\;  & \;\;   \texttt{tree} \;\;\; & \texttt{rf}  \;\;\;& \texttt{svm} & \;\;   \texttt{C} \;\;\; &  \texttt{MS} \\
\noalign{\smallskip}
\hline
\noalign{\smallskip}
 0   &  0.00 (0.00) & 0.00 (0.00) &0.00 (0.00) &0.00 (0.00) &0.00 (0.00)  \\
 0.1 & 0.10 (0.02) & 0.10 (0.02)  & 0.10 (0.02)  & 0.10 (0.02)  & 0.10 (0.02) \\
 0.3 & 0.30 (0.02) & 0.30 (0.03) & 0.30 (0.03) & 0.30 (0.02) &  0.30 (0.03) \\
 0.5 & 0.50 (0.03) &  0.50 (0.03) &  0.50 (0.03) &  0.50 (0.03) &  0.50 (0.03)  \\
 0.7 & 0.70 (0.03) & 0.70 (0.03) & 0.69 (0.03) & 0.69 (0.03) &  0.70 (0.03) \\
 0.9 & 0.90 (0.02) & 0.90 (0.02) & 0.90 (0.02) & 0.90 (0.02) & 0.90 (0.02)  \\

\hline
\end{tabular}
}
\vspace*{0.25cm}

\footnotesize{
\begin{tabular}{l || ccccc||}
\multicolumn{1}{c}{} & \multicolumn{5}{c}{{Model~1 (a=1)}}\\
\hline
\multicolumn{1}{c}{} &  \multicolumn{5}{c}{$N=n= 1000$}  \\
\hline\noalign{\smallskip}
  $\varepsilon$ \;\;\;  & \;\;   \texttt{tree} \;\;\; & \texttt{rf}  \;\;\;& \texttt{svm} & \;\;   \texttt{C} \;\;\; &  \texttt{MS} \\
\noalign{\smallskip}
\hline
\noalign{\smallskip}
 0   &  0.00 (0.00) & 0.00 (0.00) &0.00 (0.00) &0.00 (0.00) &0.00 (0.00)  \\
 0.1 & 0.10 (0.02) & 0.10 (0.02) & 0.10 (0.02)& 0.10 (0.02) & 0.10 (0.02) \\
 0.3 & 0.30 (0.02) & 0.30 (0.02)& 0.30 (0.02) & 0.30(0.02)& 0.30 (0.02) \\
 0.5 & 0.50 (0.03) & 0.50 (0.03) & 0.50 (0.02)& 0.50 (0.03)& 0.50 (0.03) \\
 0.7 & 0.70 (0.02)& 0.70 (0.02) & 0.70 (0.02)& 0.70 (0.02) & 0.70 (0.03) \\
 0.9 & 0.90 (0.02) & 0.90 (0.02) & 0.90 (0.02)& 0.90 (0.02) & 0.90 (0.02) \\

\hline
\end{tabular}
}
\vspace*{0.25cm}
\footnotesize{
\begin{tabular}{l || ccccc||}
\multicolumn{1}{c}{} & \multicolumn{5}{c}{{Model~5}}\\
\hline
\multicolumn{1}{c}{} &  \multicolumn{5}{c}{$N=n= 100$}  \\
\hline\noalign{\smallskip}
  $\varepsilon$ \;\;\;  & \;\;   \texttt{tree} \;\;\; & \texttt{rf}  \;\;\;& \texttt{svm} & \;\;   \texttt{C} \;\;\; &  \texttt{MS} \\
\noalign{\smallskip}
\hline
\noalign{\smallskip}
 0   &  0.00 (0.00) & 0.00 (0.00) &0.00 (0.00) &0.00 (0.00) &0.00 (0.00)  \\
 0.1 & 0.10 (0.02) & 0.10 (0.02) & 0.10 (0.02) & 0.10 (0.02) & 0.10 (0.02) \\
 0.3 & 0.30 (0.03) & 0.30 (0.03) & 0.30 (0.03) & 0.30 (0.03) &  0.30 (0.03) \\
 0.5 & 0.50 (0.03) & 0.50 (0.03) & 0.50 (0.02) & 0.50 (0.03) &   0.50 (0.03) \\
 0.7 & 0.70 (0.02) & 0.70 (0.03) & 0.70 (0.03) & 0.69 (0.03) &  0.69 (0.03) \\
 0.9 & 0.90 (0.02) & 0.90 (0.02) & 0.90 (0.01) & 0.90 (0.02) &  0.90 (0.02) \\

\hline
\end{tabular}
}
\vspace*{0.25cm}
\footnotesize{
\begin{tabular}{l || ccccc||}
\multicolumn{1}{c}{} & \multicolumn{5}{c}{{Model~5}}\\
\hline
\multicolumn{1}{c}{} &  \multicolumn{5}{c}{$N=n= 1000$}  \\
\hline\noalign{\smallskip}
  $\varepsilon$ \;\;\;  & \;\;   \texttt{tree} \;\;\; & \texttt{rf}  \;\;\;& \texttt{svm} & \;\;   \texttt{C} \;\;\; &  \texttt{MS} \\
\noalign{\smallskip}
\hline
\noalign{\smallskip}
 0   &  0.00 (0.00) & 0.00 (0.00) &0.00 (0.00) &0.00 (0.00) &0.00 (0.00)  \\
 0.1 & 0.10 (0.01) & 0.10 (0.02)  & 0.10 (0.02) & 0.10 (0.02) & 0.10 (0.02) \\
 0.3 & 0.30 (0.02) & 0.30 (0.02) & 0.30 (0.02) & 0.30 (0.03) &  0.30 (0.02) \\
 0.5 & 0.49 (0.03) & 0.50 (0.03) & 0.50 (0.03) & 0.50 (0.03) & 0.50 (0.03)  \\
 0.7 & 0.70 (0.03) & 0.70 (0.03) & 0.70 (0.02) & 0.70 (0.03) & 0.70 (0.03)  \\
 0.9 & 0.90 (0.02) & 0.90 (0.02) & 0.90 (0.01) &  0.90 (0.02) & 0.90 (0.02) \\

\hline
\end{tabular}
}
\end{center}

\caption{Empirical reject rate $\hat{r}$ of five plug-in $\varepsilon$-predictors in Model~1 when $a=0.25$ or $1$ and Model~5. The standard deviation is
provided between parenthesis.
}
 \label{Tab:Rejectr}
\end{table}

\begin{figure}[ht]%
 \centering
 \subfloat{\includegraphics[scale=0.24]{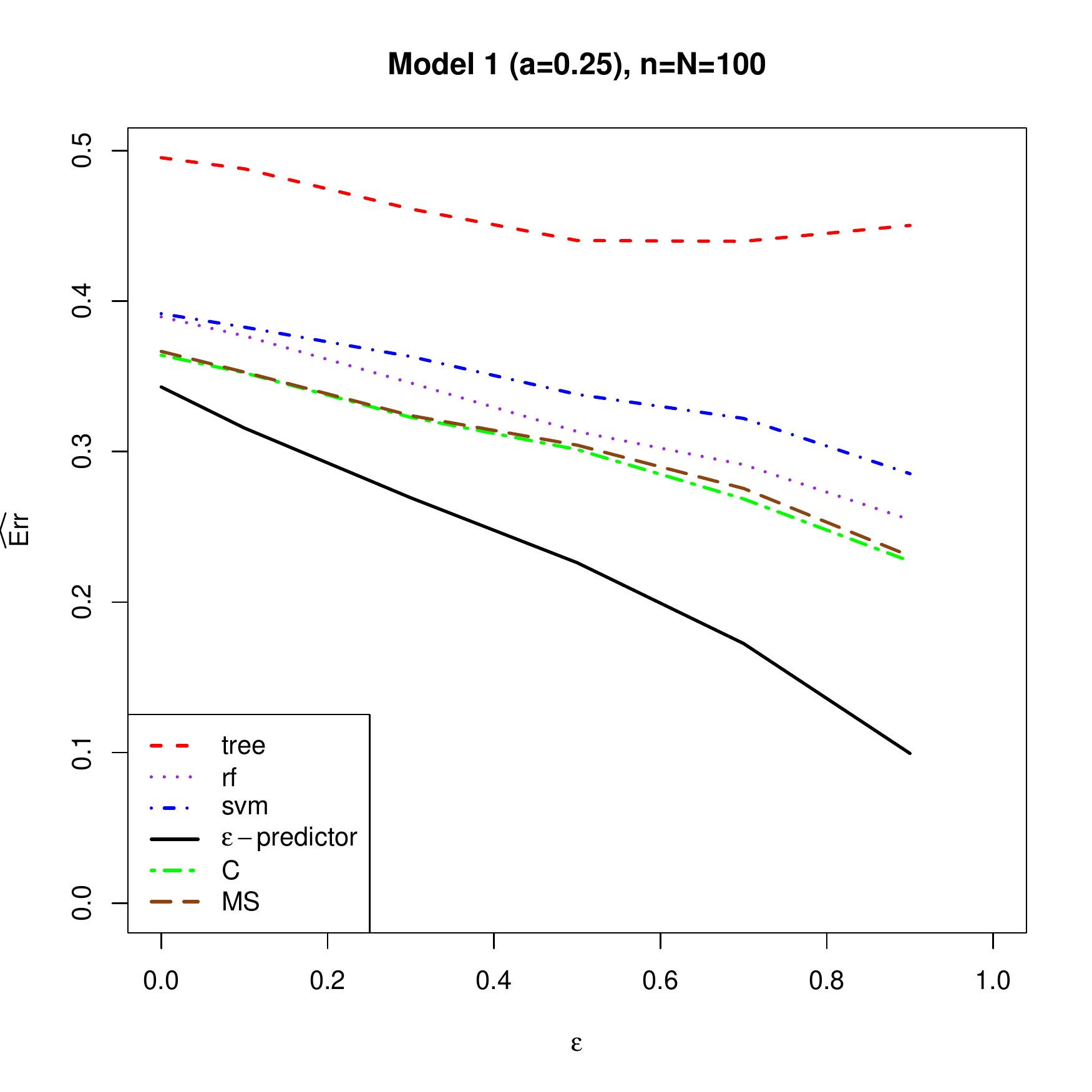}}%
 \subfloat{\includegraphics[scale=0.24]{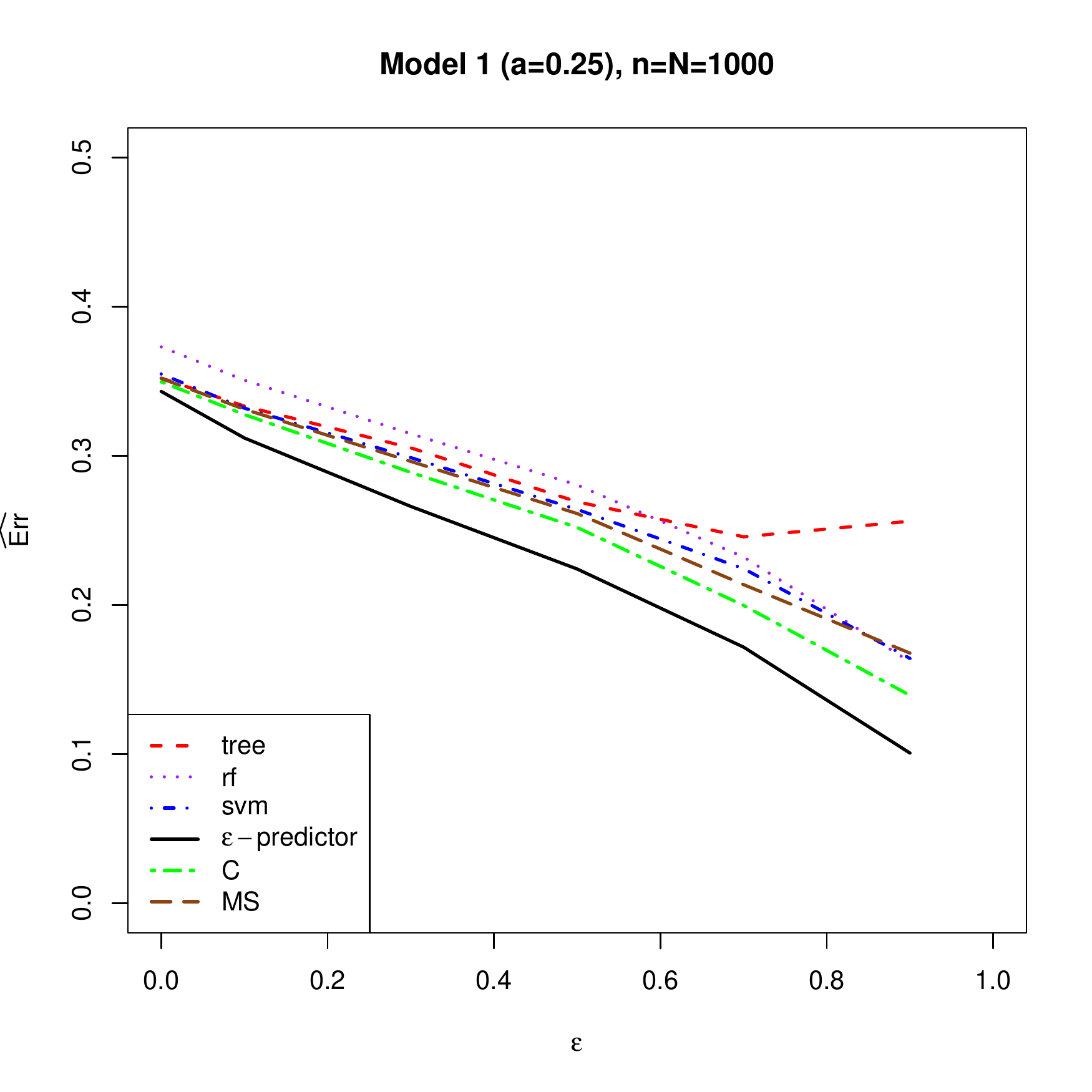}}
  \subfloat{\includegraphics[scale=0.24]{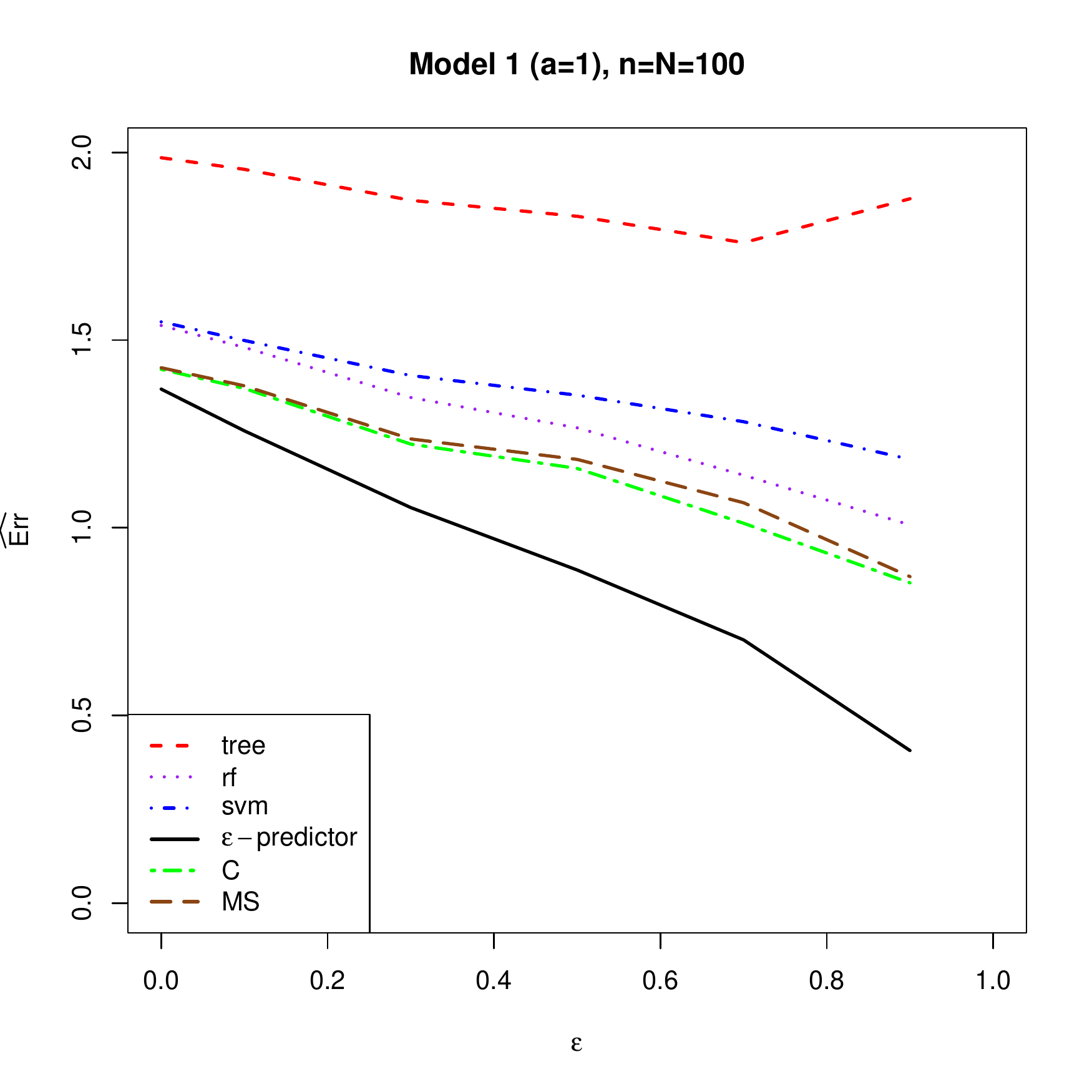}}
   \subfloat{\includegraphics[scale=0.24]{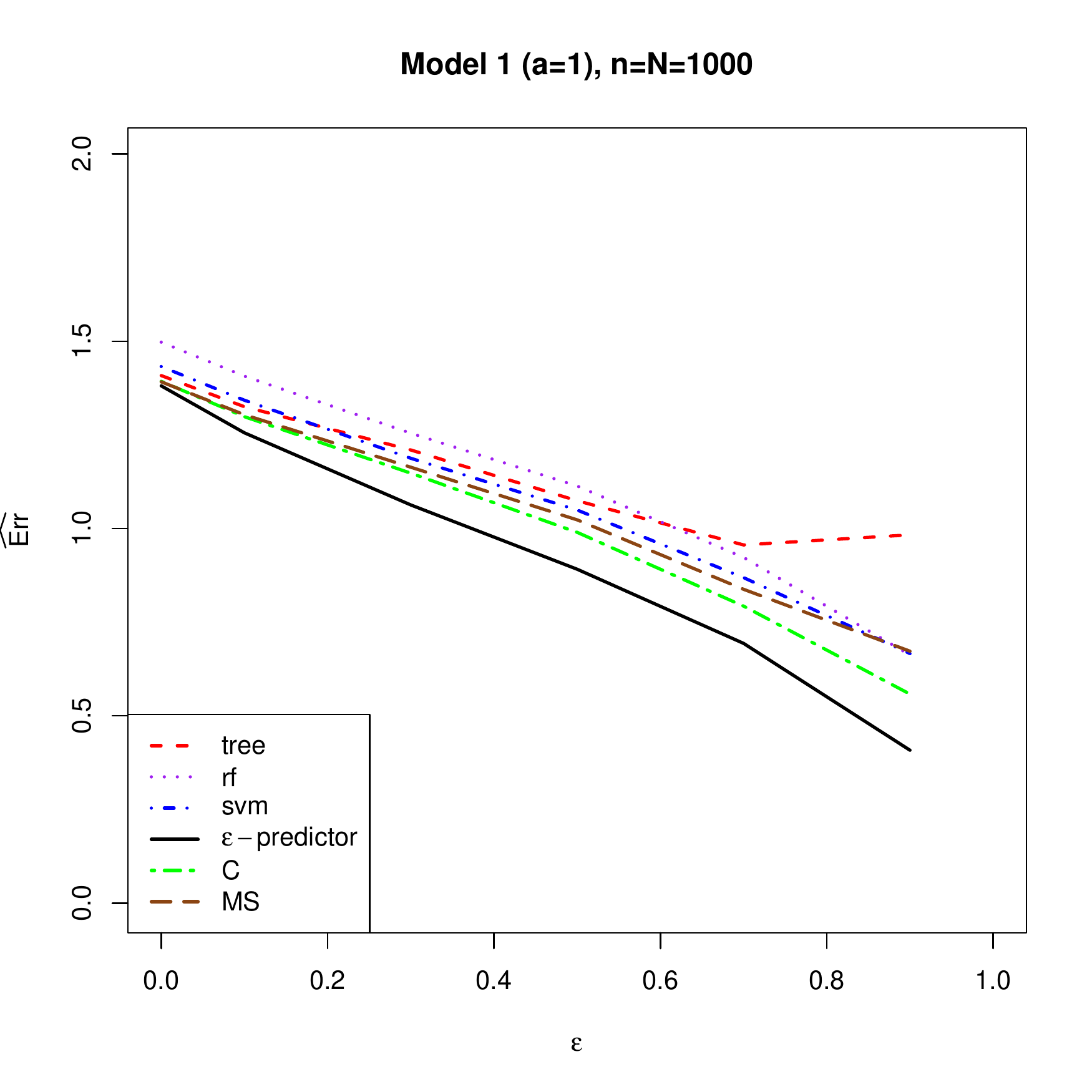}}\\
   \subfloat{\includegraphics[scale=0.24]{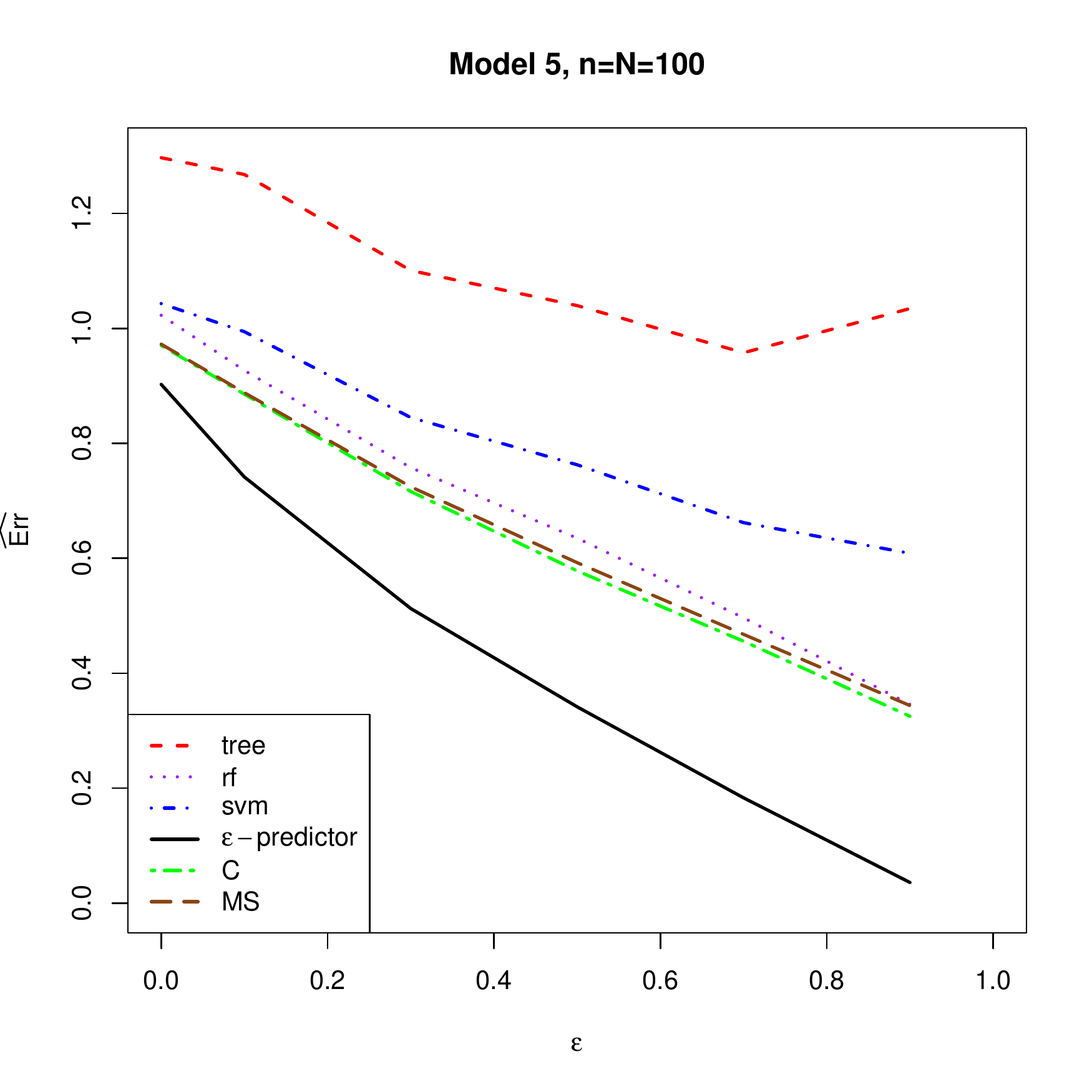}}
   \subfloat{\includegraphics[scale=0.24]{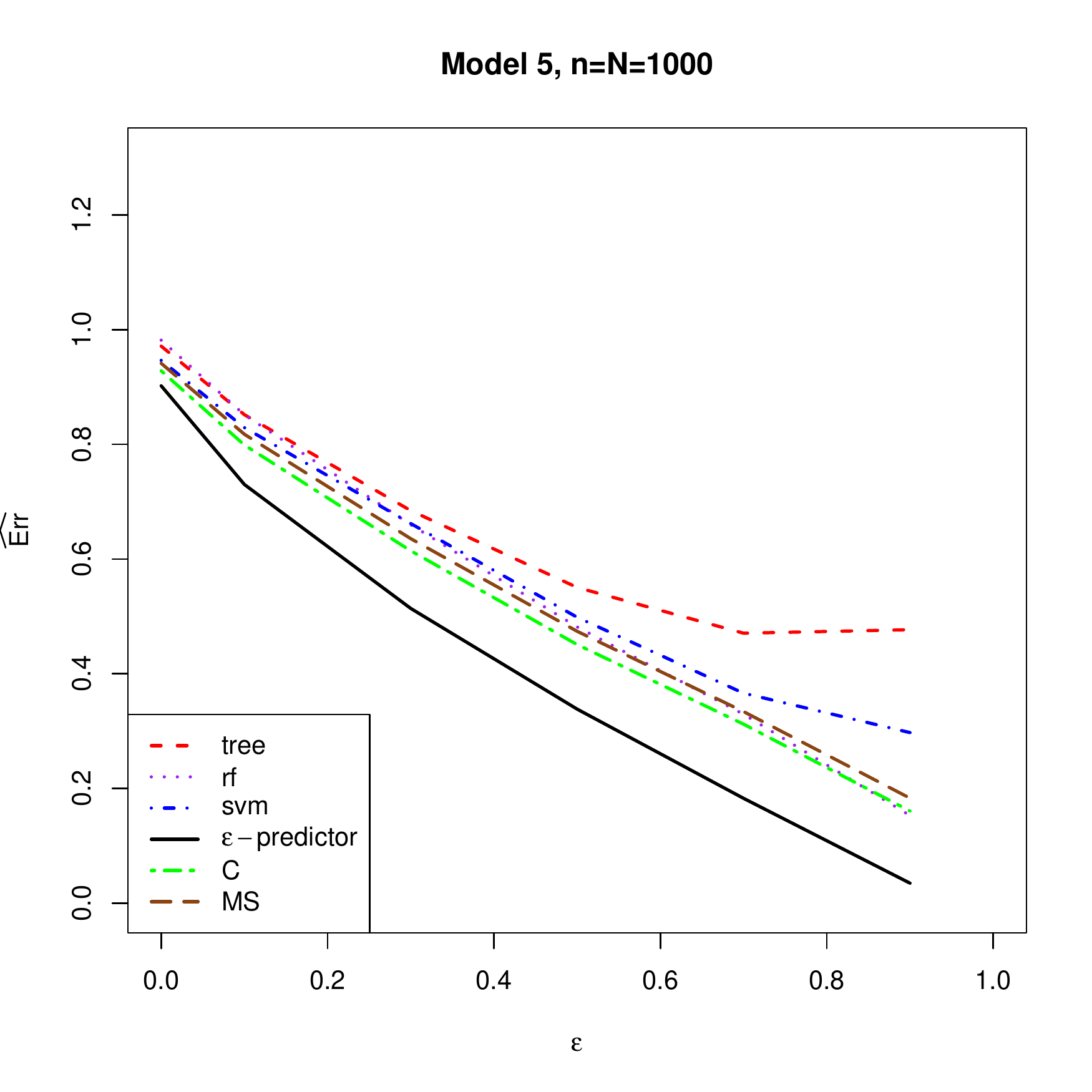}}
 \caption{Visual description of the performance of five plug-in $\varepsilon$-predictors.}%
 \label{fig:ErrorReject}%
\end{figure}

\section{Conclusion}
In the regression setting, we estimated the variance function by the model selection and convex aggregation when the set of initial estimators are constructed by the residual based-method. We called the estimators of the two procedures the \texttt{MS}-estimator and \texttt{C}-estimator respectively. We established the consistency of our estimators under mild assumptions and provided rate of convergence for these methods in $L_2$-norm that are of order $O(\log(M_1)/N)^{1/8})$ when $Y$ is satisfied the gaussian model; and $O(\log(M_1)/N)^{1/4})$ when $Y$ is bounded.\\

\bibliographystyle{plain}
\bibliography{sample}

\begin{thebibliography}{10}

\bibitem{Anderson_Lund97}
T.G. Anderson and J.~Lund.
\newblock Estimating continuous-time stochastic volatility models of the
  short-term interest rate.
\newblock {\em Journal of Econometrics}, 77(2):343--377, 1997.

\bibitem{Audibert04}
J.Y. Audibert.
\newblock Aggregated estimators and empirical complexity for least square
  regression.
\newblock {\em Ann. Inst. H. Poincaré Probab. Statist.}, 40(6):685--736, 2004.

\bibitem{Audibert09}
J.Y. Audibert.
\newblock Robust linear least squares regression.
\newblock {\em Annals of Statistics}, 37(4):1591--1646, 2009.

\bibitem{Biau_Devroye15}
G.~Biau and L.~Devroye.
\newblock {\em Lectures on the Nearest Neighbor Method}.
\newblock Springer Series in the Data Sciences. Springer New York, 2015.

\bibitem{BL07}
L.D. Brown and M.~Levine.
\newblock Variance estimation in nonparametric regression via the difference
  sequence.
\newblock {\em Annals of statistics}, 35(5):2219--2232, 2007.

\bibitem{Brown_Levine07}
L.D. Brown and M.~Levine.
\newblock Variance estimation in nonparametric regression via the difference
  sequence method.
\newblock {\em Annals of statistics}, 35(5):2219--2232, 2007.

\bibitem{Bunea_Tsybakov_Wegkamp07}
F.~Bunea, A.B. Tsybakov, and M.H. Wegkamp.
\newblock Aggregation for gaussian regression.
\newblock {\em Annals of Statistics}, 35(4):1674--1697, 2007.

\bibitem{Chow57}
C.~Chow.
\newblock An optimum character recognition system using decision functions.
\newblock {\em IRE Transactions on Electronic Computers}, (4):247--254, 1957.

\bibitem{Chow70}
C.~Chow.
\newblock On optimum error and reject trade-off.
\newblock {\em IEEE Transactions on Information Theory}, 16:41--46, 1970.

\bibitem{Denis_Hebiri19}
C.~Denis and M.~Hebiri.
\newblock Consistency of plug-in confidence sets for classification in
  semi-supervised learning.
\newblock {\em Journal of Nonparametric Statistics}, 2019.

\bibitem{Denis_Hebiri_Zaoui20}
C.~Denis, M.~Hebiri, and A.~Zaoui.
\newblock Regression with reject option and application to knn.
\newblock NeurIPS, 2020.

\bibitem{Fan92}
J.~Fan.
\newblock Design-adaptive nonparametric regression.
\newblock {\em Journal of the American Statistical Association},
  87(420):998--1004, 1992.

\bibitem{Fan_Yao98}
J.~Fan and Q.~Yao.
\newblock Efficient estimation of conditional variance functions in stochastic
  regression.
\newblock {\em Biometrika}, 85(3):645--660, 1998.

\bibitem{FriedmanHastieTibshirani10}
J.~Friedman, T.~Hastie, and R.~Tibshirani.
\newblock { Regularization paths for generalized linear models via coordinate
  descent}.
\newblock {\em Journal of Statistical Software}, 33:1--22, 2010.

\bibitem{GKKW02}
L.~Gyorfi, M.~Kohler, A.~Krzyzak, and H.~Walk.
\newblock {\em A distribution-free theory of nonparametric regression}.
\newblock Springer-Verlag, New York, 2002.

\bibitem{Hall_Caroll89}
P.~Hall and R.J. Carroll.
\newblock Variance function estimation in regression: the effect of estimating
  the mean.
\newblock {\em Journal of the Royal Statistical Society. Series B
  (Methodological)}, 51(1):3--14, 1989.

\bibitem{Hardle_Tsybakov97}
W.~H\"ardle and A.B. Tsybakov.
\newblock Local polynomial estimators of the volatility function in
  nonparametric autoregression.
\newblock {\em Journal of Econometrics}, 81(1):223--242, 1997.

\bibitem{Herbei_Wegkamp06}
R.~Herbei and M.~Wegkamp.
\newblock Classification with reject option.
\newblock {\em The Canadian Journal of Statistics}, 34(4):709--721, 2006.

\bibitem{Juditsky_Nemirovski00}
A.~Juditsky and A.~Nemirovski.
\newblock Functional aggregation for nonparametric regression.
\newblock {\em Annals of Statistics}, 28(3):681--712, 2000.

\bibitem{KulikWichelhaus11}
R.~Kulik and C.~Wichelhaus.
\newblock {Nonparametric conditional variance and error density estimation in
  regression models with dependent errors and predictors}.
\newblock {\em Electron. J. Statist.}, 5:856--898, 2011.

\bibitem{Lecue13}
G.~Lecué.
\newblock Empirical risk minimization is optimal for the convex aggregation
  problem.
\newblock {\em Bernoulli}, 19(5B):2153--2166, 2013.

\bibitem{Lecue_Mendelson09}
G.~Lecué and S.~Mendelson.
\newblock Aggregation via empirical risk minimization.
\newblock {\em Probability theory and related fields}, 145(3-4):591--613, 2009.

\bibitem{Lei14}
J.~Lei.
\newblock Classification with confidence.
\newblock {\em Biometrika}, 101(4):755--769, 2014.

\bibitem{Li19}
S.~Li.
\newblock {Fnn: Fast nearest neighbor search algorithms and applications}.
\newblock 2019.

\bibitem{LiawWiener02}
A.~Liaw and M.~Wiener.
\newblock {Classification and regression by randomforest}.
\newblock {\em R News}, 2:18--22, 2002.

\bibitem{Mammen_Nielsen_Scholz_Sperlich19}
E.~Mammen, J.P. Nielsen, M.~Scholz, and S.~Sperlich.
\newblock Conditional variance forecasts for long-term stock returns.
\newblock {\em Machine learning in insurance}, 7(4), 2019.

\bibitem{Muller_Stadtmuller87}
H.G. M\"uller and U.~Stadtm\"uller.
\newblock Estimation of heteroscedasticity in regression analysis.
\newblock {\em The annals of statistics}, 15(2):610--625, 1987.

\bibitem{Naadeem_Zucker_Hanczar10}
M.~Naadeem, J.D. Zucker, and B.~Hanczar.
\newblock Accuracy-rejection curves ({ARC}s) for comparing classification
  methods with a reject option.
\newblock In {\em MLSB}, pages 65--81, 2010.

\bibitem{Neumann94}
M.H. Neumann.
\newblock Fully data-driven nonparametric variance estimators.
\newblock {\em Statistics}, 25:189--212, 1994.

\bibitem{Opsomer_Ruppert_Wand_Holst_Hossjer99}
J.D. Opsomer, D.~Ruppert, M.P Wand, U.~Holst, and O.~Hossjer.
\newblock Kriging with nonparametric variance function estimation.
\newblock {\em Biometrics}, 55(3):704--710, 1999.

\bibitem{Ripley19}
B.~Ripley.
\newblock tree: Classification and regression trees.
\newblock 2019.

\bibitem{Ruppert_Wand_Holst_Hosjer97}
D.~Ruppert, M.P. Wand, U.~Holst, and O.~H\"oSJER.
\newblock Local polynomial variance function estimation.
\newblock {\em Technometrics}, 39(3):262--273, 1997.

\bibitem{scornet2015}
E.~Scornet, G.~Biau, and J.-P. Vert.
\newblock Consistency of random forests.
\newblock {\em Annals of Statistics}, 43(4):1716--1741, 08 2015.

\bibitem{Stone77}
C.~Stone.
\newblock Consistent nonparametric regression.
\newblock {\em Annals of Statistics}, pages 595--620, 1977.

\bibitem{Tsybakov08}
A.~Tsybakov.
\newblock {\em Introduction to Nonparametric Estimation}.
\newblock Springer Series in Statistics. Springer New York, 2008.

\bibitem{Tsybakov03}
A.B. Tsybakov.
\newblock Optimal rates of aggregation.
\newblock {\em Learning Theory and Kernel Machines}, pages 303--313, 2003.

\bibitem{Tsybakov14}
A.B. Tsybakov.
\newblock Aggregation and minimax optimality in high-dimensional estimation.
\newblock {\em Proceedings of International Congress of Mathematicians},
  3:225--246, 2014.

\bibitem{Verzelen_Gassiat18}
N.~Verzelen and E.~Gassiat.
\newblock Adaptive estimation of high-dimensional signal-to-noiseratios.
\newblock {\em Bernoulli}, 24(4B):3683--3710, 2018.

\bibitem{Vovk_Gammerman_Shafer05}
V.~Vovk, A.~Gammerman, and G.~Shafer.
\newblock {\em Algorithmic learning in a random world}.
\newblock Springer, New York, 2005.

\bibitem{Wang_Brown_Cai_Levine08}
L.~Wang, L.~D.Brown, T.Tony Cai, and M.~Levine.
\newblock Effect of mean on variance function estimation in nonparametric
  regression.
\newblock {\em Annals of Statistics}, 36(2):646--664, 2008.

\bibitem{Xu_Phillips11}
K.~Xu and P.C.~B Phillips.
\newblock Tilted nonparametric estimation of volatility functions with
  empirical applications.
\newblock {\em Journal of Business \& Economic Statistics}, 29(4):518--528,
  2011.

\bibitem{Yu_Jones04}
K.~Yu and M.~Jones.
\newblock Likelihood based-local linear estimation of the conditional variance
  function.
\newblock {\em Journal of the American Statistical Association},
  99(465):139--144, 2004.

\bibitem{Yuhong04}
Y.~Yuhong.
\newblock Aggregating regression procedures to improve performance.
\newblock {\em Bernoulli}, 10(1):25--47, 2004.

\bibitem{Ziegelmann02}
Flavio~A. Ziegelmann.
\newblock Nonparametric estimation of volatility functions: The local
  exponential estimator.
\newblock {\em Econometric Theory}, 18:985--991, 2002.

\end{thebibliography}


\newpage

\appendix
\section*{Appendix}
This section gathers the proof of our results.
\section{Proof of Theorem~\ref{thm:Risk_MS_sigma3}}
Note that the quantity $\mathbb{E}\left[|\hat{\sigma}^{2}_{\texttt{MS}}(X)-\sigma^{2}(X)|^2\right]$ is the exces risk of the estimator $\hat{\sigma}^{2}_{\texttt{MS}}$ and  defines as follow 
$$
\mathbb{E}\left[|\hat{\sigma}^{2}_{\texttt{MS}}(X)-\sigma^{2}(X)|^2\right]:=\mathbb{E}\left[R(\hat{\sigma}^{2}_{\texttt{MS}})- R(\sigma^2)\right]\enspace,
$$
where $R(\sigma^2)=\mathbb{E}\left[|Z-\sigma^{2}(X)|^2\right]$ is the true risk of the variance function. 
Besides, we introduce a minimizer of the risk $R$, denoted by $\bar{\sigma}^{2}_{\texttt{MS}}$ and  given 
\begin{equation}
\label{eq:bar_sigmaMS}
\bar{\sigma}^{2}_{\texttt{MS}}:= \hat{\sigma}^{2}_{\hat{s},\bar{m}}\enspace,\;\text{where}\enspace \enspace \bar{m}\in \argmin_{m\in[M_2]} R(\hat{\sigma}^{2}_{\hat{s},m}).
\end{equation}
We consider the following decomposition
\begin{equation}
R(\hat{\sigma}^{2}_{\texttt{MS}})- R(\sigma^2)= \underbrace{R(\hat{\sigma}^{2}_{\texttt{MS}})-R(\bar{\sigma}^{2}_{\texttt{MS}})}_{\text{estimation error}}+\underbrace{R(\bar{\sigma}^{2}_{\texttt{MS}})- R(\sigma^2)}_{\text{approximation error}}.
\label{eq:decompositionExcesRisk}
\end{equation}
Each of these errors is obviously positive.  The random term $R(\hat{\sigma}^{2}_{\texttt{MS}})-R(\bar{\sigma}^{2}_{\texttt{MS}})$ is called the estimation error (or the variance). It measures how close $\hat{\sigma}^{2}_{\texttt{MS}}$ is to the best possible rule in $[M_2]$ in terms of the risk $R$. The deterministic term $R(\bar{\sigma}^{2}_{\texttt{MS}})- R(\sigma^2)$ is called the approximation error (or the bias). We start with the following lemma
\begin{lemma}
\label{lem:aprox_error}
Let $\bar{\sigma}^{2}_{\texttt{MS}}$ be an aggregate defined in Equation~\eqref{eq:bar_sigmaMS}. Then, 
\begin{eqnarray*}
\mathbb{E}\left[R(\bar{\sigma}^{2}_{\texttt{MS}})-R(\sigma^2)\right]=\mathbb{E}\left[\min_{m\in[M_2]}\mathbb{E}_{X}\left[|\hat{\sigma}^{2}_{\hat{s},m}(X)-\sigma^{2}(X)|^{2}\right]\right].
\end{eqnarray*}
\end{lemma}
This result explicitly determines the approximation error.
\begin{proof}[Proof of Lemma~\ref{lem:aprox_error}]
For all $m\in[M_{2}]$, the excess risk of $\hat{\sigma}^{2}_{\hat{s},m}$ is given as follow
\begin{equation}
\label{eq:Lemma1}
\mathbb{E}\left[|Z-\hat{\sigma}^{2}_{\hat{s},m}(X)|^{2}\right]-\mathbb{E}\left[|Z-\sigma^{2}(X)|^{2}\right]=\mathbb{E}_{X}\left[|\hat{\sigma}^{2}_{\hat{s},m}(X)-\sigma^{2}(X)|^{2}\right].
\end{equation}
We apply $min$ in Equation~\eqref{eq:Lemma1} and we get
\begin{eqnarray*}
\mathbb{E}\left[R(\bar{\sigma}^{2}_{\texttt{MS}})-R(\sigma^2)\right]&=&\mathbb{E}\left[\min_{m\in[M_2]}\mathbb{E}_{X}\left[|\hat{\sigma}^{2}_{\hat{s},m}(X)-\sigma^{2}(X)|^{2}\right]\right].
\end{eqnarray*}
\end{proof}
\begin{proof}[Proof of Theorem~\ref{thm:Risk_MS_sigma3}]
We thanks the decomposition in Eq.~\eqref{eq:decompositionExcesRisk}, we have
\begin{equation}
\mathbb{E}\left[|\hat{\sigma}^{2}_{\texttt{MS}}(X)-\sigma^{2}(X)|^2\right]= \mathbb{E}\left[R(\hat{\sigma}^{2}_{\texttt{MS}})-R(\bar{\sigma}^{2}_{\texttt{MS}})\right]+\mathbb{E}\left[R(\bar{\sigma}^{2}_{\texttt{MS}})- R(\sigma^2)\right].
\label{eq:Eq1}
\end{equation}
\item \textbf{Step 1.} Study of the term $\mathbb{E}\left[R(\bar{\sigma}^{2}_{\texttt{MS}})- R(\sigma^2)\right]$. We begin with the following Lemma
\begin{lemma}
\label{lem:ContrProhatsneqStar_s}
Let $\hat{s}$ and $s^*$ be two estimators defined in~\eqref{est:hats} and~\eqref{est:star(s)} respectively. Then, under Assumptions~\ref{ass:fstar_bound},~\ref{ass:Y_bound},~\ref{ass:hatfs_hatsigma_hats_m_bound} and~\ref{ass:SeperHyp} there exists an absolute constant $C$ such that 
\begin{eqnarray*}
\mathbb{P}\left(\hat{s}\neq s^{*}\right)\leq C\left(\frac{\log(M_1)}{N}\right)^{1/2}.
\end{eqnarray*}
\end{lemma}
\begin{proof}
Under Assumption~\ref{ass:SeperHyp}, we have firstly  $\delta^*\left(\mathcal{D}_n\right)=\min_{s\neq s^*}\left\{|\mathcal{R}(\hat{f}_{s^*})-\mathcal{R}(\hat{f}_{s})|\right\}>\delta_0>0$. Recall that $\mathcal{R}(\hat{f}_{s^*})\leq \mathcal{R}(\hat{f}_{\hat{s}})$. On the event $\left\{\hat{s}\neq s^*\right\}$, we have two cases
\begin{itemize}
\item $\hat{\mathcal{R}}_{N}(\hat{f}_{\hat{s}}) <\mathcal{R}(\hat{f}_{s^*})$, and then 
\begin{equation*}
\delta^*\leq |\mathcal{R}(\hat{f}_{\hat{s}})-\mathcal{R}(\hat{f}_{s^*})|\leq |\hat{\mathcal{R}}_{N}(\hat{f}_{s^{*}})-\mathcal{R}(\hat{f}_{s^*})|\leq \max_{s\in[M_1]}|\hat{\mathcal{R}}_{N}(\hat{f}_{s})-\mathcal{R}(\hat{f}_{s})|
\end{equation*}
\item $\hat{\mathcal{R}}_{N}(\hat{f}_{\hat{s}}) \geq\mathcal{R}(\hat{f}_{s^*})$, and then 
\begin{equation*}
\delta^*\leq |\hat{\mathcal{R}}_{N}(\hat{f}_{\hat{s}})-\mathcal{R}(\hat{f}_{\hat{s}})|+|\hat{\mathcal{R}}_{N}(\hat{f}_{s^{*}})-\mathcal{R}(\hat{f}_{s^*})|\leq 2\max_{s\in[M_1]}|\hat{\mathcal{R}}_{N}(\hat{f}_{s})-\mathcal{R}(\hat{f}_{s})|.
\end{equation*}
\end{itemize}
Therefore,
\begin{eqnarray*}
\mathbb{P}\left(\hat{s}\neq s^*\right)\leq \mathbb{P}\left(\max_{s\in[M_1]}|\hat{\mathcal{R}}_{N}(\hat{f}_{s})-\mathcal{R}(\hat{f}_{s})|\geq \delta_0/2\right)
\end{eqnarray*}
We control this term using Bernstein's inequality. We check that the conditions for Bernstein's inequality are satisfied. For all $s\in [M_1]$, set $V_{i}(s)=|Y_i-\hat{f}_{s}(X_i)|^{2}=|f^{*}(X_i)-\hat{f}_{s}(X_i)+\sigma(X_i)\xi_i|^{2}$ for all $i=1,\cdots, N$. First, Assumptions~\ref{ass:fstar_bound} and~\ref{ass:hatfs_hatsigma_hats_m_bound} ensure that there exist a positive constants $L_1$ and $L_2$ such that $|f^{*}(X)-\hat{f}_{s}(X)|\leq L_{1}$ and $|\sigma^{2}(X)|\leq L_2$. Second, note that since the variables $V_{i}(s)$ are \iid and by the elementary inequality $(x+y)^{4}\leq 2^{3}(x^{4}+y^{4})$ for all $x, y\in \mathbb{R}$, by Lemma~\ref{lem:momentsBound}, and by  the elementary inequality $x^4+y^4\leq (x+y)^4$ for all $x,y\geq 0$ we have
\begin{eqnarray*}
\sum_{i=1}^{N}\mathbb{E}\left[V^{2}_{i}(s)\right] \leq 2^{3}\sum_{i=1}^{N}\mathbb{E}\left[|f^{*}(X)-\hat{f}_{s}(X)|^4+\sigma^{4}(X_i)\xi_{i}^{4}\right]\leq 2^{7}N(L_{1}+\sqrt{L_{2}})^{4}:=v_{N}\enspace,
\end{eqnarray*} 
and for $k\geq 3$  we follow the elementary inequality $(x+y)^{2k}\leq 2^{2k-1}(x^{2k}+y^{2k})$ for all $x,y\in \mathbb{R}$, Lemma~\ref{lem:momentsBound}, and the following elementary inequality $ x^{2k}+y^{2k}\leq (x+y)^{2k}$ for all $x,y\geq 0$ 
\begin{eqnarray*}
\sum_{i=1}^{N}\mathbb{E}\left[(V^{k}_{i}(s)\vee 0)\right]&=& \sum_{i=1}^{N}\mathbb{E}\left[|f^{*}(X_i)-\hat{f}_{s}(X_i)+\sigma(X_i)\xi_i|^{2k}\right]\\
&\leq & 2^{2k-1}\sum_{i=1}^{N}\mathbb{E}\left[|f^{*}(X_i)-\hat{f}_{s}(X_i)|^{2k}+|\sigma^{2}(X_i)|^{k}|\xi_{i}|^{2k}\right]\\
&\leq & \frac{1}{2} 2^{2k} N\left(L_{1}^{2k}+2^{k+1}(\sqrt{L_2})^{2k}(k)!\right)\\
&\leq & \frac{1}{2}2^{3k+1}N\left(L_{1}+\sqrt{L_2}\right)^{2k}k!\\
&\leq & \frac{1}{2}v_{N}c^{k-2}k!\enspace.
\end{eqnarray*}
where $c:=8\left(L_{1}+\sqrt{L_2}\right)^2$. Using the Bernstein's inequality (Lemma~\ref{lem:BernsteinIneq}), we get for all $s\in [M_1]$
\begin{equation*}
\mathbb{P}\left(|\hat{\mathcal{R}}_{N}(\hat{f}_{s})-\mathcal{R}(\hat{f}_{s})|\geq \frac{\delta_0}{2}\right)\leq 2\exp\left(-\frac{N\delta_{0}^2}{2^{10}(L_{1}+\sqrt{L_{2}})^{4}+4c\delta_0}\right)
\end{equation*}
By union bound on $s\in [M_1]$, we obtain
\begin{equation*}
\mathbb{P}\left(\hat{s}\neq s^*\right)\leq 2\exp\left(\log(M_1)-\frac{N\delta_0^2}{2^{10}(L_{1}+\sqrt{L_{2}})^{4}+4c\delta_0}\right)\leq C\left(\frac{\log(M_1)}{N}\right)^{1/2}\enspace,
\end{equation*}
where $C$ is a positive constant which depends on $L_1$, $L_2$ and $\delta_0$.
\end{proof}
By Lemmas~\ref{lem:aprox_error} and~\ref{lem:ContrProhatsneqStar_s}, and under Assumptions~\ref{ass:fstar_bound} and~\ref{ass:hatfs_hatsigma_hats_m_bound} we  get
\begin{eqnarray*}
\mathbb{E}\left[R(\bar{\sigma}^{2}_{\texttt{MS}})-R(\sigma^2)\right]&=&\mathbb{E}\left[\min_{m\in[M_2]}\mathbb{E}_{X}\left[|\hat{\sigma}^{2}_{\hat{s},m}(X)-\sigma^{2}(X)|^{2}\left\{\one_{\left\{\hat{s}=s^*\right\}}+\one_{\left\{\hat{s}\neq s^*\right\}}\right\}\right]\right]\\
&\leq & \mathbb{E}\left[\min_{m\in[M_2]}\mathbb{E}_{X}\left[|\hat{\sigma}^{2}_{s^*,m}(X)-\sigma^{2}(X)|^{2}\right]\right] + C\left(\frac{\log(M_1)}{N}\right)^{1/2}\enspace
\end{eqnarray*}
where $C$ is a constant which depends on $K_2$ , $\sigma^2$ and the constant in Lemma~\ref{lem:ContrProhatsneqStar_s}.
\item \textbf{Step 2.}  Study of the term $\mathbb{E}\left[R(\hat{\sigma}^{2}_{\texttt{MS}})-R(\bar{\sigma}^{2}_{\texttt{MS}})\right]$. 
To treat the estimation error, we introduce an aggregate $\tilde{\sigma}^{2}_{\texttt{MS}}$ which is based on minimization of the empirical risk of $R$
\begin{equation*}
\tilde{\sigma}^{2}_{\texttt{MS}}:= \hat{\sigma}^{2}_{\hat{s},\tilde{m}}\enspace,\;\text{where}\enspace \enspace \tilde{m}\in \argmin_{m\in[M_2]} R_{N}(\hat{\sigma}^{2}_{\hat{s},m})\enspace,
\end{equation*}
with $R_N(\hat{\sigma}^{2}_{\hat{s},m})=\frac{1}{N}\sum_{i=1}^{N}|Z_i-\hat{\sigma}^{2}_{\hat{s},m}(X_i)|^2$. Moroever, we consider the decomposition
\begin{equation*}
\mathbb{E}\left[R(\hat{\sigma}^{2}_{\texttt{MS}})-R(\bar{\sigma}^{2}_{\texttt{MS}})\right]=\mathbb{E}\left[R(\hat{\sigma}^{2}_{\texttt{MS}})-R(\tilde{\sigma}^{2}_{\texttt{MS}})\right]+\mathbb{E}\left[R(\tilde{\sigma}^{2}_{\texttt{MS}})-R(\bar{\sigma}^{2}_{\texttt{MS}})\right].
\end{equation*}
\item \textbf{Step 2.1} Study of the term $\mathbb{E}\left[R(\tilde{\sigma}^{2}_{\texttt{MS}})-R(\bar{\sigma}^{2}_{\texttt{MS}})\right]$. We decompose the term $\mathbb{E}\left[R(\tilde{\sigma}^{2}_{\texttt{MS}})-R(\bar{\sigma}^{2}_{\texttt{MS}})\right]$ into two positive terms
\begin{equation}
\mathbb{E}\left[R(\tilde{\sigma}^{2}_{\texttt{MS}})-R(\bar{\sigma}^{2}_{\texttt{MS}})\right]=\mathbb{E}\left[R(\tilde{\sigma}^{2}_{\texttt{MS}})-R_{N}(\tilde{\sigma}^{2}_{\texttt{MS}})\right]+\mathbb{E}\left[R_{N}(\tilde{\sigma}^{2}_{\texttt{MS}})-R(\bar{\sigma}^{2}_{\texttt{MS}})\right].
\label{eq:Eq2}
\end{equation}
We use the fact that $R_{N}(\tilde{\sigma}^{2}_{\texttt{MS}})\leq R_{N}(\bar{\sigma}^{2}_{\texttt{MS}})$ in Eq.~\eqref{eq:Eq2}, and we get the uniform bound
\begin{equation*}
\mathbb{E}\left[R(\tilde{\sigma}^{2}_{\texttt{MS}})-R(\bar{\sigma}^{2}_{\texttt{MS}})\right]\leq 2\mathbb{E}\left[\max_{(s,m)\in [M_1]\times[M_2]}|R_N(\hat{\sigma}^{2}_{s,m})-R(\hat{\sigma}^{2}_{s,m})|\right].
\end{equation*}
Then using Assumption~\ref{ass:Y_bound}, for some $(s,m)\in [M_1]\times[M_2]$, set $T_{i}(s,m)=|Z_i-\hat{\sigma}^{2}_{s,m}(X_i)|^{2}=|\sigma^{2}(X_i)\xi^{2}_{i}-\hat{\sigma}^{2}_{s,m}(X_i)|^{2}$ for all $i=1,\cdots, N$. First, note that since the variables $T_{i}(s,m)$ are \iid, conditionally on $\mathcal{D}_n$ we have 
\begin{eqnarray*}
|R_N(\hat{\sigma}^{2}_{s,m})-R(\hat{\sigma}^{2}_{s,m})|&=&\bigg|\frac{1}{N}\sum_{i=1}^{N}(T_{i}(s,m)-\mathbb{E}[T_{i}(s,m)])\bigg|\\
&\leq & \bigg|\frac{1}{N}\sum_{i=1}^{N}(T_{i}(s,m)-\mathbb{E}[T_{i}(s,m)])\one_{\{|\xi_{i}|\leq L\}}\bigg|+\bigg|\frac{1}{N}\sum_{i=1}^{N}(T_{i}(s,m)-\mathbb{E}[T_{i}(s,m)])\one_{\{|\xi_{i}|> L\}}\bigg|
\end{eqnarray*}
for any $L>0$. Therefore, conditionally on $\mathcal{D}_n$
\begin{equation}
\label{eq:decompdeviationbound}
\begin{split}
\mathbb{E}\left[\max_{(s,m)\in [M_1]\times[M_2]}|R_N(\hat{\sigma}^{2}_{s,m})-R(\hat{\sigma}^{2}_{s,m})|\right]&\leq  \mathbb{E}\left[\max_{(s,m)\in [M_1]\times[M_2]}\bigg|\frac{1}{N}\sum_{i=1}^{N}(T_{i}(s,m)-\mathbb{E}[T_{i}(s,m)])\one_{\{|\xi_{i}|\leq L\}}\bigg|\right]\\
&+\mathbb{E}\left[\max_{(s,m)\in [M_1]\times[M_2]}\bigg|\frac{1}{N}\sum_{i=1}^{N}(T_{i}(s,m)-\mathbb{E}[T_{i}(s,m)])\one_{\{|\xi_{i}|> L\}}\bigg|\right].
\end{split}
\end{equation}
\textbf{Step 2.1.1}. We control the first term on the r.h.s. of Eq.~\eqref{eq:decompdeviationbound}.  On the event $\{|\xi| \leq L\}$ and under Assumptions~\ref{ass:fstar_bound} and~\ref{ass:hatfs_hatsigma_hats_m_bound}, we get $|T_{i}(s,m)|\leq c_{1}L^4+2K_{2}^{2}$ for all $i=1,\cdots, N$ for some $c_{1}>0$ that depends on $\sigma^2$. Conditionally on $\mathcal{D}_n$, we apply Hoeffding's inequality, for all $(s,m)\in [M_1]\times[M_2]$, and all $t\geq 0$
\begin{equation*}
\mathbb{P}\left(\bigg|\frac{1}{N}\sum_{i=1}^{N}(T_{i}(s,m)-\mathbb{E}[T_{i}(s,m)])\one_{\{|\xi_{i}|\leq L\}}\bigg|\geq t\right) \leq 2 \exp\left(-\frac{Nt^{2}}{2(c_{1}L^{4}+2K_{2}^{2})^2}\right)\enspace,
\end{equation*}
Conditionally on $\mathcal{D}_n$, by a union bound on $(s,m)\in[M_1]\times[M_2]$, we deduce that for all $t\geq 0$
\begin{equation*}
\mathbb{P}\left(\max_{(s,m)\in [M_1]\times[M_2]}\bigg|\frac{1}{N}\sum_{i=1}^{N}(T_{i}(s,m)-\mathbb{E}[T_{i}(s,m)])\one_{\{|\xi_{i}|\leq L\}}\bigg|\geq t\right) \leq 2 \exp\left(\log(M_{1}M_2)-\frac{Nt^{2}}{2(c_{1}L^{4}+2K_{2}^{2})^2}\right).
\end{equation*}
We apply Lemma~\ref{lem:technProba}. Then, there exists a positive constant $\mathbf{c}$ such that
\begin{equation*}
 \mathbb{E}\left[\max_{(s,m)\in [M_1]\times[M_2]}\bigg|\frac{1}{N}\sum_{i=1}^{N}(T_{i}(s,m)-\mathbb{E}[T_{i}(s,m)])\one_{\{|\xi_{i}|\leq L\}}\bigg|\right]\leq \mathbf{c}\left(c_{2}L^{4}+c_{3}\right)\left(\frac{\log(M_{1}M_{2})}{N}\right)^{1/2},
\end{equation*}
where $c_2$ is a positive constant that depends on $c_1$ and $c_{3}$ depends on $K_2$.
\\ \textbf{Step 2.1.2} We control the second term on the r.h.s. of Eq.~\eqref{eq:decompdeviationbound}. By union bound on $(s,m)\in [M_1]\times [M_2]$, by Cauchy–Schwarz inequality, under Assumptions~\ref{ass:fstar_bound},~\ref{ass:Y_bound} and~\ref{ass:hatfs_hatsigma_hats_m_bound}, and Lemma~\ref{lem:gaussianProb} we obtain
\begin{eqnarray*}
&&\mathbb{E}\left[\max_{(s,m)\in [M_1]\times[M_2]}\bigg|\frac{1}{N}\sum_{i=1}^{N}(T_{i}(s,m)-\mathbb{E}[T_{i}(s,m)])\one_{\{|\xi_{i}|> L\}}\bigg|\right]\\
&\leq & \sum_{s=1}^{M_1}\sum_{m=1}^{M_2}\frac{1}{N}\sum_{i=1}^{N}\mathbb{E}[|T_{i}(s,m)-\mathbb{E}[T_{i}(s,m)]|\one_{\{\xi_{i}> L\}}]\\
&\leq & \sum_{s=1}^{M_1}\sum_{m=1}^{M_2}\frac{1}{N}\sum_{i=1}^{N}\sqrt{\mathbb{E}[|T_{i}(s,m)-\mathbb{E}[T_{i}(s,m)]|^2]\mathbb{P}(|\xi_{i}|> L)}]\\
&\leq & cM_{1}M_{2}\sqrt{\mathbb{P}(|\xi_1|> L)}\\
&\leq &c M_{1}M_{2}\frac{\exp(-L^2/4)}{L^{1/2}} \enspace, 
\end{eqnarray*}
where $c$ is a positive constant which depends on $\xi$, $\sigma^2$ and $K_2$. 
\\Combining the results of the \textbf{Step 2.1.1} and \textbf{Step 2.1.2} in Eq.~\eqref{eq:decompdeviationbound}
\begin{equation*}
\mathbb{E}\left[\max_{(s,m)\in [M_1]\times[M_2]}|R_N(\hat{\sigma}^{2}_{s,m})-R(\hat{\sigma}^{2}_{s,m})|\right]\leq  \mathbf{c}(c_{2}L^{4}+c_3)\left(\frac{\log(M_{1}M_{2})}{N}\right)^{1/2}+c M_{1}M_{2}\frac{\exp(-L^2/4)}{L^{1/2}}.
\end{equation*}
Choosing $L=2\sqrt{\log(N)}$ and we get 
\begin{equation*}
\mathbb{E}\left[\max_{(s,m)\in [M_1]\times[M_2]}|R_N(\hat{\sigma}^{2}_{s,m})-R(\hat{\sigma}^{2}_{s,m})|\right]\leq  C\left(\frac{\log(N)^{4}\log(M_{1}M_{2})}{N}\right)^{1/2}\enspace,
\end{equation*}
where $C$ is a positive constant that depends on $c_2$ and $\mathbf{c}$, and \textbf{Step 2.1.2} is finished.

 We combine the results of the \textbf{Step 2.1.1} and \textbf{Step 2.1.2} and we get the following bound
$$
\mathbb{E}\left[R(\tilde{\sigma}^{2}_{\texttt{MS}})-R(\bar{\sigma}^{2}_{\texttt{MS}})\right]\leq 2C\left(\frac{\log(N)^{4}\log(M_{1}M_{2})}{N}\right)^{1/2}.
$$
\begin{remark}
It is clear that when $Y$ is bounded, there exists an absolute constant $C>0$ such that
$$
\mathbb{E}\left[R(\tilde{\sigma}^{2}_{\texttt{MS}})-R(\bar{\sigma}^{2}_{\texttt{MS}})\right]\leq C\left(\frac{\log(M_{1}M_{2})}{N}\right)^{1/2}.
$$
\end{remark}

\item \textbf{Step 2.2} Study of the term $
\mathbb{E}\left[R(\hat{\sigma}^{2}_{\texttt{MS}})-R(\tilde{\sigma}^{2}_{\texttt{MS}})\right]$. We start with the following decomposition
\begin{equation}
\label{eq:Eq3}
\mathbb{E}\left[R(\hat{\sigma}^{2}_{\texttt{MS}})-R(\tilde{\sigma}^{2}_{\texttt{MS}})\right]=\mathbb{E}\left[R(\hat{\sigma}^{2}_{\texttt{MS}})-R_{N}(\hat{\sigma}^{2}_{\texttt{MS}})\right]+\mathbb{E}\left[R_{N}(\hat{\sigma}^{2}_{\texttt{MS}})-R_{N}(\tilde{\sigma}^{2}_{\texttt{MS}})\right]+\mathbb{E}\left[R_{N}(\tilde{\sigma}^{2}_{\texttt{MS}})-R(\tilde{\sigma}^{2}_{\texttt{MS}})\right].
\end{equation}
We use the same arguments in \textbf{Step 2.1} to control the first term and the last term on the r.h.s. of Eq.~\eqref{eq:Eq3}, and we get the following bound
\begin{eqnarray*}
\mathbb{E}\left[R(\hat{\sigma}^{2}_{\texttt{MS}})-R_{N}(\hat{\sigma}^{2}_{\texttt{MS}})\right]+\mathbb{E}\left[R_{N}(\tilde{\sigma}^{2}_{\texttt{MS}})-R(\tilde{\sigma}^{2}_{\texttt{MS}})\right]&\leq & 2\mathbb{E}\left[\max_{(s,m)\in [M_1]\times[M_2]}|R_N(\hat{\sigma}^{2}_{s,m})-R(\hat{\sigma}^{2}_{s,m})|\right]\\
&\leq&  C\left(\frac{\log(N)^{4}\log(M_{1}M_{2})}{N}\right)^{1/2}\enspace.
\end{eqnarray*}
\begin{remark}
If $Y$ is bounded, there exists an absolute constant $C>0$ such that
$$
\mathbb{E}\left[R(\hat{\sigma}^{2}_{\texttt{MS}})-R_{N}(\hat{\sigma}^{2}_{\texttt{MS}})\right]+\mathbb{E}\left[R_{N}(\tilde{\sigma}^{2}_{\texttt{MS}})-R(\tilde{\sigma}^{2}_{\texttt{MS}})\right]\leq C\left(\frac{\log(M_{1}M_{2})}{N}\right)^{1/2}.
$$
\end{remark}
We study now the second term on the r.h.s. of Eq.~\eqref{eq:Eq3}. For that, we need the following decomposition
\begin{equation}
\label{eq:Eq4}
\mathbb{E}\left[R_{N}(\hat{\sigma}^{2}_{\texttt{MS}})-R_{N}(\tilde{\sigma}^{2}_{\texttt{MS}})\right]=\mathbb{E}\left[R_{N}(\hat{\sigma}^{2}_{\texttt{MS}})-\hat{R}_{N}(\hat{\sigma}^{2}_{\texttt{MS}})\right]+\mathbb{E}\left[\hat{R}_{N}(\hat{\sigma}^{2}_{\texttt{MS}})-R_{N}(\tilde{\sigma}^{2}_{\texttt{MS}})\right].
\end{equation}
Using $\hat{R}_{N}(\hat{\sigma}^{2}_{\texttt{MS}})\leq \hat{R}_{N}(\tilde{\sigma}^{2}_{\texttt{MS}})$ in Eq.~\eqref{eq:Eq4}, we obtain the following inequality
\begin{equation}
\label{eq:IneqUniformBound}
\mathbb{E}\left[R_{N}(\hat{\sigma}^{2}_{\texttt{MS}})-R_{N}(\tilde{\sigma}^{2}_{\texttt{MS}})\right]\leq 2\mathbb{E}\left[\max_{m\in[M_2]}|\hat{R}_{N}(\hat{\sigma}^{2}_{\hat{s},m})-R_{N}(\hat{\sigma}^{2}_{\hat{s},m})|\right].
\end{equation}
We control the term $\mathbb{E}\left[\max_{m\in[M_2]}|\hat{R}_{N}(\hat{\sigma}^{2}_{\hat{s},m})-R_{N}(\hat{\sigma}^{2}_{\hat{s},m})|\right] $. By definition of $\hat{R}_{N}$ and $R_{N}$, and under Assumption~\ref{ass:hatfs_hatsigma_hats_m_bound}, we get for all $m\in[M_2]$
\begin{eqnarray*}
|\hat{R}_{N}(\hat{\sigma}^{2}_{\hat{s},m})-R_{N}(\hat{\sigma}^{2}_{\hat{s},m})|&\leq & 
 \frac{1}{N}\sum_{i=1}^{N}|\hat{Z}_{i}-Z_{i}|^{2}+\frac{2}{N}\sum_{i=1}^{N}|\hat{Z}_{i}-Z_{i}|(|Z_i|+K_2)\enspace,
\end{eqnarray*}
where $K_2$ is the bound of $\hat{\sigma}^{2}_{\hat{s},m}$. The upper-bound of $|\hat{R}_{N}(\hat{\sigma}^{2}_{\hat{s},m})-R_{N}(\hat{\sigma}^{2}_{\hat{s},m})|$ does not depend on $m$, therefore 
\begin{equation}
\mathbb{E}\left[\max_{m\in[M_2]}|\hat{R}_{N}(\hat{\sigma}^{2}_{\hat{s},m})-R_{N}(\hat{\sigma}^{2}_{\hat{s},m})|\right]\leq \mathbb{E}\left[\frac{1}{N}\sum_{i=1}^{N}|\hat{Z}_{i}-Z_{i}|^{2}\right]+2\mathbb{E}\left[\frac{1}{N}\sum_{i=1}^{N}|\hat{Z}_{i}-Z_{i}|(|Z_i|+K_2)\right].
\label{eq: Ineg1}
\end{equation}
Note that, by Assumptions~\ref{ass:fstar_bound} and~\ref{ass:hatfs_hatsigma_hats_m_bound} we obtain for all $i=1,\cdots, N$
\begin{eqnarray*}
|f^{*}(X_i)-\hat{f}_{\texttt{MS}}(X_i)|\leq \|f^{*}\|_{\infty}+\max_{s\in [M_1]}\|\hat{f}_{s}\|_{\infty}\leq \|f^{*}\|_{\infty}+K_1\leq L_1< \infty.
\end{eqnarray*}
Since $x^2-y^2= (x-y)(x+y)$, $(x+y)^2\leq 2(x^2+y^2)$,  we obtain the following inequality for all $i=1,\cdots, N$
\begin{eqnarray}
|\hat{Z_i}-Z_i|^2&=&|(Y_i-\hat{f}_{\texttt{MS}}(X_i))^2-(Y_i-f^{*}(X_i))^2|^2 \notag \\
&=& |(f^{*}(X_i)-\hat{f}_{\texttt{MS}}(X_i))(2(Y_i-f^{*}(X_i))+(f^{*}(X_i)-\hat{f}_{\texttt{MS}}(X_i))|^{2}\notag\\
&\leq & |f^{*}(X_i)-\hat{f}_{\texttt{MS}}(X_i)|^{2}|(8|Y_i-f^{*}(X_i)|^2+2L_{1}^{2}|\enspace,
\label{eq:hatZ-Z}
\end{eqnarray}
\\\textbf{Control of $\mathbb{E}\left[\frac{1}{N}\sum_{i=1}^{N}|\hat{Z}_{i}-Z_{i}|^{2}\right]$.} First, since Assumptions~\ref{ass:fstar_bound}-\ref{ass:Y_bound} are satisfied, we have that for all $i=1,\cdots,N$,  $\mathbb{E}\left[|Y_i-f^{*}(X_i)|^{4}\right]\leq k_1 <\infty$. Second, by inequality~\eqref{eq:hatZ-Z}, Cauchy-Schwarz inequality, Jensen's inequality, and under Assumptions~\ref{ass:fstar_bound}, and~\ref{ass:Y_bound}, one gets
\begin{eqnarray*}
\mathbb{E}\left[\frac{1}{N}\sum_{i=1}^{N}|\hat{Z}_{i}-Z_{i}|^{2}\right]&\leq &   2L_{1}^{2}\mathbb{E}\left[\|\hat{f}_{\texttt{MS}}-f^{*}\|_{N}^{2}\right]+\frac{8}{N}\sum_{i=1}^{N}\mathbb{E}\left[|Y_i-f^{*}(X_i)|^{2}|f^{*}(X_i)-\hat{f}_{\texttt{MS}}(X_i)|^{2}\right]\\
&\leq & 2 L_{1}^{2} \mathbb{E}\left[\|\hat{f}_{\texttt{MS}}-f^{*}\|_{N}^{2}\right]+\frac{8}{N}\sum_{i=1}^{N}\sqrt{\mathbb{E}\left[|Y_i-f^{*}(X_i)|^{4}\right]}\sqrt{\mathbb{E}\left[|f^{*}(X_i)-\hat{f}_{\texttt{MS}}(X_i)|^{4}\right]}\\
&\leq & 2 L_{1}^{2}\mathbb{E}\left[\|\hat{f}_{\texttt{MS}}-f^{*}\|_{N}^{2}\right]+\frac{8\sqrt{k_1}L_1}{N}\sum_{i=1}^{N}\sqrt{\mathbb{E}\left[|f^{*}(X_i)-\hat{f}_{\texttt{MS}}(X_i)|^{2}\right]}\\
&\leq & 2 L_{1}^{2}\mathbb{E}\left[\|\hat{f}_{\texttt{MS}}-f^{*}\|_{N}^{2}\right]+8\sqrt{k_1}L_1\sqrt{\mathbb{E}\left[\frac{1}{N}\sum_{i=1}^{N}|f^{*}(X_i)-\hat{f}_{\texttt{MS}}(X_i)|^{2}\right]}\\
&\leq & C_1\sqrt{\mathbb{E}\left[\|\hat{f}_{\texttt{MS}}-f^{*}\|_{N}^{2}\right]}\enspace,
\end{eqnarray*}
where $C_1$ is a positive constant that depends on $k_1$ and $L_1$.
\\\textbf{Control of $\mathbb{E}\left[\frac{1}{N}\sum_{i=1}^{N}|\hat{Z}_{i}-Z_{i}|(|Z_i|+K_2)\right]$.} First, since Assumptions~\ref{ass:fstar_bound}-\ref{ass:Y_bound} are satisfied, we have that for all $i=1,\cdots,N$,  $\mathbb{E}\left[(|Y_i-f^{*}(X_i)|^{2}+K_2)^2\right]\leq k_2 <\infty$. Second, by Cauchy-Schwarz inequality and Jensen's inequality, one gets
\begin{eqnarray*}
\frac{1}{N}\sum_{i=1}^{N}\mathbb{E}\left[|\hat{Z}_{i}-Z_{i}|(|Z_i|+K_2)\right]
&\leq & \frac{1}{N}\sum_{i=1}^{N}\sqrt{\mathbb{E}\left[|\hat{Z}_{i}-Z_{i}|^2\right]}\sqrt{\mathbb{E}\left[(|Y_i-f^{*}(X_i)|^2+K_1)^2\right]}\\
&\leq & \frac{\sqrt{k_2}}{N}\sum_{i=1}^{N}\sqrt{\mathbb{E}\left[|\hat{Z}_{i}-Z_{i}|^2\right]}\\
&\leq & \sqrt{k_2}\sqrt{\mathbb{E}\left[\frac{1}{N}\sum_{i=1}^{N}|\hat{Z}_{i}-Z_{i}|^2\right]}\\
&\leq & C_2 \mathbb{E}\left[\|\hat{f}_{\texttt{MS}}-f^{*}\|_{N}^{2}\right]^{1/4}\enspace,
\end{eqnarray*}
where $C_2$ is a positive constant that depends on $C_1$ and $k_2$. Thus, there exists an absolute constant $C$ such that 
\begin{equation*}
\mathbb{E}\left[\max_{m\in[M_2]}|\hat{R}_{N}(\hat{\sigma}^{2}_{\hat{s},m})-R_{N}(\hat{\sigma}^{2}_{\hat{s},m})|\right]\leq C\mathbb{E}\left[\|\hat{f}_{\texttt{MS}}-f^{*}\|_{N}^{2}\right]^{1/4}.
\end{equation*}
We need the following proposition: 
\begin{proposition}
\label{prob:Emp_Nor_fhatMS}
Let $\hat{f}_{\texttt{MS}}$ the aggregate defined in Eq.~\eqref{est: hatfMS}. Then, under Assumptions~\ref{ass:Y_bound} and~\ref{ass:hatfs_hatsigma_hats_m_bound} there exists an absolute constant $C$ such that 
\begin{equation*}
\mathbb{E}\left[\|\hat{f}_{\texttt{MS}}-f^{*}\|_{N}^{2}\right]\leq \min_{s\in[M_1]}\mathbb{E}\left[\|\hat{f}_{s}-f^{*}\|_{N}^{2}\right]+ C\left(\frac{\log(M_1)}{N}\right)^{1/2}.
\end{equation*}
\end{proposition}
This result study the upper-bound of empirical norm risk of the aggregate $\hat{f}_{\texttt{MS}}$ and the proof of it exists in~\cite{Tsybakov14}. Besides, the Proposition~\ref{prob:Emp_Nor_fhatMS} and the elementary inequality $(x+y)^{1/4}\leq x^{1/4}+y^{1/4}$ for all $x,y\geq 0$ give us the following inequality
\begin{equation*}
\mathbb{E}\left[R(\hat{\sigma}^{2}_{\texttt{MS}})-R(\tilde{\sigma}^{2}_{\texttt{MS}})\right]\leq C^{'}\left\{\min_{s\in[M_1]}\mathbb{E}\left[\|\hat{f}_{s}-f^{*}\|_{N}^{2}\right]\right\}^{1/4} + C^{"}\left(\frac{\log(M_1)}{N}\right)^{1/8}\enspace,
\end{equation*}
where $C^{'}$ is a constant which depends on $C_2$ and $C^{"}$ is a constant which depends on $C_2$ and the constant in Proposition~\ref{prob:Emp_Nor_fhatMS}.

Merging the results of the \textbf{Step 1.} and \textbf{Step 2.} in Eq.~\eqref{eq:Eq1} and we get the result.
\begin{remark}
In the case where $Y$ is bounded and from Eq.~\eqref{eq:hatZ-Z}, we oberve that there exists a constant $C_3$ such that
\begin{eqnarray}
|\hat{Z_i}-Z_i|^2 \leq C_{3}|f^{*}(X_i)-\hat{f}_{\texttt{MS}}(X_i)|^{2}\enspace.
\label{eq:hatZ_Zbounded}
\end{eqnarray}
By Jensen’s inequality twice an inequality~\eqref{eq: Ineg1} and from Eq.\eqref{eq:hatZ_Zbounded}, one gets there exists an absolute constant $C_4$ such that
\begin{equation}
\mathbb{E}\left[\max_{m\in[M_2]}|\hat{R}_{N}(\hat{\sigma}^{2}_{\hat{s},m})-R_{N}(\hat{\sigma}^{2}_{\hat{s},m})|\right]\leq C_4\mathbb{E}\left[\|\hat{f}_{\texttt{MS}}-f^{*}\|_{N}^{2}\right]^{1/2}.
\label{eq: Inegbbound}
\end{equation}
Finally, we apply Proposition~\ref{prob:Emp_Nor_fhatMS} in Eq.\eqref{eq: Inegbbound} to get the result.
\end{remark}
\end{proof}
\section{Proof of Proposition~\ref{prob:Emp_Nor_fhatMS}}
From the definition of \texttt{MS}-estimator $\hat{f}_{\texttt{MS}}$, we get by a simple algebra that, for any $s\in [M_1]$
\begin{eqnarray*}
\|\hat{f}_{\texttt{MS}}-f^*\|^{2}_{N}\leq \|\hat{f}_{s}-f^*\|^{2}_{N}+2<\hat{f}_{\texttt{MS}}-\hat{f}_{s}, Y-f^{*}>,
\end{eqnarray*}
where $<\hat{f}_{\texttt{MS}}-\hat{f}_{s}, Y-f^{*}>:=\frac{1}{N}\sum_{i=1}^{N}\left((\hat{f}_{\texttt{MS}}(X_i)-\hat{f}_{s}(X_i))(Y_i-f^{*}(X_i))\right)$. Therefore, one gets for any $s\in [M_1]$
\begin{eqnarray}
\label{eq:PropIneq1}
\mathbb{E}\left[\|\hat{f}_{\texttt{MS}}-f^*\|^{2}_{N}\right]\leq \mathbb{E}\left[\|\hat{f}_{s}-f^*\|^{2}_{N}\right]+2\mathbb{E}\left[<\hat{f}_{\texttt{MS}}-\hat{f}_{s}, Y-f^{*}>\right].
\end{eqnarray}
We control the second term in the r.h.s. of Eq~\eqref{eq:PropIneq1}. Firstly, we notice that
\begin{eqnarray*}
\mathbb{E}\left[<\hat{f}_{\texttt{MS}}-\hat{f}_{s}, Y-f^{*}>\right]\leq \mathbb{E}\left[\max_{1\leq j \leq M_1}<\hat{f}_{j}-\hat{f}_{s}, Y-f^{*}>\right].
\end{eqnarray*}
Secondly, since $Y-f^*$ is $\rho$-subgaussian where $\rho$ is a positive constant which depends on $Y-f^*$, then the variables $<\hat{f}_{j}-\hat{f}_{s}, Y-f^{*}>$ is $\bar{\rho}$-subgaussian where $\bar{\rho}^2=\frac{\rho^2\|\hat{f}_j-\hat{f}_{s}\|_{N}^{2}}{N}$. Moreover, under Assumption~\ref{ass:hatfs_hatsigma_hats_m_bound}, it is clear that $\max_{1\leq j\leq M_1}\|\hat{f}_{j}-\hat{f}_s\|_{N}^{2}\leq B$ where $B$ is a constant which depends on $K_1$. Therefore, we use Lemma~\ref{lem:maxIneqSubG} and we get
\begin{eqnarray*}
\mathbb{E}\left[\max_{1\leq j \leq M_1}<\hat{f}_{j}-\hat{f}_{s}, Y-f^{*}>\right]\leq \rho \sqrt{B}\sqrt{\frac{2\log(M_1)}{N}}.
\end{eqnarray*}
Thus, 
\begin{eqnarray*}
\mathbb{E}\left[\|\hat{f}_{\texttt{MS}}-f^*\|^{2}_{N}\right]\leq \min_{s\in [M_1]}\mathbb{E}\left[\|\hat{f}_{s}-f^*\|^{2}_{N}\right]+2\rho \sqrt{B}\sqrt{\frac{2\log(M_1)}{N}}.
\end{eqnarray*}
\section{Proof of Theorem~\ref{thm:UpperBoundCM}}
We introduce the following aggregates
\begin{equation*}
\tilde{\sigma}^{2}_{\texttt{C}}:= \hat{\sigma}^{2}_{\hat{\lambda},\tilde{\beta}}\enspace,\;\text{where}\enspace \enspace \tilde{\beta}\in \argmin_{\beta\in \Lambda^{M_{2}}} R_{N}(\hat{\sigma}^{2}_{\hat{\lambda},\beta})\enspace,
\end{equation*}
and
\begin{equation*}
\bar{\sigma}^{2}_{\texttt{C}}:= \hat{\sigma}^{2}_{\hat{\lambda},\bar{\beta}}\enspace,\;\text{where}\enspace \enspace \bar{\beta}\in \argmin_{\beta\in\Lambda^{M_{2}}} R(\hat{\sigma}^{2}_{\hat{\lambda},\beta}).
\end{equation*}
Consider the following decomposition
\begin{equation}
\mathbb{E}\left[|\hat{\sigma}^{2}_{\texttt{C}}(X)-\sigma^{2}(X)|^2\right]= \mathbb{E}\left[R(\hat{\sigma}^{2}_{\texttt{C}})-R(\tilde{\sigma}^{2}_{\texttt{C}})\right]+\mathbb{E}\left[R(\tilde{\sigma}^{2}_{\texttt{C}})-R(\bar{\sigma}^{2}_{\texttt{C}})\right]+\mathbb{E}\left[R(\bar{\sigma}^{2}_{\texttt{C}})- R(\sigma^2)\right].
\label{eq:DecompRiskCM}
\end{equation}
\textbf{Step 1.} Study of the term $\mathbb{E}\left[R(\bar{\sigma}^{2}_{\texttt{C}})- R(\sigma^2)\right]$. We use the same proof of Lemma~\ref{lem:aprox_error}, and we get
\begin{equation*}
\mathbb{E}\left[R(\bar{\sigma}^{2}_{\texttt{C}})- R(\sigma^2)\right]\leq \mathbb{E}\left[\inf_{\beta\in \Lambda^{M_2}}\mathbb{E}_{X}\left[|\hat{\sigma}_{\hat{\lambda},\beta}^2(X)-\sigma^{2}(X)\right]\right].
\end{equation*}
\textbf{Step 2.}  Study of the term $\mathbb{E}\left[R(\tilde{\sigma}^{2}_{\texttt{C}})-R(\bar{\sigma}^{2}_{\texttt{C}})\right]$. We use the fact that $R_{N}(\tilde{\sigma}^{2}_{\texttt{C}})\leq R_{N}(\bar{\sigma}^{2}_{\texttt{C}})$ , and we get the uniform bound
\begin{equation*}
\mathbb{E}\left[R(\tilde{\sigma}^{2}_{\texttt{C}})-R(\bar{\sigma}^{2}_{\texttt{C}})\right]\leq 2\mathbb{E}\left[\sup_{(\lambda,\beta)\in   \Lambda^{M_1}\times\Lambda^{M_2}}|R_N(\hat{\sigma}^{2}_{\lambda,\beta})-R(\hat{\sigma}^{2}_{\lambda,\beta})|\right].
\end{equation*}
Since $\Lambda^{M_2}$ (resp. $\Lambda^{M_1}$) is compact, we have $\Lambda^{M_2}\subset \bar{B}(0,1)$ (the closed unit 
 ball) (resp. $\Lambda^{M_1}\subset \bar{B}(0,1)$), and there exists an $\epsilon_2$-net $\Lambda^{M_2}_{\epsilon_2}$ of $\Lambda^{M_2}$ (resp. an $\epsilon_1$-net $\Lambda^{M_1}_{\epsilon_1}$ of $\Lambda^{M_1}$)~\wrt $\|\cdot\|_{1,M_2}$ (resp. $\|\cdot\|_{1,M_1}$) such that $|\Lambda^{M_2}_{\epsilon_2}|\leq \left(3/\epsilon_2\right)^{M_2}$ (resp. $|\Lambda^{M_1}_{\epsilon_1}|\leq \left(3/\epsilon_1\right)^{M_1}$). In particular, for all $\beta\in\Lambda^{M_2}$ (resp. $\lambda\in\Lambda^{M_1}$) there exists $\beta^{\epsilon_2}\in\Lambda^{M_2}_{\epsilon_2}$ (resp. $\lambda^{\epsilon_1}\in\Lambda^{M_1}_{\epsilon_1}$ ) such that $\|\beta-\beta^{\epsilon_2}\|_{1,M_2}\leq \epsilon_2$ (resp. $\|\lambda-\lambda^{\epsilon_1}\|_{1,M_1}\leq \epsilon_1$). From triangle inequality, one gets
\begin{multline}
|R_N(\hat{\sigma}^{2}_{\lambda,\beta})-R(\hat{\sigma}^{2}_{\lambda,\beta})|\leq |R_N(\hat{\sigma}^{2}_{\lambda,\beta})-R_N(\hat{\sigma}^{2}_{\lambda,\beta^{\epsilon_2}})|+|R_N(\hat{\sigma}^{2}_{\lambda,\beta^{\epsilon_2}})-R_N(\hat{\sigma}^{2}_{\lambda^{\epsilon_1},\beta^{\epsilon_2}})|+|R_N(\hat{\sigma}^{2}_{\lambda^{\epsilon_1},\beta^{\epsilon_2}})-R(\hat{\sigma}^{2}_{\lambda^{\epsilon_1},\beta^{\epsilon_2}})|\\
+ |R(\hat{\sigma}^{2}_{\lambda^{\epsilon_1},\beta^{\epsilon_2}})-R(\hat{\sigma}^{2}_{\lambda,\beta^{\epsilon_2}})|+|R(\hat{\sigma}^{2}_{\lambda,\beta^{\epsilon_2}})-R(\hat{\sigma}^{2}_{\lambda,\beta})|.
\label{eq:IneqRiskCM1}
\end{multline}
\begin{enumerate}
\item 
\textbf{Control of $|R(\hat{\sigma}^{2}_{\lambda,\beta^{\epsilon_2}})-R(\hat{\sigma}^{2}_{\lambda,\beta})|$.} By Jensen's inequality, under assumptions~\ref{ass:fstar_bound}-~\ref{ass:Y_bound}-~\ref{ass:hatflambda_hatsigma_hatlambda__bound} and $\mathbb{E}[\xi^2]=1$ we obtain 
\begin{eqnarray*}
|R(\hat{\sigma}^{2}_{\lambda,\beta^{\epsilon_2}})-R(\hat{\sigma}^{2}_{\lambda,\beta})|&\leq & \mathbb{E}\left[\big||Z-\hat{\sigma}^{2}_{\lambda,\beta^{\epsilon_2}}(X)|^2-|Z-\hat{\sigma}^{2}_{\lambda,\beta}(X)|^2\big|\right]\\
& = & \mathbb{E}\left[\big|\left(\sum_{j=1}^{M_2}\left(\beta_{j}-\beta^{\epsilon_2}_{j}\right)\hat{\sigma}^{2}_{\lambda,j}(X)\right)\left(2Z-\hat{\sigma}^{2}_{\lambda,\beta}(X)-\hat{\sigma}^{2}_{\lambda,\beta^{\epsilon_2}}(X)\right) \big|\right]\\
&\leq & C_1 \epsilon_2\enspace,
\end{eqnarray*}
where $C_1$ is a constant which depends on the upper bounds of
$\sigma^2$ and $\hat{\sigma}^{2}_{\lambda,j}$.
\item
\textbf{Control of $|R_N(\hat{\sigma}^{2}_{\lambda,\beta})-R_N(\hat{\sigma}^{2}_{\lambda,\beta^{\epsilon_2}})|$.} Since Assumptions~\ref{ass:fstar_bound},~\ref{ass:Y_bound}, and~\ref{ass:hatfs_hatsigma_hats_m_bound} are satisfied, we obtain
\begin{eqnarray*}
|R_N(\hat{\sigma}^{2}_{\lambda,\beta})-R_N(\hat{\sigma}^{2}_{\lambda,\beta^{\epsilon_2}})|&\leq &  \frac{1}{N}\sum_{i=1}^{N}\left(\sum_{j=1}^{M_2}|\beta_{j}-\beta^{\epsilon_2}_{j}||\hat{\sigma}^{2}_{\lambda,j}(X_i)|\right)|2\sigma^{2}(X_i)\xi^{2}_{i}-\hat{\sigma}^{2}_{\lambda,\beta}(X_i)-\hat{\sigma}^{2}_{\lambda,\beta^{\epsilon_2}}(X_i)|\\
&\leq & k\epsilon_{2}\left(\frac{C_2}{N}\sum_{i=1}^{N}\xi_{i}^{2}+C_3\right),
\end{eqnarray*}
where $k$ is the bound of $\hat{\sigma}^{2}_{\lambda,j}$, $C_2$ is the constant which depends on $\sigma^2$ and $C_3$ is the constant which depends on the upper bounds $\hat{\sigma}^{2}_{\lambda,\beta}$ and $\hat{\sigma}^{2}_{\lambda,\beta^{\epsilon_2}}$.
\item 
\textbf{Control of $|R(\hat{\sigma}^{2}_{\lambda^{\epsilon_1},\beta^{\epsilon_2}})-R(\hat{\sigma}^{2}_{\lambda,\beta^{\epsilon_2}})|$.}  Under Assumptions~\ref{ass:fstar_bound},~\ref{ass:Y_bound},~\ref{ass:hatflambda_hatsigma_hatlambda__bound}, and~\ref{ass:KLipchitz}, we get
\begin{eqnarray*}
&&|R(\hat{\sigma}^{2}_{\lambda^{\epsilon_1},\beta^{\epsilon_2}})-R(\hat{\sigma}^{2}_{\lambda,\beta^{\epsilon_2}})|\\
&\leq &\mathbb{E}\left[\sum_{j=1}^{M_2}\beta_{j}^{\epsilon_2}|\hat{\sigma}^{2}_{\lambda,j}(X)-\hat{\sigma}^{2}_{\lambda^{\epsilon_1},j}(X)||2\sigma^2(X)\xi^2-\hat{\sigma}^{2}_{\lambda^{\epsilon_1},\beta^{\epsilon_2}}(X)-\hat{\sigma}^{2}_{\lambda,\beta^{\epsilon_2}}(X)|\right]\\
&\leq & \sum_{j=1}^{M_2}\beta_{j}^{\epsilon_2}\left(2\mathbb{E}\left[\mathbb{E}\left[|\hat{\sigma}^{2}_{\lambda,j}(X)-\hat{\sigma}^{2}_{\lambda^{\epsilon_1},j}(X)|\sigma^2(X)\xi^2|\mathcal{D}_{n},X\right]\right]+\mathbb{E}\left[|\hat{\sigma}^{2}_{\lambda,j}(X)-\hat{\sigma}^{2}_{\lambda^{\epsilon_1},j}(X)||\hat{\sigma}^{2}_{\lambda^{\epsilon_1},\beta^{\epsilon_2}}(X)-\hat{\sigma}^{2}_{\lambda,\beta^{\epsilon_2}}(X)|\right]\right)\\
&\leq & \sum_{j=1}^{M_2}\beta_{j}^{\epsilon_2}\left(2\mathbb{E}\left[|\hat{\sigma}^{2}_{\lambda,j}(X)-\hat{\sigma}^{2}_{\lambda^{\epsilon_1},j}(X)|\sigma^2(X)\mathbb{E}\left[\xi^2\right]\right]+\mathbb{E}\left[|\hat{\sigma}^{2}_{\lambda,j}(X)-\hat{\sigma}^{2}_{\lambda^{\epsilon_1},j}(X)||\hat{\sigma}^{2}_{\lambda^{\epsilon_1},\beta^{\epsilon_2}}(X)-\hat{\sigma}^{2}_{\lambda,\beta^{\epsilon_2}}(X)|\right]\right)\\
&\leq & C_4\epsilon_1,
\end{eqnarray*}
where $C_4$ is constant which depends on $K$ and the upper bounds of $\sigma^2$, $\hat{\sigma}^{2}_{\lambda^{\epsilon_1},\beta^{\epsilon_2}}$ and $\hat{\sigma}^{2}_{\lambda,\beta^{\epsilon_2}}$.
\item \textbf{Control of $|R_N(\hat{\sigma}^{2}_{\lambda^{\epsilon_1},\beta^{\epsilon_2}})-R_N(\hat{\sigma}^{2}_{\lambda,\beta^{\epsilon_2}})|$.} We use the same way as $3.$ and we obtain
\begin{equation*}
|R_N(\hat{\sigma}^{2}_{\lambda^{\epsilon_1},\beta^{\epsilon_2}})-R_N(\hat{\sigma}^{2}_{\lambda,\beta^{\epsilon_2}})|\leq \epsilon_1\left(\frac{C_2}{N}\sum_{i=1}^{N}\xi_{i}^{2}+C_5\right),
\end{equation*}
where $C_5$ is constant which depends on the upper bounds of $\hat{\sigma}^{2}_{\lambda^{\epsilon_1},\beta^{\epsilon_2}}$ and $\hat{\sigma}^{2}_{\lambda,\beta^{\epsilon_2}}$.
\end{enumerate}
Therefore, we deduce that
\begin{equation*}
\mathbb{E}\left[\sup_{(\lambda,\beta)\in   \Lambda^{M_1}\times\Lambda^{M_2}}|R_N(\hat{\sigma}^{2}_{\lambda,\beta})-R(\hat{\sigma}^{2}_{\lambda,\beta})|\right]\leq C_{k,C_2, C_3, C_4,C_5}(\epsilon_1+ \epsilon_2)+\mathbb{E}\left[\sup_{(\lambda,\beta)\in  \Lambda^{M_1}_{\epsilon_1}\times\Lambda^{M_2}_{\epsilon_{2}}}|R_N(\hat{\sigma}^{2}_{\lambda,\beta})-R(\hat{\sigma}^{2}_{\lambda,\beta})|\right].
\end{equation*}
For some $(\lambda,\beta)\in \Lambda^{M_1}_{\epsilon_1}\times\Lambda^{M_2}_{\epsilon_{2}}$, set $T_{i}(\lambda,\beta)=|Z_i-\hat{\sigma}^{2}_{\lambda,\beta}(X_i)|^{2}=|\sigma^{2}(X_i)\xi^{2}_{i}-\hat{\sigma}^{2}_{\lambda,\beta}(X_i)|^{2}$ for all $i=1,\cdots, N$. Let $L>0$.  Since the variables $T_{i}(\lambda,\beta)$ are \iid, we have 
\begin{equation}
\label{eq:decompdeviationboundCMSigma}
\begin{split}
\mathbb{E}\left[\sup_{(\lambda,\beta)\in \Lambda^{M_1}_{\epsilon_1}\times\Lambda^{M_2}_{\epsilon_{2}}}|R_N(\hat{\sigma}^{2}_{\lambda,\beta})-R(\hat{\sigma}^{2}_{\lambda,\beta})|\right]&\leq  \mathbb{E}\left[\sup_{(\lambda,\beta)\in \Lambda^{M_1}_{\epsilon_1}\times\Lambda^{M_2}_{\epsilon_{2}}}\bigg|\frac{1}{N}\sum_{i=1}^{N}(T_{i}(\lambda,\beta)-\mathbb{E}[T_{i}(\lambda,\beta)])\one_{\{|\xi_{i}|\leq L\}}\bigg|\right]\\
&+\mathbb{E}\left[\sup_{(\lambda,\beta)\in \Lambda^{M_1}_{\epsilon_1}\times\Lambda^{M_2}_{\epsilon_{2}}}\bigg|\frac{1}{N}\sum_{i=1}^{N}(T_{i}(\lambda,\beta)-\mathbb{E}[T_{i}(\lambda,\beta)])\one_{\{|\xi_{i}|> L\}}\bigg|\right].
\end{split}
\end{equation}
\textbf{Step 2.1}. We control the first term on the r.h.s. of Eq.~\eqref{eq:decompdeviationboundCMSigma}.  On the event $\{|\xi| \leq L\}$ and under assumptions~\ref{ass:fstar_bound},~\ref{ass:Y_bound} and~\ref{ass:hatflambda_hatsigma_hatlambda__bound}, we get $|T_{i}(\lambda,\beta)|\leq c_{1}L^4+\bar{c}_1$ for all $i=1,\cdots, N$ where $c_{1}$ is a positive constant which depends on the upper bound of $\sigma^2$ and  $\bar{c}_1$ depends on the upper bound of $\hat{\sigma}^{2}_{\lambda,\beta}$. Conditionally on $\mathcal{D}_n$, we apply Hoeffding's inequality, for all $(\lambda,\beta)\in \Lambda^{M_1}_{\epsilon_1}\times\Lambda^{M_2}_{\epsilon_{2}}$, and all $t\geq 0$
\begin{equation*}
\mathbb{P}\left(\bigg|\frac{1}{N}\sum_{i=1}^{N}(T_{i}(\lambda,\beta)-\mathbb{E}[T_{i}(\lambda,\beta)])\one_{\{|\xi_{i}|\leq L\}}\bigg|\geq t\right) \leq 2 \exp\left(-\frac{-Nt^{2}}{2(c_{1}L^4+\bar{c}_1)^2}\right)\enspace,
\end{equation*}
By a union bound on $(\lambda,\beta)\in \Lambda^{M_1}_{\epsilon_1}\times\Lambda^{M_2}_{\epsilon_{2}}$ and choosing $\epsilon_1=\epsilon_2=\frac{3}{N}$, we deduce that for all $t\geq 0$
\begin{equation*}
\mathbb{P}\left(\sup_{(\lambda,\beta)\in \Lambda^{M_1}_{\epsilon_1}\times\Lambda^{M_2}_{\epsilon_{2}}}\bigg|\frac{1}{N}\sum_{i=1}^{N}(T_{i}(\lambda,\beta)-\mathbb{E}[T_{i}(\lambda,\beta)])\one_{\{|\xi_{i}|\leq L\}}\bigg|\geq t\right) \leq 2 \exp\left((M_1+M_2)\log(N)-\frac{-Nt^{2}}{2(c_{1}L^4+\bar{c}_1)^2}\right).
\end{equation*}
We apply Lemma~\ref{lem:technProba}. Then, there exists a positive constant $\mathbf{c}$ such that
\begin{equation*}
 \mathbb{E}\left[\sup_{(\lambda,\beta)\in \Lambda^{M_1}_{\epsilon_1}\times\Lambda^{M_2}_{\epsilon_{2}}}\bigg|\frac{1}{N}\sum_{i=1}^{N}(T_{i}(\lambda,\beta)-\mathbb{E}[T_{i}(\lambda,\beta)])\one_{\{|\xi_{i}|\leq L\}}\bigg|\right]\leq \mathbf{c}(c_{2}L^{4}+\bar{c}_2)\left(\frac{(M_1+M_2)\log(N)}{N}\right)^{1/2},
\end{equation*}
where $c_2$ is constant which depends on $c_1$ and $\bar{c}_2$ on $\bar{c}_1$.\\
\textbf{Step 2.2.} We control the second term on the r.h.s. of Eq.~\eqref{eq:decompdeviationboundCMSigma}. Thanks of the boundness of $\sigma^2$ and $\hat{\sigma}^{2}_{\lambda,\beta}$ and $\mathbb{E}[\xi^4]=3$, we get $\mathbb{E}[T_{i}(\lambda,\beta)]\leq c_3$ and $T_{i}(\lambda,\beta)\leq c_4 \xi_{i}^{4} +c_5$ for all $i=1,\cdots,N$ where $c_4$ and $c_5$ are constants which depend on the upper bounds of $\sigma^2$ and $\hat{\sigma}_{\lambda,\beta}^2$ respectively. By Cauchy–Schwarz inequality and Lemma~\ref{lem:gaussianProb} we obtain
\begin{eqnarray*}
\mathbb{E}\left[\sup_{(\lambda,\beta)\in \Lambda^{M_1}_{\epsilon_1}\times\Lambda^{M_2}_{\epsilon_{2}}}\bigg|\frac{1}{N}\sum_{i=1}^{N}(T_{i}(s,m)-\mathbb{E}[T_{i}(s,m)])\one_{\{|\xi_{i}|> L\}}\bigg|\right]
&\leq & \frac{c_4}{N}\sum_{i=1}^{N}\mathbb{E}[\xi_{i}^{4}\one_{\{|\xi_{i}|> L\}}]+(c_3+c_5)\mathbb{P}(|\xi_1|>L)\\
&\leq & \bar{c}_{4}\sqrt{\mathbb{P}(|\xi_1|>L)}+(c_3+c_5)\mathbb{P}(|\xi_1|>L)\\
&\leq & \frac{\bar{c}_{4}\exp(-L^2/4)}{\sqrt{L}}+\frac{(c_3+c_5)\exp(-L^2/2)}{L}\enspace,
\end{eqnarray*}
where $\bar{c}_{4}$ is a positive constant that depends on $c_4$ and $\xi$. 

Merging the results of the \textbf{Step 2.1.} and \textbf{Step 2.2.} in Eq.\eqref{eq:decompdeviationboundCMSigma}, and we obtain
\begin{multline*}
\mathbb{E}\left[\sup_{(\lambda,\beta)\in \Lambda^{M_1}_{\epsilon_1}\times\Lambda^{M_2}_{\epsilon_{2}}}|R_N(\hat{\sigma}^{2}_{\lambda,\beta})-R(\hat{\sigma}^{2}_{\lambda,\beta})|\right]\leq \mathbf{c}(c_{2}L^{4}+\bar{c}_2)\left(\frac{(M_1+M_2)\log(N)}{N}\right)^{1/2}+ \frac{\bar{c}_{4}\exp(-L^2/4)}{\sqrt{L}}\\
+\frac{(c_3+c_5)\exp(-L^2/2)}{L}.
\end{multline*}
Puting $L=\sqrt{2\log(N)}$, and we get
\begin{equation*}
\mathbb{E}\left[\sup_{(\lambda,\beta)\in \Lambda^{M_1}_{\epsilon_1}\times\Lambda^{M_2}_{\epsilon_{2}}}|R_N(\hat{\sigma}^{2}_{\lambda,\beta})-R(\hat{\sigma}^{2}_{\lambda,\beta})|\right] \leq c_{6}\left(\frac{(M_1+M_2)\log^{5}(N)}{N}\right)^{1/2}\enspace,
\end{equation*}
where $c_6$ is constant which depends on $c_{2}$. Thus,
\begin{equation*}
\mathbb{E}\left[R(\tilde{\sigma}^{2}_{\texttt{C}})-R(\bar{\sigma}^{2}_{\texttt{C}})\right]\leq C\left(\frac{(M_1+M_2)\log^{5}(N)}{N}\right)^{1/2}\enspace,
\end{equation*}
where $C$ is constant which depends on $c_6$ and $\mathbf{c}$.\\
\begin{remark} When $Y$ is bounded, it's clear that there exists an absolute constant $C>0$
\begin{equation*}
\mathbb{E}\left[R(\tilde{\sigma}^{2}_{\texttt{C}})-R(\bar{\sigma}^{2}_{\texttt{C}})\right]\leq C\left(\frac{(M_1+M_2)\log(N)}{N}\right)^{1/2}\enspace .
\end{equation*}
\end{remark}
\textbf{Step 3.} Study of the term $\mathbb{E}\left[R(\hat{\sigma}^{2}_{\texttt{C}})-R(\tilde{\sigma}^{2}_{\texttt{C}})\right]$. We use the same arguments of proof of Theorem~\ref{thm:Risk_MS_sigma3} (\textbf{Step 2.2}), and we get that there exists two positive constants $C_1$ and $C_2$ such that
\begin{equation}
\label{eq:bound11}
\mathbb{E}\left[R(\hat{\sigma}^{2}_{\texttt{C}})-R(\tilde{\sigma}^{2}_{\texttt{C}})\right]\leq C_{1}\left\{\mathbb{E}\left[\|\hat{f}_{\texttt{C}}-f^{*}\|^{2}_{N}\right]\right\}^{1/p}+ C_{2}\alpha_{N}\enspace,
\end{equation}
where $p=2$ if $Y$ is bounded, $p=4$ otherwise, and
\[ \alpha_N= \left\{ \begin{array}{ll}
         \left(\frac{(M_1+M_2)\log(N)}{N}\right)^{1/2} & \mbox{if  $Y$ is bounded};\\
        \left(\frac{(M_1+M_2)\log^{5}(N)}{N}\right)^{1/2} & \mbox{otherwise}.\end{array} \right. \] 
In the sequel, we give the following proposition
\begin{proposition}
\label{prop:EmpiricalNormfCM}
Let $\hat{f}_{\texttt{C}}$ be the aggregate defined in Eq.~\eqref{est:fhatCM}. Then, under Assumptions~\ref{ass:Y_bound} and~\ref{ass:hatflambda_hatsigma_hatlambda__bound} there exists an absolute constant $C>0$
\begin{equation*}
\mathbb{E}\left[\|\hat{f}_{\texttt{C}}-f^{*}\|_{N}^{2}\right]\leq \min_{\lambda\in \Lambda^{M_1}}\mathbb{E}\left[\|\hat{f}_{\lambda}-f^{*}\|_{N}^{2}\right]+ C\sqrt{\frac{\log(M_1)}{N}}\enspace.
\end{equation*}
\end{proposition}
The proof of this proposition is similar of the proof of Proposition~\ref{prob:Emp_Nor_fhatMS}. Thus, we apply Proposition~\ref{prop:EmpiricalNormfCM} in inequality~\eqref{eq:bound11} and we get
\begin{equation*}
\mathbb{E}\left[R(\hat{\sigma}^{2}_{\texttt{C}})-R(\tilde{\sigma}^{2}_{\texttt{C}})\right]\leq C_{1}\left\{\min_{\lambda\in \Lambda^{M_2}}\mathbb{E}\left[\|\hat{f}_{\lambda}-f^{*}\|_{N}^{2}\right]\right\}^{1/p}+
 \bar{C}_{1}\phi^{\texttt{C}}_{N}(M_1)\enspace,
\end{equation*}
where $\bar{C}_{1}$ is a constant that depends on $C_{1}$ and the constant in Proposition~\ref{prop:EmpiricalNormfCM}, where $p=2$ if $Y$ is bounded, $p=4$ otherwise, and
\[ \phi^{\texttt{C}}_{N}(M_1)= \left\{ \begin{array}{ll}
         \left(\frac{\log(M_{1})}{N}\right)^{1/4} & \mbox{if  $Y$ is bounded};\\
        \left(\frac{\log(M_{1})}{N}\right)^{1/8}& \mbox{otherwise}.\end{array} \right. \]
  Combining \textbf{Step $1$}, \textbf{Step $2$} and \textbf{Step $3$} in Eq~\eqref{eq:DecompRiskCM} yields the result.
\section{Technical lemmas}
In this section, we gather several technical results which are used to derive the proof of results of this paper.
\begin{lemma}
\label{lem:gaussianProb}
Let $X$ be the standard gaussian distribution, then for any $x>0$, it holds
\begin{equation*}
\mathbb{P}(X> x)\leq \frac{\exp(-x^2/2)}{\sqrt{2\pi}x}\enspace, \enspace \textit{and} \enspace \; \mathbb{P}(|X|> x)\leq \sqrt{\frac{2}{\pi}}\frac{\exp(-x^2/2)}{x}\enspace.
\end{equation*}
\end{lemma}
\begin{proof}
Since $X\sim \mathcal{N}(0,1)$, one gets
\begin{eqnarray*}
\mathbb{P}\left(X>x\right)=\frac{1}{\sqrt{2\pi}}\int_{x}^{+\infty}\exp(-u^2/2)du\leq \frac{1}{\sqrt{2\pi}}\int_{x}^{+\infty}\frac{u}{x}\exp(-u^2/2)du=\frac{\exp(-x^2/2)}{\sqrt{2\pi}x}.
\end{eqnarray*}
The second inequality follows from symmetry and the last one using the union bound
\begin{eqnarray*}
\mathbb{P}(|X|> x)\leq 2 \mathbb{P}\left(X>x\right).
\end{eqnarray*} 
\end{proof}
\begin{lemma}
\label{lem:momentsBound}
Let $X\sim \mathcal{N}(0,1)$ and $k\geq 1$, then 
\begin{equation*}
\mathbb{E}\left[|X|^{2k}\right]\leq 2^{k+1}k!.
\end{equation*}
\end{lemma}
\begin{proof}
\begin{eqnarray*}
\mathbb{E}\left[|X|^{2k}\right]=\int_{0}^{+\infty}\mathbb{P}\left(|X|^{2k}>t\right)dt=\int_{0}^{+\infty}\mathbb{P}\left(|X|>t^{\frac{1}{2k}}\right)dt &\leq& 2\int_{0}^{+\infty}\exp\left(-t^{\frac{1}{k}}/2\right)dt\\
&\overset{u=t^{\frac{1}{k}}/2}{=}& 2^{k+1}k\int_{0}^{+\infty}u^{k-1}\exp(-u)du= 2^{k+1}k!.
\end{eqnarray*}
\end{proof}

\begin{lemma}
\label{lem:maxIneqSubG}
Let $X_1,\cdots, X_M$ be zero mean $\nu$-subgaussian random variables, \ie   $\mathbb{E}\left[\exp(r X_i)\right]\leq exp\left(\frac{r^2\nu^2}{2}\right)$ for all $r>0$. Then
\begin{equation*}
\mathbb{E}\left[\max_{1\leq i \leq M}X_i\right]\leq \nu \sqrt{2\log(M)}.
\end{equation*}
\end{lemma}
\begin{proof}
By Jensen's inequality, for any $r>0$
\begin{eqnarray*}
\mathbb{E}\left[\max_{1\leq i \leq N}X_i\right]=\frac{1}{r}\mathbb{E}\left[\log \left(\exp \left(r\max_{1\leq i \leq M}X_i\right)\right)\right]&\leq & \frac{1}{r}\log \left(\mathbb{E}\left[\exp \left(r\max_{1\leq i \leq M}X_i\right)\right]\right)\\
&=&\frac{1}{r}\log \left(\mathbb{E}\left[\max_{1\leq i \leq M}\exp \left(r X_i\right)\right]\right)\\
&\leq & \frac{1}{r}\log \left(\sum_{i=1}^{M}\mathbb{E}\left[\exp \left(r X_i\right)\right]\right)\\
&\leq &\frac{1}{r}\log \left(\sum_{i=1}^{M}\mathbb{E}\left[\exp \left(\frac{r^2 \nu^2}{2}\right)\right]\right)=\frac{\log(M)}{r}+\frac{\nu^2r}{2}\enspace,
\end{eqnarray*}
taking  $r=\sqrt{\frac{2\log(M)}{\nu^2}}$ and we get the result.
\end{proof}
\begin{lemma}
\label{lem:technProba}
Let $N\in \mathbb{N}^{*}$, $a\geq 1$, $b$ and $c$ be two non negative real numbers. Consider $Z$ a positive random variable such that
\begin{equation}
\label{cond:lem}
\mathbb{P}\left(Z\geq t\right) \leq \min(1, \exp(a-bNt^2)\enspace. 
\end{equation}
Then, there exists a constant $C >0$ not depending of $N$ such that
\begin{equation*}
\mathbb{E}[Z]\leq C\left(\frac{a}{bN}\right)^{1/2}.
\end{equation*}
\end{lemma}
\begin{proof} By condition~\eqref{cond:lem}, we have
\begin{eqnarray}
\label{lem: eq1}
\mathbb{E}[Z]\leq \int_{0}^{+\infty}\min(1, \exp(a-bNt^2)dt\leq \left(\frac{a}{bN}\right)^{1/2}+\int_{\left(\frac{a}{bN}\right)^{1/2}}^{+\infty}\exp(a-bNt^2)dt.
\end{eqnarray}
The following elementary inequality $(x-y)^2\leq x^2-y^2$ for all $x,y\geq 0$ yields to 
\begin{eqnarray}
\label{lem: eq2}
\int_{\left(\frac{a}{bN}\right)^{1/2}}^{+\infty}\exp(a-bNt^2)dt\leq \int_{\left(\frac{a}{bN}\right)^{1/2}}^{+\infty}\exp\left(-bN\left(t-\left(\frac{a}{bN}\right)^{1/2}\right)^2\right)dt=\int_{0}^{+\infty}\exp\left(-bNu^2\right)du\leq C\left(\frac{1}{bN}\right)^{1/2}.
\end{eqnarray}
Combining Equation~\eqref{lem: eq2} in Equation~\eqref{lem: eq1} to yield the result.
\end{proof}
\begin{lemma}[Bernstein's inequality]
\label{lem:BernsteinIneq}
Let $T_1,\cdots, T_n$ be independent real valued random variables. Assume that there exists some positive numbers $v$ and $c$ such that 
\begin{equation}
\sum_{i=1}^{n}\mathbb{E}[T_{i}^{2}]\leq v\enspace,
\label{eq:condB1}
\end{equation}
and for all integers $k\geq 3$
\begin{equation}
\sum_{i=1}^{n}\mathbb{E}[(T_{i}\vee0)^{k}]\leq \frac{k!}{2}vc^{k-2}\enspace.
\label{eq:condB2}
\end{equation}
Let $S=\sum_{i=1}^{n}\left(T_i-\mathbb{E}[T_i]\right)$, then for every any positive $x$ we have
\begin{equation*}
\mathbb{P}\left(|S|\geq x\right)\leq 2 \exp\left(-\frac{x^2}{2(v+cx)}\right)\enspace.
\end{equation*}
\end{lemma}
\begin{lemma}[Hoeffding’s inequality]
Let $N\in \mathbb{N}^{*}$ and $a>0$ be a real number. Let $X_1,\cdots,X_N$ be independent random variables having values in $[-a,a]$, then for all $t>0$
\begin{equation*}
\mathbb{P}\left(\bigg|\frac{1}{N}\sum_{i=1}^{N}(X_i-\mathbb{E}[X_i])\bigg|>t\right)\leq 2\exp\left(-\frac{Nt^2}{2a^2}\right).
\end{equation*}
\end{lemma}
\begin{lemma}[Hoeffding's Lemma] Let $X\in [a,b]$ be a bounded random variable with $\mathbb{E}[X]=0$. Then, for all $\lambda \in \mathbb{R}$
\begin{equation*}
\mathbb{E}\left[\exp(\lambda X)\right]\leq \exp\left(\frac{\lambda^2(b-a)^2}{8}\right).
\end{equation*} 
\label{lem:HoeffdingLemma}
\end{lemma}

\end{document}